\documentclass{article}

\PassOptionsToPackage{numbers, sort&compress}{natbib}

\usepackage[main, final]{neurips_2026}

\usepackage[utf8]{inputenc} 
\usepackage[T1]{fontenc}    
\usepackage[backref = page, colorlinks = True, linkcolor = violet!70!white, citecolor = teal]{hyperref}       
\usepackage{url}            
\usepackage{booktabs}       
\usepackage{amsfonts}       
\usepackage{nicefrac}       
\usepackage{microtype}      
\usepackage{xcolor}         
\usepackage{amsmath,amssymb}
\usepackage{amsthm}
\usepackage{mathtools}
\usepackage{graphicx}
\usepackage{enumitem}
\usepackage{subcaption}
\usepackage{algorithm}
\usepackage{algpseudocode}
\usepackage{wrapfig}
\usepackage{cleveref}
\usepackage{titlesec}
\usepackage{standalone}
\usepackage{tikz}
\usetikzlibrary{arrows.meta,fit}
\usepackage{wrapfig}
\usepackage{algpseudocode}
\usepackage{caption}
\usepackage{tcolorbox}
\usepackage{threeparttable}
\usepackage{multirow}
\tcbset{colback=white, colframe=black, boxrule=0.5pt, arc=2pt, left=4pt, right=4pt, top=4pt, bottom=4pt}

\titlespacing*{\paragraph}{0pt}{0pt}{0.5em}
\titlespacing*{\section}{0pt}{*0.8}{*0.6}
\titlespacing*{\subsection}{0pt}{*0.7}{*0.5}

\usepackage{xspace}

\newtheorem{proposition}{Proposition}
\newtheorem{theorem}{Theorem}
\newtheorem{definition}{Definition}

\newcommand{\E}{\mathbb{E}}

\newcommand{\y}{y_{1:T}}
\newcommand{\x}{\xi_{1:T}}

\DeclareMathOperator*{\argmin}{arg\,min}

\DeclarePairedDelimiter\abs{\lvert}{\rvert}%
\DeclarePairedDelimiter\norm{\lVert}{\rVert}%

\makeatletter
\let\oldabs\abs
\def\abs{\@ifstar{\oldabs}{\oldabs*}}
\let\oldnorm\norm
\def\norm{\@ifstar{\oldnorm}{\oldnorm*}}
\makeatother

\title{\textsc{Action-BED}: Task-Driven Bayesian Experimental Design with Singly Intractable Objectives}

\author{%
  Tom Rossa \\
  Department of Statistics\\
  University of Oxford\\
  \texttt{tom.rossa@stats.ox.ac.uk} \\\\
  \And
  Angus Phillips \\
  Department of Statistics \\
  University of Oxford \\
  \texttt{angus.phillips@stats.ox.ac.uk} \\
  \And
  Tom Rainforth \\
  Department of Statistics \\
  University of Oxford \\
  \texttt{rainforth@stats.ox.ac.uk} \\
}

\begin{document}

\maketitle

\begin{abstract}
\looseness=-1
    Bayesian experimental design (BED)  has traditionally been based on maximising expected uncertainty reductions from prior to posterior. 
    A major shortfall of this approach is that it leads to doubly intractable objectives that are difficult to optimise, while customising them to particular downstream tasks of interest can also be difficult. 
    Following first principles decision theory, we demonstrate that BED can alternatively be formulated in terms of an expected future loss (EFL) on downstream actions, providing a simple and naturally task-driven framework.
    Critically, we then show that all such EFLs can be rearranged into \emph{singly intractable objectives} that can be jointly optimised with respect to both the design policy and a downstream \emph{action policy} using stochastic gradients, an approach we refer to as \textsc{Action-BED}.
    This formulation further sidesteps the need for any explicit posterior or marginal likelihood estimation and is naturally implicit, requiring only the ability to sample from the joint model over model parameters and data, and evaluate the downstream loss function.  It thus allows design policies to be learned more effectively, efficiently, and simply than existing methods, while providing easy customisation to different downstream tasks and losses.  
\end{abstract}

\section{Introduction}
\label{intro} 

\looseness=-1
Bayesian Experimental Design (BED) \citep{lindley_measure_1956, lindley_1972, chaloner_bayesian_1995, huan_optimal_2024, rainforth_modern_2024, degroot1962uncertainty} is a principled framework for the optimal acquisition of data, in which designs are chosen to maximise the expected uncertainty reduction (EUR) from prior to posterior over unknown values of interest. 
Typically the uncertainty measure is taken to be the Shannon entropy \citep{shannon_mathematical_1948}, leading to the classical expected information gain (EIG) objective \citep{lindley_measure_1956}. In this perspective, data is valued according to how it changes our belief distribution over unknown values, which is theoretically justified by the fact that downstream Bayes-optimal actions only depend on observed data through the updated belief distribution~\citep{Savage1951}.

\looseness=-1
However, there are two major issues with using this framework in practice.
First, defining uncertainties that are customised to particular downstream tasks of interest can be challenging~\citep{kerrigan_geometric_2025}, noting that the data gathered is rarely the end goal in itself and is instead used to aid downstream estimation, prediction, or decision-making.
In particular, common metrics like Shannon entropy constitute generic uncertainties which fail to reflect that not all information is equally useful for particular tasks of interest.

\looseness=-1
Second, EURs are generally doubly intractable~\citep{rainforth_modern_2024}: without an analytic posterior, the EUR is inevitably a nested expectation~\citep{rainforth_nesting_2018} as meaningful uncertainty measures are nonlinear~\citep{lindley1982scoring,Savage1971Elicitation,dawid2007geometry,gneiting2007strictly}.
This makes estimating and optimising EURs very challenging as standard estimators cannot be applied, while the specialised computational strategies proposed to deal with EURs~\citep{foster_deep_2021, goda_multilevel_2020, ao_estimating_2024, iollo_bayesian_2025, iollo_pasoa_2024, iqbal_nesting_2024, fort2017mcmc} are often only applicable to EIGs and tend to suffer from high computational costs and bias.
For example, one can look to
circumvent the double intractability of EURs by learning amortised posterior approximations~\citep{foster_variational_2019,barber_information_2003,phillips2026bayesian}, often learning this alongside the design/policy optimisation using variational formulations~\citep{foster2020unified, shen2025variational, huang_aline_2025, bracher_jadai_2025, huang_amortized_2024, huang2026efficient}.
However, this amortised posterior learning is still itself a very challenging task and inaccuracies in the posterior approximations can undermine the resulting data acquisition.

We propose to address these challenges with a fundamental move away from the EUR framework by instead defining the utility of data through \emph{the performance of the downstream actions it enables}.
Specifically, using first principles Bayesian decision theory, we show that the Bayes-optimal data gathering policy can be formulated as the minimiser of the \emph{expected future loss} (EFL) of downstream Bayes-optimal actions, allowing for the definition of experimental design objectives that are naturally customised to our downstream tasks and losses of interest.

While the set of unique objectives indexed by the EFL and EUR frameworks turns out to be equivalent, our critical insight is that the former allows for a simple rearrangement into a \emph{singly intractable} objective that we jointly optimise with respect to a design policy and a \emph{downstream action policy}.
This results in a much simpler and computationally tractable policy-based framework than existing approaches that we call \textbf{\textsc{Action-BED}}.  As shown in~\Cref{fig:method_placeholder},
\textsc{Action-BED} avoids the need to directly perform posterior approximations or explicitly estimate any uncertainties: it simply requires the minimisation of an expected loss with respect to the joint policy parameters, which can be done using standard stochastic gradient descent methods.
Furthermore, the task-driven EFL objectives it uses are \emph{naturally implicit}, requiring only reparametrisable simulations from the joint model over parameters and data, and a differentiable loss function.
%
These benefits can substantially simplify the policy training and we empirically find that \textsc{Action-BED} often significantly outperforms previous state-of-the-art in terms of EIG values even when targeting different downstream tasks and losses.

\looseness=-1
In summary, we propose a paradigm shift in BED from thinking about uncertainties to future losses.
Not only does this allow us to easily customise design policy learning to downstream tasks of interest, it also provides substantial practical benefits via our \textsc{Action-BED} framework, which circumvents the double intractability of EURs, avoids explicit posterior approximation or model density evaluations, and ultimately collects 
more informative and targeted data than existing approaches.
A reference implementation of ActionBED is provided in the TestBED PyTorch library (\url{http://github.com/testbed-benchmark/testbed}); its accompanying paper also provides additional empirical evaluations of \textsc{Action-BED} that further confirm its efficacy compared with baseline approaches~\citep{kerrigan2026testbed}.

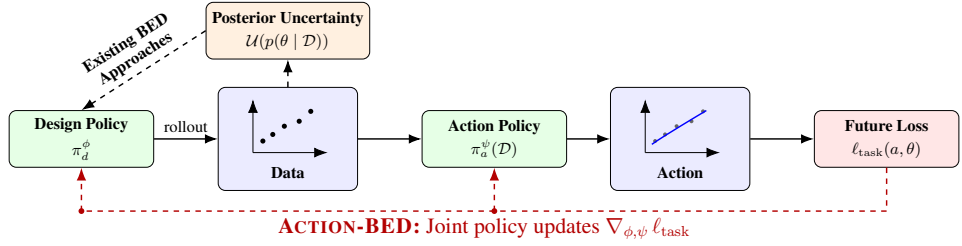
\begin{figure}
    \centering
    \begin{tikzpicture}[
    font=\small,
    >={Latex[length=2.5mm]},
    box/.style={
        draw,
        rounded corners,
        minimum width=2.8cm,
        minimum height=1.1cm,
        align=center
    },
    policy/.style={box, fill=green!10},
    prior/.style={box, fill=blue!10},
    block/.style={box, fill=blue!8},
    metric/.style={box, fill=orange!12},
    loss/.style={box, fill=red!10}
]

\node[policy] (design) at (0,0) {\textbf{Design Policy}\\[1mm] $\pi_d^\phi$};

\node[block, minimum width=2.7cm, minimum height=2.0cm] (data) at (4,0) {
\begin{tikzpicture}[x=0.32cm,y=0.32cm,baseline={(current bounding box.center)}]
    \draw[->] (0,0) -- (4.5,0);
    \draw[->] (0,0) -- (0,3.2);
    \fill (0.6,0.7) circle (1.4pt);
    \fill (1.2,1.1) circle (1.4pt);
    \fill (1.9,1.6) circle (1.4pt);
    \fill (2.8,1.9) circle (1.4pt);
    \fill (3.5,2.5) circle (1.4pt);
\end{tikzpicture}
\\[1.5mm]
\textbf{Data}
};

\node[policy] (actionnet) at (8,0) {\textbf{Action Policy}\\[1mm] $\pi_a^\psi(\mathcal{D})$};

\node[block, minimum width=2.7cm, minimum height=2.0cm] (action) at (11.6,0) {
\begin{tikzpicture}[x=0.32cm,y=0.32cm,baseline={(current bounding box.center)}]
    \draw[->] (0,0) -- (4.5,0);
    \draw[->] (0,0) -- (0,3.2);
    \fill[black!60] (0.6,0.7) circle (1.1pt);
    \fill[black!60] (1.2,1.1) circle (1.1pt);
    \fill[black!60] (1.9,1.6) circle (1.1pt);
    \fill[black!60] (2.8,1.9) circle (1.1pt);
    \fill[black!60] (3.5,2.5) circle (1.1pt);
    \draw[thick,blue]
        plot[smooth] coordinates {(0.4,0.5) (1.0,0.9) (1.8,1.4) (2.7,1.95) (3.7,2.55)};
\end{tikzpicture}
\\[1.5mm]
\textbf{Action}
};

\node[loss] (taskloss) at (15.6,0) {\textbf{Future Loss}\\[1mm] $\ell_{\mathrm{task}}(a,\theta)$};

\node[metric] (uncert) at (4,2.1) {\textbf{Posterior Uncertainty}\\[1mm] $\mathcal{U}\!\left(p(\theta\mid\mathcal{D})\right)$};

\draw[->,thick] (design) -- node[above] {rollout} (data);
\draw[->,thick] (data) -- (actionnet);
\draw[->,thick] (actionnet) -- (action);
\draw[->,thick] (action) -- (taskloss);

\draw[->,thick,dashed] (data) -- (uncert);

\draw[->,thick,dashed]
    (uncert.west) -- node[midway,sloped,above,align=center,text width=2.1cm] {\small \textbf{Existing BED Approaches}} (design.north);

\coordinate (lossdown)   at (15.6,-1.4);
\coordinate (actiontap)  at (8.0,-1.4);
\coordinate (designtap)  at (0.0,-1.4);

\draw[thick,dashed,red!70!black]
    (taskloss.south) -- (lossdown) --
    node[midway,below,text=red!70!black]
    {\scriptsize \large \textbf{\textsc{Action-BED:}} Joint policy updates $\nabla_{\phi,\psi}\,\ell_{\mathrm{task}}$}
    (designtap);

\fill[red!70!black] (actiontap) circle (1.6pt);
\draw[->,thick,dashed,red!70!black] (actiontap) -- (actionnet.south);

\fill[red!70!black] (designtap) circle (1.6pt);
\draw[->,thick,dashed,red!70!black] (designtap) -- (design.south);

\end{tikzpicture}
    \caption{
    \textbf{\textsc{Action-BED}}; our task-driven, singly intractable algorithm for policy-based BED. Rather than targeting uncertainty reductions, \textsc{Action-BED} jointly optimises a downstream action policy alongside the design policy to minimise the expected future loss using stochastic gradient descent.}
    \label{fig:method_placeholder}
    \vspace{-1em}
\end{figure}

\section{Background}
\label{back}

\paragraph{Bayesian experimental design.} Bayesian experimental design (BED) \citep{lindley_measure_1956, lindley_1972, fisher1935design, chaloner_bayesian_1995, huan_optimal_2024, rainforth_modern_2024, degroot1962uncertainty, dawid1998coherent} is typically defined as choosing designs to maximise the reduction in a particular uncertainty measure between the prior and the posterior. 
Namely, let $p(\theta)p(y | \theta, \xi)$ denote the joint generating process over outcome $y \in \mathcal{Y}$ and variables of interest $\theta \in \Theta$ given experimental design $\xi \in \Xi$. Then the expected uncertainty reduction (EUR) in $\theta$ for an uncertainty measure $\mathcal{U}$ and design $\xi$ is:
\begin{equation}
\mathrm{EUR}_\theta(\xi)
=
\mathbb{E}_{p(\theta)p(y | \theta, \xi)}
\left[
\mathcal{U}[p(\theta)] - \mathcal{U}[p(\theta | y, \xi)]
\right].
\end{equation}
Typically, $\mathcal{U}$ is taken to be the Shannon entropy, $\mathcal{U}[q(\theta)] = \mathbb{E}_{q(\theta)}[- \log q(\theta)]$, in which case the EUR corresponds to the \emph{expected information gain} (EIG) in $\theta$, i.e.~the mutual information between $\theta$ and $y$.

\paragraph{Sequential and policy-based BED.} BED is often most impactful in sequential settings where one performs an experiment adaptively, say over $T$ rounds. At each round $t$, a design $\xi_t \in \Xi_t$ is proposed and outcome $y_t \in \mathcal{Y}_t$ is observed according to our proposed model $p(y_t | \xi_t, \theta)$. Let $h_{t} := \{(\xi_1,y_1),\dots,(\xi_{t},y_{t})\}$ be the experimental history to time $t$, referred to in generality as the data. Traditional approaches to sequential BED \citep{myung_tutorial_2013} would at each step infer the posterior and select the next design by optimising the EIG, however this process is prohibitively expensive. Instead, policy-based approaches \citep{foster_deep_2021, blau_optimizing_2022, ivanova_implicit_2021, iqbal_nesting_2024, huang_aline_2025, bracher_jadai_2025, shen_bayesian_2023, shen2025variational, dong_variational_2025, hedman_stepdad_2025, huang_amortized_2024,phillips2026bayesian,huang2026efficient} train an adaptive policy network $\pi_d \in \Pi_d$ which maps historical data $h_{t}$ onto optimal next designs $\xi_{t+1}$. The policy is generally trained to maximise (an approximation of) the overall EUR from observing the full experimental history.
This results in an acquisition policy which acts non-myopically to maximise the EUR from observing the full dataset and which can be deployed almost instantaneously.
Note that the optimal design policy is always deterministic~\citep{lindley_1972} and so we will assume deterministic policies throughout, but our methods can equally be applied with probabilistic design policies if desired.

\paragraph{Optimising EIG.} The longstanding challenge with EUR objectives, such as the EIG, is their fundamental double intractability. This arises because the uncertainty measures considered are non-linear functionals of the intractable posterior distribution~\citep{lindley1982scoring,Savage1971Elicitation,dawid2007geometry,gneiting2007strictly}, which are then nested within expectations over the data generating process. 
Much work has been devoted to developing estimation and optimisation procedures for these objectives, in particular for the EIG \citep{ao_estimating_2024, coons_multifidelity_2025, foster_deep_2021, foster2020unified, foster_variational_2019, goda_multilevel_2020, goda_unbiased_2022, huan_gradientbased_2012, huan_simulationbased_2013, iollo_pasoa_2024, iollo_bayesian_2025, iqbal_nesting_2024, ivanova_implicit_2021, kleinegesse_bayesian_2020, kleinegesse_efficient_2019, rainforth_nesting_2018}. Of particular relevance is the family of variational bounds that allow the EIG to be represented as an optimisation over some kind of amortised posterior approximation~\citep{foster_variational_2019, foster2020unified} (which is sometimes indirectly defined through an approximation of the marginal likelihood or its score~\citep{phillips2026bayesian}).
For example, the Barber Agakov (BA) bound is given by~\citep{barber_information_2003}:
\begin{equation}
    \mathcal{I}_T(\pi_d) \geq \mathcal{L}(\pi_d,\psi) = \mathbb{E}_{p(\theta)p(\y | \theta; \pi_d)} \left[\log q_\psi(\theta | h_T) - \log p(\theta)\right]
\end{equation}
where $q_\psi(\theta | h_T)$ is an amortised approximation of the posterior distribution with variational parameters $\psi \in \Psi$. 
If there exists some $\psi \in \Psi$ such that $q_\psi(\theta | h_T)$ perfectly captures the true posterior $p(\theta | h_T)$ for all possible $h_T$, then $\max_{\psi \in \Psi} \mathcal{L}(\pi_d,\psi) = \mathcal{I}_T(\pi_d)$.
This bound then allows a singly intractable joint optimisation of both design policy $\pi_d$ and variational posterior $q_\psi$. However, in practice, it is limited by the variational family and difficulty of learning accurate amortised posteriors.


\section{Loss-Driven Bayesian Experimental Design}

As we have established, the de-facto perspective on BED as maximising EURs leads to objectives which are fundamentally doubly intractable and therefore pose a real computational barrier to their adoption, particularly at scale. Furthermore, variational formulations for circumventing this rely on learning accurate amortised posterior approximations, which is typically challenging and often unnecessary when the posterior itself is not the downstream goal of the experiment. 
They are also specific to the EIG, and there is a lack of more general computationally viable frameworks for other uncertainty measures, in particular customising to downstream tasks and losses of interest.

\looseness=-1
Instead we propose an entirely loss-driven approach, inspired by the principle that the value of data should be measured by how well it allows us to complete downstream tasks. Grounding ourselves in first principles Bayesian decision theory (BDT), we show Bayes-optimal design policies take the form of the 
minimisers of an expected future loss when future actions are made Bayes-optimally. We then show how this can be turned into a practical, computationally tractable BED algorithm, \textsc{Action-BED}, which solves the aforementioned challenges and is additionally loss-driven and implicit, only relying on joint simulation of parameters and data under the proposed model.
It further has the benefit of avoiding the need for full posterior approximation, even after the data has been gathered.


\subsection{Bayesian Experimental Design from First Principles}

\looseness=-1
Bayesian decision theory~\citep{cox1946probability,Savage1951,lindley_1972,degroot1962uncertainty,raiffa2000applied,wald1950statistical,berger_statistical_1985,vonneumann1944theory} dictates that a rational agent should plan ahead in a way that assumes that they will act optimally given the information available to them in the future.
We can thus derive a Bayes-optimal design policy from first principles, by starting with terminal actions taken after all data is gathered and then working backwards through our sequential decisions.


To derive the optimal downstream action given acquired data,
we note that the Savage axioms imply the existence, for any rational agent, of a loss function on actions and the state of the world~\citep{Savage1951},
$\ell: \mathcal{A} \times \Theta \to \mathbb{R}$, which encapsulates our preferences over the consequences from the actions we take, and a subjective belief model on the world given observed data $p(\theta | h_T)$.
Here $\theta \in \Theta$ now denotes whatever aspect of the true world state directly influences our loss function. The Bayes-optimal action, $a^B$, is then the minimiser of the Bayesian expected loss,
$a^B = \argmin_{a \in \mathcal{A}} \E_{p(\theta | h_T)}\left[\ell(a, \theta)\right]$. This is easily generalised to cases where the loss depends on the data $h_T$ itself, but we omit this dependence for ease of notation (see Appendix~\ref{app:action_policy_lifting}). 

For a given belief system, loss function, and action space,
the minimal achievable Bayesian expected loss, $\min_{a \in \mathcal{A}} \E_{p(\theta | h_T)}\left[\ell(a, \theta)\right]$, is now only a function of the data we observe.
It is analogous to a value function in reinforcement learning (RL)~\citep{murphy2024reinforcement}:
better data will provide a more certain belief state and, in turn, a better achievable expected downstream loss.
This downstream value of our data is independent of the policy that was used to gather it~\citep{phillips2026bayesian}, with the policy instead influencing the belief distribution on data that might be generated,
$p(h_T;\pi_d)$.
The existence of this belief distribution is again implied for a rational agent by the Savage axioms, noting that we are uncertain about both $h_T$ and $\theta$ before the experiment.
Further, our beliefs on $\theta$ should not depend on our data gathering policy, so we should have that our ``prior'' $p(\theta)=\int p(h_T;\pi_d) p(\theta | h_T) dh_T$ is independent of $\pi_d$.

As minimising the loss is our ultimate objective and a rational agent will act optimally once the data is gathered, the Bayes-optimal data gathering policy is then the minimiser of the \emph{expected future loss} (EFL) when downstream actions are themselves made Bayes-optimally as follows.

\begin{definition}[Bayes-Optimal Experimental Design Policy]
\label{def:obed_policy}
The Bayes-optimal sequential experimental design policy $\pi_d^B$ for a given belief model and downstream loss pairing is defined as:
\vspace{-2pt}
\begin{equation}
    \label{eqn:min_data_min_a}
    \pi_d^B = \argmin_{\pi_d \in \Pi_d}  \mathrm{EFL}_{\ell}^B(\pi_d) \quad \mathrm{where} \quad \mathrm{EFL}_{\ell}^B(\pi_d) = \E_{p(h_T; \pi_d)}\left[\min_{a \in \mathcal{A}} \E_{p(\theta | h_T)}\left[\ell(a, \theta)\right]\right]. 
\end{equation}
\end{definition}
\vspace{-4pt}
We note that similar theoretical EFL formulations of the Bayes-optimal experimental designs have already been established in the classical BED literature, most notably \citet{lindley_1972} and \citet{dawid1998coherent}.
However, they only considered individual design decisions rather than design \emph{policies} and, moreover, it has not since been directly used for deriving computationally viable BED approaches.
Instead, the BED field has since almost exclusively relied on the mathematical equivalence of this EFL formulation to EURs (see~\Cref{sec:unification}) when constructing practical methods.


\subsection{Singly Intractable and Loss-Driven BED via Joint Policy Optimisation}

While \Cref{def:obed_policy} provides a precise formulation of the Bayes-optimal data gathering policy as the minimiser of a task-driven EFL, directly solving \Cref{eqn:min_data_min_a} is very difficult as it involves a nested optimisation due to the dependence of the optimal action $a^B$ on the observed data $h_T$. Our key insight is that we can switch from optimising individual actions to optimising an \emph{action policy}, $\pi_a$, which maps from observed data to downstream actions. As we show below, this leads to a singly intractable joint objective for design and action policy training which is compatible with any downstream loss.


Let $\pi_a \in\Pi_a := \{\pi \mid \pi: \mathcal{H}_T \to \mathcal{A}\}$, where $\mathcal{H}_T = \prod_{t=1}^T (\mathcal{Y}_t \times \Xi_t)$, be an \emph{action policy} mapping observed data to downstream actions. 
Noting that $\Pi_a$ contains all admissible measurable mappings from the data to actions and the optimal action is defined on a data-by-data basis for a given model and loss, under standard measurability assumptions detailed in Appendix~\ref{app:action_policy_lifting}, we can lift the optimisation over individual actions in the EFL outside the expectation to become an optimisation over action policies as follows:
\begin{align}
    \mathrm{EFL}_{\ell}^B(\pi_d) &= \min_{\pi_a \in \Pi_a} \E_{p(h_T; \pi_d)p(\theta|h_T)} \left[\ell(\pi_a(h_T), \theta)\right].
    \label{eq:efl_bayes_action_policy}
\end{align}
While the minimisation over action policies in~\Cref{eq:efl_bayes_action_policy} does not usually have an analytic solution, we can instead think of the EFL under more general downstream action policies before jointly optimising for both $\pi_a$ and $\pi_d$, which will still recover the Bayes-optimal policies for both actions:
\begin{align}
\begin{split}
    \label{eqn:joint_opt_problem}
    (\pi_a, \pi_d)^B &= \argmin_{(\pi_a, \pi_d) \in \Pi_a \times \Pi_d} \mathrm{EFL}_{\ell}(\pi_d,\pi_a), \\
    \mathrm{where} \quad
    \mathrm{EFL}_{\ell}(\pi_d,\pi_a) &= \E_{p(\theta)p(\y|\theta;\pi_d)} \left[\ell(\pi_a(h_T), \theta)\right],
    \end{split}
\end{align}
and we have rewritten our belief model using $p(\theta)p(\y|\theta;\pi_d)=p(h_T; \pi_d)p(\theta|h_T)$.

We refer to approaches that directly train policies using~\Cref{eqn:joint_opt_problem} as \textbf{\textsc{Action-BED}}.
Its seemingly straightforward reinterpretation of experimental design as a joint optimisation problem now has remarkably impactful consequences. Indeed the EFL represents a completely general, singly intractable, naturally implicit and loss-driven objective for optimal experimental design. It never requires the approximation of posterior distributions or evaluation of unknown densities, it simply relies on simulations from the model. As we will see in~\Cref{meth}, parameterising both action and design policies as deep neural networks leads to a practical end-to-end experimental design algorithm. 

\subsection{Recovering the Expected Uncertainty Reduction Formulation}
\label{sec:unification}
\looseness=-1
Traditional EUR objectives can be formalised and reconciled with our BED formulation through the framework of proper scoring rules. A scoring rule is a function $s: \mathcal{P}_\theta \times \Theta \to \mathbb{R}$ which assigns a score $s(p_\theta, \theta)$ to a predictive distribution $p_\theta \in \mathcal{P}_\theta$ when the realisation of a random variable $\Theta$ is $\theta$. The scoring rule is (strictly) proper if its expectation over $\Theta$ is (strictly) minimised when the predicted distribution is exactly the distribution of $\Theta$. Following ideas introduced by \cite{degroot1962uncertainty, dawid1998coherent,bickfordsmith2025}, we can define general uncertainty measures, known as \emph{generalised entropies}, using $h_s[q(\theta)] = \E_{q(\theta)}[s(q(\cdot), \theta)]$. 
We then define a classical BED objective via the expected posterior uncertainty (EPU):
\begin{equation}
    \text{EPU}_s(\pi_d) = \E_{p(\y ; \pi_d)}\left[h_s[p(\theta | h_T)]\right] = \E_{p(\theta) p(\y | \theta; \pi_d)}\left[s(p_\theta(\cdot | h_T), \theta)\right].
\end{equation}
Assuming the scoring rule does not directly depend on $\pi_d$ or $\y$, the uncertainty of the prior $h_s[p(\theta)]$ is constant in $\pi_d$, so minimising such an EPU is equivalent to maximising the more commonly seen EUR given by $\text{EUR}_{\theta}(\pi_d) = h_s[p(\theta)]-\text{EPU}_s(\pi_d)$, where our uncertainty measure $\mathcal{U}$ is $h_s$.

We now show that via an equivalence between loss-action pairings and proper scoring rules, any optimal BED policy according to \Cref{def:obed_policy} can be represented as the minimiser of an EPU for some scoring rule $s$. Conversely, minimising the EPU associated with any proper scoring rule yields a Bayes-optimal design policy for at least one belief-loss pairing (see Appendix \ref{appendix:equivalence_between_visions} for proof).
\begin{theorem}[Equivalence of EFL \& EPU]
    \label{thm:equiv_BED}
    Let
    $s : \mathcal{P}_{\theta} \times \Theta \to \bar{\mathbb{R}}$
    be a proper scoring rule on $\Theta$. 
    Then
    \begin{equation}
    \label{eq:bayes_opt_psr}
    \pi_d^B
    \in 
    \arg\min_{\pi_d \in \Pi_d}
    \mathbb{E}_{p(\theta)p(\y | \theta;\pi_d)}
    \big[
    s\big(p_\theta(\cdot | h_T), \theta\big)
    \big]
    \end{equation}
    is a Bayes-optimal design policy for the belief model $p(\theta)p(\y | \theta;\pi_d)$ and for any downstream loss $\ell \in \mathcal L_s$ satisfying the regularity conditions stated in Appendix~\ref{appendix:equivalence_between_visions}, where
    \begin{equation}
    \label{eq:loss_set}
    \mathcal{L}_s
    =
    \Big\{
    \ell :
    \forall q_\theta \in \mathcal{P}_{\theta}, \,
    \inf_{a \in \mathcal{A}} \E_{\theta \sim q_\theta} \big[\ell(a, \theta)\big]
    =
    \E_{\theta \sim q_\theta} \big[s(q_\theta, \theta)\big]
    \;
    \Big\}. 
    \end{equation}
    Furthermore, the set $\mathcal{L}_s$ is never empty, and any admissible loss function $\ell$ induces one or more corresponding proper scoring rules, though not necessarily strictly proper.
\end{theorem}
Consequently, minimising EFL is equivalent to minimising EPU under the conditions of \Cref{thm:equiv_BED}, and in turn maximising EUR provided the scoring rule does not depend on $\pi_d$ or $y_{1:T}$ (see Appendix~\ref{appendix:equivalence_between_visions} for discussions of the case where this does not hold). 
For example, maximising EIG is equivalent to minimising EPU with the log score $s(q_\theta(\cdot), \theta) = -\log q_\theta(\theta)$, for which the canonical downstream loss is the log loss and the optimal action is the Bayesian posterior itself. 

\looseness=-1
\Cref{thm:equiv_BED} thus reconciles expected future loss and future uncertainty as different representations of the same underlying decision problem. 
The key distinction is thus whether we specify the problem by directly specifying uncertainty measures or an explicit downstream task and loss.
The latter makes it easier to tailor objectives for specific downstream tasks, such as targeting the point estimation of some quantity of interest, or incorporating in real-world decision-making, such as business decisions or choosing medical treatments. It also makes clear how choices of uncertainty measures in the EUR framework implicitly target particular downstream tasks and losses, e.g.~EIG targets log loss on probabilistic predictions, expected variance reduction targets point estimation with squared error loss, etc.

\looseness=-1
Critically, though, there are also significant practical distinctions between the viewpoints in terms of their computational viability, with the EFL formulation leading to the general purpose singly intractable problem formulation we introduced in~\Cref{eqn:joint_opt_problem}.
Interestingly, this formulation recovers existing variational formulations~\cite{foster_variational_2019,barber_information_2003} as a special case when using an EFL that is equivalent to the EIG (as here the optimal action is the posterior, see Appendix \ref{app:variational_recovery}).
However, it crucially generalises to other action spaces and losses, and it turns out the action policies we need to learn for other downstream losses are often much simpler than those required for the EIG.
In particular, when the required downstream actions correspond to point estimates or explicit decisions rather than distributions, then the action policy is a simple functional mapping to single values, avoiding the need to implicitly estimate or approximate the posterior.
For example, if using squared error loss on an estimation task, then the action policy only needs to learn to map from data to the posterior mean.
This can substantially simplify the policy training as we observe in our experiments (see~\Cref{exp}).
\section{\textsc{Action-BED}}
\label{meth}

\begin{wrapfigure}{r}{0.53\columnwidth}
\vspace{-24pt}
\begin{tcolorbox}
\begin{algorithmic}[1]
\vspace{-4pt}
\Require Policies $\pi_d^\phi, \pi_a^\psi$, initialisations $\phi_0, \psi_0$
\For{$k = 0,1,\dots,K-1$}
    \State $\theta^{1:B}, \varepsilon_{1:T}^{1:B} \sim p(\theta)q(\varepsilon_{1:T})$ 
    \State $g^\phi_k, g_k^\psi \gets \widehat{\nabla_{\phi, \psi}} {\text{EFL}}(\phi_k, \psi_k)$  (\cref{eqn:mc_gradient_est})
    \State $\phi_{k+1}, \psi_{k+1}  \gets \textsc{Update}(\phi_k, g^\phi_k, \psi_k, g^\psi_k)$
\EndFor
\State \Return $\phi_K, \psi_K$
\end{algorithmic}
\captionsetup{type=algorithm}
\captionof{algorithm}{\textsc{Action-BED} \vspace{-3pt}}
\label{alg:main_alg}
\end{tcolorbox}
\vspace{-20pt}
\end{wrapfigure}
We now show how our \textsc{Action-BED} framework, based on the joint optimisation of the EFL with respect to action and design policies as per \Cref{eqn:joint_opt_problem}, can be realised as a practical end-to-end algorithm for task-driven BED.
Specifically, building on the deep adaptive design (DAD) approach of~\citep{foster_deep_2021,ivanova_implicit_2021}, we propose to directly parameterise both policies using deep neural networks,  denoted $\pi_d^\phi$ and $\pi_a^\psi$, with parameters $\phi\in\Phi$ and $\psi\in\Psi$.
We then train these networks using end-to-end stochastic gradient descent on the expected future loss for rollouts drawn from the model, as shown in \Cref{alg:main_alg}. Empirically, it is often helpful to initialise $\psi \in \Psi$ by pre-training using a fixed random design policy (Appendix~\ref{app:downstream_warmup}), following which our algorithm remains unchanged (with $\psi_0$ set to this learned initialisation).
If $p(\y|\theta; \pi_d^\phi)$ is suitably reparameterisable, we can optimise this objective using pathwise gradients \citep{mohamed_monte_2020} obtained via reparameterisation \citep{kingma_autoencoding_2013}:
\begin{equation}
    \label{eqn:EFL_reparam_gradient}
    \nabla_{\phi, \psi} \text{EFL}(\phi,\psi)
    =
    \mathbb{E}_{p(\theta)q(\varepsilon_{1:T})}
    \left[ \nabla_{\phi, \psi}
        \ell\big(
        \pi_a^\psi(h_T(\theta,\varepsilon_{1:T};\pi_d^\phi)),
        \theta
        \big)
    \right],
\end{equation}
where we reparameterised $h_T = h_T(\theta,\varepsilon_{1:T};\pi_d^\phi)$ for $\varepsilon_{1:T}\sim q(\varepsilon_{1:T})$. See Appendix~\ref{app:reparam} for full details. We can then estimate these gradients unbiasedly using simple Monte Carlo:
\begin{equation}
\label{eqn:mc_gradient_est}
\widehat{\nabla_{\phi,\psi}}\,\text{EFL}(\phi,\psi)
=
\frac{1}{B} \sum_{b=1}^B\nabla_{\phi,\psi}
\ell\big(\pi_a^\psi(h_T(\theta^b,\varepsilon_{1:T}^b;\pi_d^\phi)),
\theta^b\big)
\end{equation}  
for $\theta^b \sim p(\theta), \varepsilon_{1:T}^b \sim q(\varepsilon_{1:T})$, from which stochastic gradient updates can be performed using any desired optimizer. 
If the simulator is not reparameterisable, we can instead resort to using the REINFORCE estimator \citep{williams_simple_1992a} to estimate gradients, provided that we have access to the score function of the data generating process. These requirements are the main restrictions of our approach, although we note that a very large class of BED problems satisfy at least one of these assumptions. Details on the architectures used for the policies are provided in Appendix~\ref{appendix:architectures}.



\section{Related Work}
\label{sec:related_work}

\looseness=-1
There are limited prior works considering truly task-driven approaches to BED. One approach is \citet{huang_amortized_2024}, who maximise the improvement in expected utility of the optimal downstream action after observing new data. They only consider utility functions on actions living in outcome space $\mathcal{Y}$, meaning their method cannot be applied as a baseline in any of our BED experiments and is instead applied on prediction-oriented active learning settings \citep{smith_predictionoriented_2023}. Furthermore, they maximise their expected utility objective by approximating the intractable predictive distribution and explicitly solving the nested optimisation problem over downstream actions. In contrast, our approach does not require any distributional approximations and removes the nested optimisation by learning action policies. Short of being truly task-based, recent works \citep{filstroff_targeted_2024, zhong_goaloriented_2026, goaloriented, shen2025variational, kleinegesse_gradientbased_2021, chakraborty_likelihood-free_2024, smith_predictionoriented_2023} consider ``goal-oriented'' experimental design, in which designs are sought to maximise the EUR in downstream quantities of interest. While this approach shifts focus from model parameters to quantities of interest, it is still fundamentally an EUR objective and therefore inherits the associated computation challenges. Finally, \citet{kerrigan_geometric_2025} introduces a new perspective on BED, maximising the dependence between $y$ and $\theta$ by solving an optimal transport problem with a chosen cost function. The choice of cost function provides a way of tailoring designs to downstream tasks of interest, but cannot be interpreted as minimising a loss over downstream actions in a principled Bayesian decision-theoretic perspective. 


In alternative contexts, \citet{huang_lossdriven_2026} introduces a task-driven framework for active learning, 
which they solve analytically when the downstream loss is a Bregman divergence. 
This still leads to doubly intractable objectives, whereas our approach is instead singly intractable for \emph{any} downstream loss. In a Bayesian optimisation (BO) setting, \citet{neiswanger_generalizing_2022} frame classical BO acquisition functions in a loss-based perspective via generalised entropies, which they use to derive new acquisition functions for custom BO tasks. In their case, analytic solutions depend on the Gaussian process posterior, and are thus specific to a particular class of models.

\looseness=-1
Our work also shares strong connections with reinforcement learning (RL): $\min_{a \in \mathcal{A}} \E_{p(\theta|h_T)}\left[\ell(a, \theta)\right]$ can be interpreted as a value function giving the best achievable downstream loss given the current belief state. Prior works \citep{blau_optimizing_2022, blau_cross-entropy_2023, shen_bayesian_2023, barlas_performance_2025, shen2025variational} have adopted an RL approach to BED by defining rewards as the incremental improvement in the above value function. However, this has only been done for the EIG, implicitly choosing the log-loss to define the value function, while the generalisation to any downstream loss has not yet been made. Furthermore, while variational and amortised inference approaches \citep{foster2020unified, dong_variational_2025, shen2025variational, huang_aline_2025, bracher_jadai_2025} learn amortised variational approximations which can be interpreted as downstream action networks (often within RL frameworks e.g. \citet{shen2025variational}), the generalisation to arbitrary downstream actions has also not yet been made. 
Nonetheless, it is possible our approach could directly benefit from inheriting RL ideas, such as reward shaping and credit assignment, in the future.

\section{Experiments}
\label{exp}

\looseness=-1
We evaluate \textbf{\textsc{Action-BED}} across four BED tasks spanning a range of complementary settings.  
Our experiments illustrate the improved performance and computational efficiency of \textsc{Action-BED} compared to baselines. 
Further evaluations of \textsc{Action-BED} in~\citep{kerrigan2026testbed} corroborate these findings.



\paragraph{Baselines and evaluation protocol.}
We compare against three competitive policy-based BED baselines. \textbf{\textsc{DAD}}~\citep{foster_deep_2021} serves as the standard EIG-based policy baseline when an explicit likelihood is available, while \textbf{\textsc{iDAD}}~\citep{ivanova_implicit_2021} provides its likelihood-free counterpart when required. \textbf{\textsc{ALINE}} \citep{huang_aline_2025} is a state-of-the-art EIG-based approach using a transformer as a joint amortised design and inference network, trained using reinforcement learning with a dense reward. Additionally we include a \textbf{Random} policy which samples designs from a task-specific proposal. In order to evaluate these baselines on downstream actions, we train a design policy using the baseline approach, which is then frozen while a downstream action network is trained to minimise the downstream task loss on trajectories generated by the pre-trained design policy. Thus, the design network of each baseline is trained to maximise EIG while the downstream network is trained to minimise downstream loss by approximating the Bayes-optimal downstream action. 
Methods are evaluated in terms of downstream loss, design policy EIG bounds \citep{foster_deep_2021}, and end-to-end wall-clock GPU training time.

\paragraph{Training robustness.} To confirm that our relative performance gains are not driven by architectural choices, we matched architectures across all baselines and further repeated our experiments using smaller architectures in Appendix~\ref{appendix:light_archi}. While our main results are reported for a single training seed, we further assess training-seed variability for \textsc{Action-BED} and \textsc{DAD} with the main architecture on the source location finding and single pendulum tasks; see Tables~\ref{tab:source_main_seeds} and~\ref{tab:pendulum_main_seeds}. These experiments, together with the multi-seed analysis using smaller architectures in Appendix~\ref{app:training_stability_seeds}, show that the performance trends are stable across training seeds. 

\subsection{Source Location Finding}
\label{sec:slf_exp}

\begin{wrapfigure}{r}{0.50\textwidth}
    \centering
    \vspace{-2.5em}
    \includegraphics[width=0.43\textwidth]{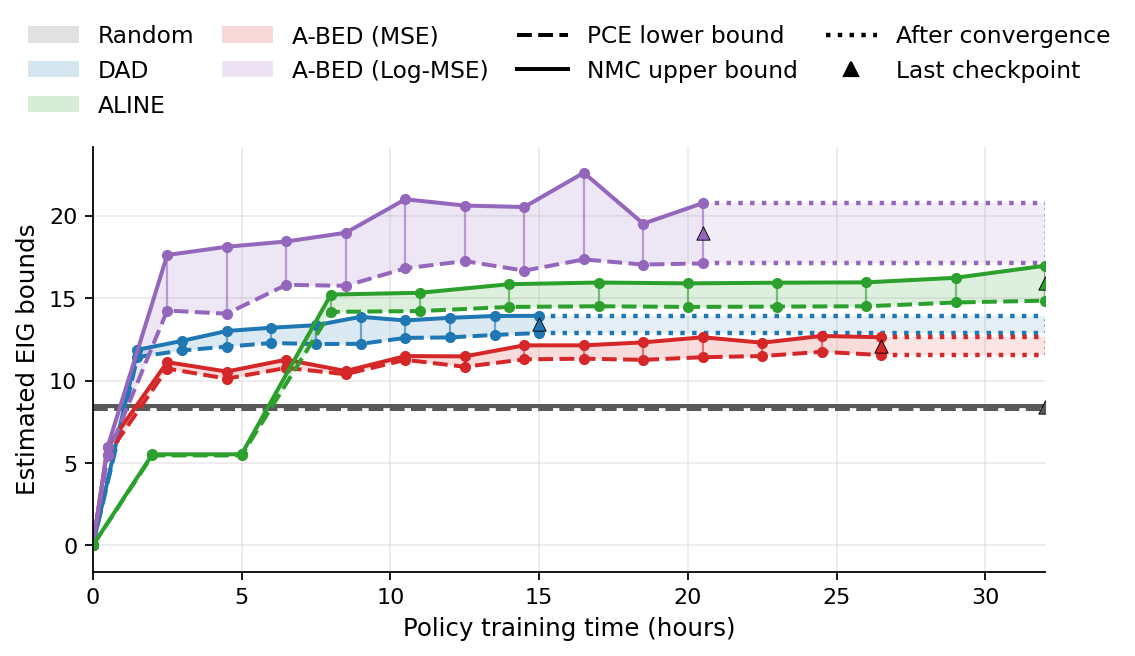}
    \caption{Comparison of EIG bounds as a function of design policy training time, until convergence.}
    \label{fig:eig_bounds_time}
    \vspace{-1.5em}
\end{wrapfigure}

\looseness=-1
The first experiment is the common source location finding experiment introduced by~\citep{foster_deep_2021} and inspired by the acoustic energy attenuation model introduced in \cite{location}. The unknown parameter $\theta \in \mathbb{R}^{2 \times 2}$ specifies the locations of two sources which emit signals that decay with distance. Over $T=30$ steps, the design policy sequentially selects locations and observes the resulting aggregate signal intensity. The goal is to learn an adaptive design strategy that enables efficient downstream prediction of source locations. 
For the downstream predictive loss, we consider the mean squared error (MSE) and the quadratic-logarithmic loss (QLL), a coordinate-wise variant that augments each squared prediction error with a logarithmic penalty before averaging across coordinates. This transformation places greater emphasis on reducing already small errors, thereby encouraging precise localisation. Both losses are made permutation-invariant over the sources. Implementation and experimental details are provided in Appendix~\ref{appendix:locationfinding}. 

\begin{table}[t]
\centering
\scriptsize
\setlength{\tabcolsep}{8.0pt}
\renewcommand{\arraystretch}{1.15}

\caption{\small
\textbf{Source Location Finding} final achieved expected losses. Error bars are reported as one standard error over $2048$ rollouts.
Parentheses indicate the loss used to train the downstream action network.
Compute time includes training both the design and downstream action networks.
}
\label{tab:location_table}

\begin{tabular}{lcccccc}
\toprule
\multirow{2}{*}{\textbf{Method}}
& \multicolumn{4}{c}{\textbf{Performance ($\pm$ 1 s.e.)}}
& \multicolumn{1}{c}{\textbf{Compute}} \\
\cmidrule(lr){2-6}
\cmidrule(lr){7-7}
& MSE $(10^{-2})$ $\downarrow$
& QLL $\downarrow$
& sPCE $\uparrow$
& sNMC $\uparrow$
& Time (min) $\downarrow$ \\
\midrule

Action-BED (MSE)
& $0.55 \pm 0.06$
& $-6.83 \pm 0.03$
& $11.63 \pm 0.04$
& $12.84 \pm 0.08$
& 1171 \\

Action-BED (QLL)
& $2.8 \pm 0.4$
& $\mathbf{-10.10 \pm 0.03}$
& $\mathbf{17.14 \pm 0.11}$
& $\mathbf{21.07 \pm 0.45}$
& 1184 \\

\addlinespace[1pt]
\midrule

DAD (MSE)
& $5.6 \pm 0.6$
& $-6.35 \pm 0.04$
& $12.94 \pm 0.03$
& $16.35 \pm 0.14$
& 1214 \\

DAD (QLL)
& $10.7 \pm 0.8$
& $-6.58 \pm 0.05$
& $12.94 \pm 0.03$
& $16.35 \pm 0.14$
& 1216 \\

\addlinespace[1pt]
\midrule

ALINE (MSE)
& $\mathbf{0.21 \pm 0.03}$
& $-7.58 \pm 0.05$
& $14.83 \pm 0.05$
& $16.97 \pm 0.13$
& 2962 \\

ALINE (QLL)
& $0.81 \pm 0.10$
& $-8.11 \pm 0.07$
& $14.83 \pm 0.05$
& $16.97 \pm 0.13$
& 2970 \\

\addlinespace[1pt]
\midrule

Random
& $22.4 \pm 0.8$
& $-3.16 \pm 0.04$
& $8.27 \pm 0.04$
& $8.49 \pm 0.05$
& \textbf{31} \\

\bottomrule
\end{tabular}

\vspace{-2em}
\end{table}

\paragraph{Results.} 
Table~\ref{tab:location_table} shows the quantitative results on the source location finding task while qualitative design trajectories are included in Appendix~\ref{appendix:qualitative_results_lf}. We firstly observe the impressive performance of \textsc{Action-BED} under the QLL objective; we achieve both the best downstream QLL error as well as state-of-the-art EIG of the design policy. We believe that this is possible because the QLL remains sensitive to very accurate predictions and continues to reward design policies that identify the sources more precisely, whereas, for example, the sPCE contrastive lower bound saturates at high EIG values therefore becoming harder to train to. A quantile analysis in Table~\ref{tab:location_loss_quantiles} further shows that QLL training yields substantially lower typical localisation errors but a heavier upper tail, explaining its higher mean MSE. Utilising this informative loss function and sidestepping the challenges of EIG-based training is made directly possible by our approach. Under the MSE loss, \textsc{ALINE} performs best however requires around 2.5 times the compute to achieve such a result. A possible explanation for its better performance here is that its dense training rewards help identify the latent sources more efficiently, while such guidance is absent in our setting. \textsc{Action-BED} significantly outperforms \textsc{DAD} throughout, which uses a similar amount of compute. 
In terms of computational efficiency, we show in \Cref{fig:eig_bounds_time} the attained EIG of the design policy versus training time, where \textsc{Action-BED} (QLL) achieves competitive EIG values faster than competing approaches.

\paragraph{Additional ablations and evaluation.} 
We provide several additional analyses in Appendix~\ref{app:amortised_posterior_actions} and~\ref{appendix:additional_results_lf}. We first evaluate the decision-relevant quality of learned design policies independently of their downstream action networks using a common posterior-based decision procedure. This isolates the informativeness of the acquired histories themselves and confirms the objective-dependent ranking observed in the main results. Moreover, it highlights the substantial deployment-time cost avoided by amortised decision-making with \textsc{Action-BED}; see Tables~\ref{tab:deployment_fixed_policy_light} and~\ref{tab:posterior_action_runtime}. We then investigate the coupling between the jointly learned design policy and downstream action policy in \textsc{Action-BED} by evaluating each design and action network cross-pairing; see Table~\ref{tab:cross_policy_coupling_normalized}. 
Downstream performance degrades substantially under such policy swaps, showing that the action network exploits not only the observed outcomes but also information encoded in the adaptive design trajectory itself, and can be misled when data is collected by a different policy. This coupling is strongest for the most accurate policies, whose designs concentrate efficiently near the latent sources. We further show for \textsc{Action-BED} that retraining the action policy on a fixed pre-trained design policy leads to a mild downstream degradation compared to joint training, suggesting that the joint training dynamics beneficially co-adapt the two policies beyond just ensuring their compatibility; see Table~\ref{tab:native_vs_retrained_abed}. Finally, we ablate the warmstart training of the action network with a random design policy before joint optimisation; see Table~\ref{tab:warmstart_abed}. Removing it slightly worsens downstream performance, likely because early low-quality action predictions provide a noisier training signal to the design policy, perturbing the joint optimisation dynamics. 

\subsection{Dynamical Systems}

We next consider stochastic single- and double-pendulum systems \citep{Belousov2019, iqbal_nesting_2024}. In these tasks, the system evolves over time according to a controlled SDE depending on some unknown physical parameters. The continuous dynamics are discretised using Euler-Maruyama over $T=50$ steps with $dt=0.05$. At each step an agent influences the dynamics by applying a torque $\xi_t$. The goal is to learn an adaptive design policy that enables accurate downstream prediction of the latent physical parameters $\theta \in \Theta$ from the observed trajectories. We evaluate the prediction with several downstream objectives---MSE, QLL, and a weighted MSE over the components of $\theta$---the latter requiring the design policy to prioritise parameter regions of particular interest. These tasks pose three main challenges: the data-generating process is non-exchangeable due to Markovian dynamics, designs must satisfy physical constraints, and longer rollouts induce complex feedback between interventions and observations.

\paragraph{Stochastic Pendulum. } This task comprises a single pendulum with state \(y_t=(q_t,\dot q_t)\), unknown parameters \(\theta \in \Theta \subseteq \mathbb{R}^3\), and bounded control torques \(\xi_t \in [-1,1]\).  Full experimental details are provided in Appendix~\ref{appendix:stopendulum}.
Table~\ref{tab:pendulum_downstream_losses} confirms that loss-driven acquisition improves the corresponding downstream objective: each variant performs best under the loss used for training. Interestingly, we observe again that the design policy of \textsc{Action-BED} (QLL) achieves the highest EIG, outperforming the baselines which directly optimise EIG. We also observe that our method requires substantially less compute than the baselines. 
It is also notable here that ALINE significantly underperforms even compared to DAD (except on QLL), likely because this problem requires more planning ahead than the earlier location finding experiment.



\begin{table}[t]
\centering
\scriptsize
\setlength{\tabcolsep}{4.5pt}
\renewcommand{\arraystretch}{1.15}

\caption{
\small \textbf{Stochastic Single Pendulum} final achieved expected losses. Error bars give $\pm$ one standard error over $2048$ rollouts.
Compute time includes training both the design and downstream action networks.
}
\label{tab:pendulum_downstream_losses}

\begin{tabular}{lcccccc}
\toprule
\multirow{2}{*}{\textbf{Method}}
& \multicolumn{5}{c}{\textbf{Performance ($\pm$ 1 s.e.)}}
& \multicolumn{1}{c}{\textbf{Compute}} \\
\cmidrule(lr){2-6}
\cmidrule(lr){7-7}
& MSE $(10^{-3})$ $\downarrow$
& QLL $\downarrow$
& W. MSE $(10^{-3})$ $\downarrow$
& sPCE $\uparrow$
& sNMC $\uparrow$
& Time (min) $\downarrow$ \\
\midrule

Action-BED (MSE)
& $\mathbf{6.00 \pm 0.14}$
& $-6.60 \pm 0.05$
& $4.90 \pm 0.13$
& $3.68 \pm 0.04$
& $3.71 \pm 0.04$
& 240 \\

Action-BED (QLL)
& $7.20 \pm 0.17$
& $\mathbf{-6.75 \pm 0.05}$
& $4.80 \pm 0.11$
& $\mathbf{3.99 \pm 0.04}$
& $\mathbf{4.00 \pm 0.04}$
& 287 \\

Action-BED (W. MSE)
& $9.80 \pm 0.25$
& $-6.07 \pm 0.04$
& $\mathbf{3.50 \pm 0.09}$
& $3.33 \pm 0.04$
& $3.36 \pm 0.04$
& 290 \\

\addlinespace[1pt]
\midrule

DAD (MSE)
& $7.20 \pm 0.17$
& $-6.20 \pm 0.04$
& $5.00 \pm 0.14$
& $3.82 \pm 0.04$
& $3.88 \pm 0.04$
& 670 \\

DAD (QLL)
& $8.10 \pm 0.19$
& $-6.49 \pm 0.05$
& $5.20 \pm 0.14$
& $3.82 \pm 0.04$
& $3.88 \pm 0.04$
& 709 \\

DAD (W. MSE)
& $7.30 \pm 0.17$
& $-6.44 \pm 0.05$
& $4.40 \pm 0.12$
& $3.82 \pm 0.04$
& $3.88 \pm 0.04$
& 671 \\

\addlinespace[1pt]
\midrule

ALINE (MSE)
& $23.50 \pm 0.51$
& $-5.43 \pm 0.02$
& $7.59 \pm 0.12$
& $2.23 \pm 0.06$
& $2.24 \pm 0.06$
& 2120 \\

ALINE (QLL)
& $27.53 \pm 0.59$
& $-5.69 \pm 0.02$
& $9.89 \pm 0.17$
& $2.23 \pm 0.06$
& $2.24 \pm 0.06$
& 2141 \\

ALINE (W. MSE)
& $23.58 \pm 0.50$
& $-5.61 \pm 0.02$
& $7.06 \pm 0.11$
& $2.23 \pm 0.06$
& $2.24 \pm 0.06$
& 2122 \\

\addlinespace[1pt]
\midrule

Random
& $29.0 \pm 1.2$
& $-4.25 \pm 0.05$
& $12.6 \pm 1.0$
& $1.42 \pm 0.03$
& $1.50 \pm 0.03$
& \textbf{19} \\

\bottomrule
\end{tabular}

\vspace{-2em}
\end{table}

\paragraph{Stochastic Double-Pendulum. } The double-pendulum task is substantially harder as its dynamics are nonlinear, coupled, and potentially chaotic, making trajectories highly sensitive to designs. 
The state \(y_t \in \mathbb{R}^4\) contains the two joint angles and angular velocities, while \(\theta \in \Theta \subseteq \mathbb{R}^4\) contains the unknown masses and link lengths. Designs are bounded joint torques \(\xi_t=(\xi_t^1,\xi_t^2) \in [-4,4]\times[-2,2]\). Details are given in Appendix~\ref{appendix:doublependulum}. In this setting, Table~\ref{tab:double_pendulum_downstream_losses} again shows that \textsc{Action-BED} achieves considerably improved downstream prediction for reduced computational cost. In this task, \textsc{DAD} appears to achieve marginally higher EIG although EIG bounds overlap with \textsc{Action-BED}.
ALINE again significantly underperfoms compared to the other methods except on QLL.


\begin{table}[t]
\centering
\scriptsize
\setlength{\tabcolsep}{4.0pt}
\renewcommand{\arraystretch}{1.15}

\begin{threeparttable}
\caption{
\small \textbf{Stochastic Double Pendulum} final achieved expected losses. Error bars give $\pm$ one standard error over $2048$ rollouts.
Compute time includes training both the design and downstream action networks.
}
\label{tab:double_pendulum_downstream_losses}

\begin{tabular}{lcccccc}
\toprule
\multirow{2}{*}{\textbf{Method}}
& \multicolumn{5}{c}{\textbf{Performance ($\pm$ 1 s.e.)}}
& \multicolumn{1}{c}{\textbf{Compute}} \\
\cmidrule(lr){2-6}
\cmidrule(lr){7-7}
& MSE $(10^{-4})$ $\downarrow$
& QLL $\downarrow$
& W. MSE $(10^{-4})$ $\downarrow$
& sPCE $\uparrow$
& sNMC $\uparrow$
& Time (min) $\downarrow$ \\
\midrule

Action-BED (MSE)
& $\mathbf{1.11 \pm 0.03}$
& $-10.21 \pm 0.04$
& $1.64 \pm 0.04$
& $10.73 \pm 0.05$
& $12.71 \pm 0.19$
& 501 \\

Action-BED (QLL)
& $1.31 \pm 0.03$
& $\mathbf{-10.32 \pm 0.04}$
& $1.75 \pm 0.05$
& $10.81 \pm 0.12$
& $14.29 \pm 0.56$
& 504 \\

Action-BED (W. MSE)
& $1.14 \pm 0.03$
& $-10.28 \pm 0.04$
& $\mathbf{1.51 \pm 0.05}$
& $10.65 \pm 0.04$
& $12.10 \pm 0.16$
& 504 \\

\addlinespace[1pt]
\midrule

DAD (MSE)
& $2.10 \pm 0.06$
& $-9.50 \pm 0.03$
& $2.47 \pm 0.07$
& $\mathbf{11.37 \pm 0.03}$
& $\mathbf{14.79 \pm 0.25}$
& 934 \\

DAD (QLL)
& $2.43 \pm 0.10$
& $-9.51 \pm 0.03$
& $2.79 \pm 0.08$
& $\mathbf{11.37 \pm 0.03}$
& $\mathbf{14.79 \pm 0.25}$
& 940 \\

DAD (W. MSE)
& $2.13 \pm 0.06$
& $-9.47 \pm 0.03$
& $2.46 \pm 0.07$
& $\mathbf{11.37 \pm 0.03}$
& $\mathbf{14.79 \pm 0.25}$
& 933 \\

\addlinespace[1pt]
\midrule

ALINE (MSE)
& $11.21 \pm 0.34$
& $-8.46 \pm 0.02$
& $15.12 \pm 0.04$
& $9.28 \pm 0.07$
& $9.78 \pm 0.14$
& 3013 \\

ALINE (QLL)
& $19.21 \pm 0.51$
& $-8.53 \pm 0.03$
& $23.29 \pm 0.07$
& $9.28 \pm 0.07$
& $9.78 \pm 0.14$
& 3062 \\

ALINE (W. MSE)
& $11.81 \pm 0.33$
& $-8.52 \pm 0.03$
& $14.02 \pm 0.04$
& $9.28 \pm 0.07$
& $9.78 \pm 0.14$
& 3007 \\

\addlinespace[1pt]
\midrule

Random
& $20.0 \pm 0.9$
& $-6.84 \pm 0.05$
& $48.0 \pm 1.3$
& $8.26 \pm 0.06$
& $8.41 \pm 0.08$
& \textbf{58} \\

\bottomrule
\end{tabular}

\end{threeparttable}
\vspace{-2em}
\end{table}

\subsection{MNIST Masked Classification}
Finally, we consider a sequential masked-classification task adapted from the experimental setting in \citep{iollo_bayesian_2025}. Each episode begins with an unobserved image \(\theta \in \mathbb{R}^{28 \times 28}\) from MNIST \citep{lecun1998gradient}, together with its class label \(z \in \{0,\ldots,9\}\). At each of the \(T=5\) acquisition steps, the policy selects the location \(\xi \in [1,28]^2\) of a \(5\times5\) pixel patch and observes a noisy measurement \(y = A_{\xi}\theta + \eta\), with \(\eta \sim \mathcal{N}(0,\sigma^2 I_d)\) and \(A_{\xi}\) the corresponding masking operator. The goal is to select the patches most informative for downstream digit classification. This setting is challenging because acquisition is driven by the class label, whereas measurements arise from the latent image. The patch likelihood \(p(y | \theta,\xi)\) is explicit, but its label-conditioned counterpart \(p(y | z,\xi)\) is only implicitly induced by the MNIST class distribution. This problem is also strongly information-constrained, since the downstream policy must infer the class from only a few noisy local measurements in a high-dimensional space. Additional details are provided in Appendix~\ref{appendix:mnist_classif_exp}.

\textsc{Action-BED} naturally targets this implicit problem by jointly training a design policy and downstream classifier to minimise the cross-entropy loss:
\looseness=-1
\begin{equation*}
    \min_{\phi, \psi} \mathbb{E}_{p(z)p(\theta|z)p(\y|\theta; \pi_d^\phi)}\left[\ell_{\text{CE}}\left(\pi_a^\psi(\cdot | h_T), z\right)\right]
\end{equation*}
\looseness=-1
where $\ell_{\text{CE}}\left(\pi_a^\psi(\cdot | h_T), z\right) = - \log \pi_a^\psi(z | h_T)$ is the cross-entropy loss and $\pi_a^\psi(\cdot | h_T)$ is a probabilistic classifier. Training to this objective only requires simulating $z, \theta$ from the empirical MNIST dataset and reparameterising the simulation of $p(\y | \theta; \pi_d)$. By contrast, \textsc{DAD} must target EIG in the latent image \(\theta\), since the outcomes conditional on class label $p(y | z, \xi)$ do not have explicit likelihood. \textsc{iDAD} is able to target the EIG in the label $z$, thereby providing a more competitive baseline. 
The original \textsc{ALINE} setup does not directly apply in this setting due to the hierarchical likelihood and classification target, we therefore adapted its inference network setup accordingly. 
Quantitative and qualitative results are reported respectively in Table~\ref{tab:mnist_classification_task} and \Cref{fig:four_policy_rollouts_mnist}.
\textsc{Action-BED} achieves the best downstream performance across the board, demonstrating the lowest cross-entropy loss and highest classification accuracy amongst all approaches, whilst requiring the least training time (other than Random). These results demonstrate that \textsc{Action-BED} easily maximises downstream performance in implicit models without additional modifications, critic networks, or internal posterior approximations.

\begin{table}[t]
\caption{\small \textbf{MNIST Masked Classification}. Error bars give $\pm$ one standard error across evaluation batches. We report cross-entropy and classification accuracy on both the training and test datasets. GPU time includes training of both the design and downstream classification networks.}
\label{tab:mnist_classification_task}
\centering
\scriptsize
\setlength{\tabcolsep}{6pt}
\renewcommand{\arraystretch}{1.12}
\begin{tabular}{@{}lccccc@{}}
\toprule
& \multicolumn{2}{c}{\textbf{Train dataset}}
& \multicolumn{2}{c}{\textbf{Test dataset}}
& \textbf{Compute} \\
\cmidrule(lr){2-3}
\cmidrule(lr){4-5}
\cmidrule(l){6-6}
\textbf{Method}
& Cross-entropy
& Accuracy (\%)
& Cross-entropy
& Accuracy (\%)
& GPU time (min) \\
\midrule

Action-BED
& $\mathbf{0.033 \pm 0.001}$
& $\mathbf{98.87 \pm 0.04}$
& $\mathbf{0.097 \pm 0.007}$
& $\mathbf{97.42 \pm 0.16}$
& 9 \\

\midrule

DAD
& $0.350 \pm 0.004$
& $83.32 \pm 0.04$
& $0.546 \pm 0.009$
& $81.83 \pm 0.16$
& 53 \\

iDAD
& $0.099 \pm 0.002$
& $96.49 \pm 0.08$
& $0.164 \pm 0.008$
& $95.10 \pm 0.21$
& 12 \\

ALINE
& $ 0.256 \pm 0.008 $
& $ 92.13 \pm 0.08 $
& $ 0.271 \pm 0.01 $
& $ 91.59 \pm 0.10 $
& 357 \\

\midrule

Random
& $0.786 \pm 0.015$
& $71.78 \pm 0.87$
& $0.797 \pm 0.015$
& $71.40 \pm 0.86$
& \textbf{5} \\

\bottomrule
\end{tabular}
\vspace{-1.2em}
\end{table}


\section{Conclusion}
\label{conclusion}

In this work, we proposed a paradigm shift in how we think about BED from maximising expected uncertainty reductions (EURs) to minimising the expected future loss (EFL) achieved on downstream tasks of interest.
We showed that this EFL formulation follows from first-principles decision theory and recovers the existing uncertainty-based BED construction as an equivalent representation.
Critically, though, we showed that the EFL formulation also provides a distinct beneficial mechanism for learning corresponding design policies in practice: by jointly learning a design network and a downstream \emph{action network}, we can recover \emph{singly intractable objectives} for minimising the EFL, whereas EUR formulations are generally doubly intractable.
We call this joint training approach \textsc{Action-BED}, and show that it provides a practical and effective algorithm for task-driven BED across a variety of problems.
\textsc{Action-BED} further provides the benefits of naturally allowing implicit models and avoiding ever needing to approximate the full posterior.
As such, \textsc{Action-BED} further allows us to fully amortise the \emph{entire} decision-making process: not only does it inherit the near-instantaneous design decisions of previous policy-based BED approaches, it further allows the final actions to be chosen with a single forward pass of the learned action network. 

\section*{Acknowledgements and Disclosures}

\paragraph{Funding.} Tom Rossa is supported by the EPSRC CDT in Statistics and Machine Learning (EP/Y034813/1), and by GSK Research \& Development Limited through a doctoral studentship. 
Angus Phillips is supported by the EPSRC CDT in Modern Statistics and Statistical Machine Learning (EP/S023151/1).
Tom Rainforth is supported by the EPSRC grant EP/Y037200/1 and receives additional research funding from GSK Research \& Development Limited.

\paragraph{Competing Interests.} The authors declare no additional financial relationships that could reasonably be perceived as influencing this work.

\bibliographystyle{apalike}
\bibliography{references}

@book{Savage1951,
  author    = {Savage, L. J.},
  title     = {The theory of statistical decision},
  publisher = {Journal of the American Statistical Association},
  year      = {1951},
  volume    = {46},  
  page      = {55–67},
}

@article{Savage1971Elicitation,
  author    = {Leonard J. Savage},
  title     = {Elicitation of Personal Probabilities and Expectations},
  journal   = {Journal of the American Statistical Association},
  year      = {1971},
  volume    = {66},
  number    = {336},
  pages     = {783--801},
  doi       = {10.1080/01621459.1971.10482360}
}

@book{vonneumann1944theory,
  title     = {Theory of Games and Economic Behavior},
  author    = {von Neumann, John and Morgenstern, Oskar},
  year      = {1944},
  publisher = {Princeton University Press},
  address   = {Princeton, NJ}
}

@inproceedings{bickfordsmith2025,
  title        = {Rethinking aleatoric and epistemic uncertainty},
  author       = {Bickford Smith, Freddie and Kossen, Jannik and Trollope, Eleanor and Van Der Wilk, Mark and Foster, Adam and Rainforth, Tom},
  booktitle    = {Proceedings of the 42nd International Conference on Machine Learning},
  series       = {Proceedings of Machine Learning Research},
  volume       = {267},
  pages        = {4345--4359},
  year         = {2025},
  publisher    = {PMLR}
}

@misc{goaloriented,
      title={Goal-Oriented Sequential Bayesian Experimental Design for Causal Learning}, 
      author={Zheyu Zhang and Jiayuan Dong and Jie Liu and Xun Huan},
      year={2025},
      eprint={2507.07359},
      archivePrefix={arXiv},
      primaryClass={cs.LG},
      url={https://arxiv.org/abs/2507.07359}, 
      note={arXiv preprint arXiv:2507.07359}
}

@InProceedings{foster2020unified,
  title = 	 {A Unified Stochastic Gradient Approach to Designing Bayesian-Optimal Experiments},
  author =       {Foster, Adam and Jankowiak, Martin and O'Meara, Matthew and Teh, Yee Whye and Rainforth, Tom},
  booktitle = 	 {Proceedings of the Twenty Third International Conference on Artificial Intelligence and Statistics},
  pages = 	 {2959--2969},
  year = 	 {2020},
  editor = 	 {Chiappa, Silvia and Calandra, Roberto},
  volume = 	 {108},
  series = 	 {Proceedings of Machine Learning Research},
  month = 	 {26--28 Aug},
  publisher =    {PMLR},
  url = 	 {https://proceedings.mlr.press/v108/foster20a.html}
}

@inproceedings{lee2019settransformer,
  title={Set Transformer: A Framework for Attention-based Permutation-Invariant Neural Networks},
  author={Lee, Juho and Lee, Yoonho and Kim, Jungtaek and Kosiorek, Adam R. and Choi, Seungjin and Teh, Yee Whye},
  booktitle={Proceedings of the 36th International Conference on Machine Learning},
  year={2019}
}

@InProceedings{tnp_ashman24a,
  title = 	 {In-Context In-Context Learning with Transformer Neural Processes},
  author =       {Ashman, Matthew and Diaconu, Cristiana and Weller, Adrian and Turner, Richard E.},
  booktitle = 	 {Proceedings of the 6th Symposium on Advances in Approximate Bayesian Inference},
  pages = 	 {1--29},
  year = 	 {2024},
  editor = 	 {Antorán, Javier and Naesseth, Christian A.},
  volume = 	 {253},
  series = 	 {Proceedings of Machine Learning Research},
  month = 	 {21 Jul},
  publisher =    {PMLR},
  url = 	 {https://proceedings.mlr.press/v253/ashman24a.html}
}

@inproceedings{Paszke2019PyTorch,
  author    = {Paszke, Adam and Gross, Sam and Massa, Francisco and Lerer, Adam and Bradbury, James and Chanan, Gregory and Killeen, Trevor and Lin, Zeming and Gimelshein, Natalia and Antiga, Luca and Desmaison, Alban and Kopf, Andreas and Yang, Edward and DeVito, Zachary and Raison, Martin and Tejani, Alykhan and Chilamkurthy, Sasank and Steiner, Benoit and Fang, Lu and Bai, Junjie and Chintala, Soumith},
  title     = {PyTorch: An Imperative Style, High-Performance Deep Learning Library},
  booktitle = {Advances in Neural Information Processing Systems 32 (NeurIPS)},
  year      = {2019},
  editor    = {Wallach, Hanna and Larochelle, Hugo and Beygelzimer, Alina and d'Alch{\'e}-Buc, Florence and Fox, Edward and Garnett, Roman},
  pages     = {8024--8035}
}

@inproceedings{zaheer2017deepsets,
  author    = {Manzil Zaheer and
               Satwik Kottur and
               Siamak Ravanbakhsh and
               Barnab{\'a}s P{\'o}czos and
               Ruslan Salakhutdinov and
               Alexander J. Smola},
  title     = {Deep Sets},
  booktitle = {Advances in Neural Information Processing Systems 30 (NeurIPS 2017)},
  editor    = {I. Guyon and
               U. V. Luxburg and
               S. Bengio and
               H. Wallach and
               R. Fergus and
               S. Vishwanathan and
               R. Garnett},
  pages     = {3391--3401},
  year      = {2017},
  publisher = {Curran Associates, Inc.}
}

@article{hochreiter1997lstm,
  author  = {Hochreiter, Sepp and Schmidhuber, J{\"u}rgen},
  title   = {Long Short-Term Memory},
  journal = {Neural Computation},
  volume  = {9},
  number  = {8},
  pages   = {1735--1780},
  year    = {1997},
  publisher = {MIT Press}
}

@article{lecun1998gradient,
  title   = {Gradient-Based Learning Applied to Document Recognition},
  author  = {LeCun, Yann and Bottou, L{\'e}on and Bengio, Yoshua and Haffner, Patrick},
  journal = {Proceedings of the IEEE},
  volume  = {86},
  number  = {11},
  pages   = {2278--2324},
  year    = {1998},
}

@ARTICLE{location,
  author={Xiaohong Sheng and Yu-Hen Hu},
  journal={IEEE Transactions on Signal Processing}, 
  title={Maximum likelihood multiple-source localization using acoustic energy measurements with wireless sensor networks}, 
  year={2005},
  volume={53},
  number={1},
  pages={44-53},
  doi={10.1109/TSP.2004.838930}}

@inproceedings{Belousov2019,
  author    = {Belousov, Boris and Abdulsamad, Hany and Schultheis, Moritz and Peters, Jan},
  title     = {Belief Space Model Predictive Control for Approximately Optimal System Identification},
  booktitle = {Proceedings of the Multidisciplinary Conference on Reinforcement Learning and Decision Making (RLDM)},
  year      = {2019}
}

@inproceedings{blau_optimizing_2022,
	title = {Optimizing {Sequential} {Experimental} {Design} with {Deep} {Reinforcement} {Learning}},
	issn = {2640-3498},
	url = {https://proceedings.mlr.press/v162/blau22a.html},
	language = {en},
	urldate = {2024-11-20},
	booktitle = {Proceedings of the 39th {International} {Conference} on {Machine} {Learning}},
	publisher = {PMLR},
	author = {Blau, Tom and Bonilla, Edwin V. and Chades, Iadine and Dezfouli, Amir},
	month = jun,
	year = {2022},
	pages = {2107--2128},
}

@inproceedings{foster_deep_2021,
	title = {Deep {Adaptive} {Design}: {Amortizing} {Sequential} {Bayesian} {Experimental} {Design}},
	issn = {2640-3498},
	shorttitle = {Deep {Adaptive} {Design}},
	url = {https://proceedings.mlr.press/v139/foster21a.html},
	language = {en},
	urldate = {2024-11-20},
	booktitle = {Proceedings of the 38th {International} {Conference} on {Machine} {Learning}},
	publisher = {PMLR},
	author = {Foster, Adam and Ivanova, Desi R. and Malik, Ilyas and Rainforth, Tom},
	month = jul,
	year = {2021},
	pages = {3384--3395},
}

@inproceedings{ivanova_implicit_2021,
	title = {Implicit {Deep} {Adaptive} {Design}: {Policy}-{Based} {Experimental} {Design} without {Likelihoods}},
	volume = {34},
	shorttitle = {Implicit {Deep} {Adaptive} {Design}},
	url = {https://papers.neurips.cc/paper_files/paper/2021/hash/d811406316b669ad3d370d78b51b1d2e-Abstract.html},
	urldate = {2024-11-20},
	booktitle = {Advances in {Neural} {Information} {Processing} {Systems}},
	publisher = {Curran Associates, Inc.},
	author = {Ivanova, Desi R and Foster, Adam and Kleinegesse, Steven and Gutmann, Michael U. and Rainforth, Thomas},
	year = {2021},
	pages = {25785--25798},
}

@article{rainforth_modern_2024,
	title = {Modern {Bayesian} {Experimental} {Design}},
	volume = {39},
	issn = {0883-4237, 2168-8745},
	url = {https://projecteuclid.org/journals/statistical-science/volume-39/issue-1/Modern-Bayesian-Experimental-Design/10.1214/23-STS915.full},
	doi = {10.1214/23-STS915},
	number = {1},
	urldate = {2024-11-20},
	journal = {Statistical Science},
	publisher = {Institute of Mathematical Statistics},
	author = {Rainforth, Tom and Foster, Adam and Ivanova, Desi R. and Bickford Smith, Freddie},
	month = feb,
	year = {2024},
	pages = {100--114},
}

@inproceedings{rainforth_nesting_2018,
	title = {On {Nesting} {Monte} {Carlo} {Estimators}},
	issn = {2640-3498},
	url = {https://proceedings.mlr.press/v80/rainforth18a.html},
	language = {en},
	urldate = {2024-11-20},
	booktitle = {Proceedings of the 35th {International} {Conference} on {Machine} {Learning}},
	publisher = {PMLR},
	author = {Rainforth, Tom and Cornish, Rob and Yang, Hongseok and Warrington, Andrew and Wood, Frank},
	month = jul,
	year = {2018},
	pages = {4267--4276},
}

@article{huan_optimal_2024,
	title = {Optimal experimental design: {Formulations} and computations},
	volume = {33},
	issn = {0962-4929, 1474-0508},
	shorttitle = {Optimal experimental design},
	url = {https://www.cambridge.org/core/journals/acta-numerica/article/optimal-experimental-design-formulations-and-computations/38BBD0DC1A0386FDF306B6C0167DF7D9#},
	doi = {10.1017/S0962492924000023},
	language = {en},
	urldate = {2024-11-20},
	journal = {Acta Numerica},
	author = {Huan, Xun and Jagalur, Jayanth and Marzouk, Youssef},
	month = jul,
	year = {2024},
	pages = {715--840},
}

@inproceedings{foster_variational_2019,
	title = {Variational {Bayesian} {Optimal} {Experimental} {Design}},
	volume = {32},
	url = {https://proceedings.neurips.cc/paper_files/paper/2019/hash/d55cbf210f175f4a37916eafe6c04f0d-Abstract.html},
	urldate = {2024-11-20},
	booktitle = {Advances in {Neural} {Information} {Processing} {Systems}},
	publisher = {Curran Associates, Inc.},
	author = {Foster, Adam and Jankowiak, Martin and Bingham, Elias and Horsfall, Paul and Teh, Yee Whye and Rainforth, Thomas and Goodman, Noah},
	year = {2019},
}

@inproceedings{iqbal_nesting_2024,
	title = {Nesting {Particle} {Filters} for {Experimental} {Design} in {Dynamical} {Systems}},
	issn = {2640-3498},
	url = {https://proceedings.mlr.press/v235/iqbal24a.html},
	language = {en},
	urldate = {2024-11-20},
	booktitle = {Proceedings of the 41st {International} {Conference} on {Machine} {Learning}},
	publisher = {PMLR},
	author = {Iqbal, Sahel and Corenflos, Adrien and Särkkä, Simo and Abdulsamad, Hany},
	month = jul,
	year = {2024},
	pages = {21047--21068},
}

@inproceedings{iollo_pasoa_2024,
	title = {{PASOA}- {PArticle} {baSed} {Bayesian} {Optimal} {Adaptive} design},
	issn = {2640-3498},
	url = {https://proceedings.mlr.press/v235/iollo24a.html},
	language = {en},
	urldate = {2024-11-20},
	booktitle = {Proceedings of the 41st {International} {Conference} on {Machine} {Learning}},
	publisher = {PMLR},
	author = {Iollo, Jacopo and Heinkelé, Christophe and Alliez, Pierre and Forbes, Florence},
	month = jul,
	year = {2024},
	pages = {21020--21046},
}

@article{goda_unbiased_2022,
	title = {Unbiased {MLMC} {Stochastic} {Gradient}-{Based} {Optimization} of {Bayesian} {Experimental} {Designs}},
	volume = {44},
	issn = {1064-8275},
	url = {https://epubs.siam.org/doi/10.1137/20M1338848},
	doi = {10.1137/20M1338848},
	number = {1},
	urldate = {2024-11-20},
	journal = {SIAM Journal on Scientific Computing},
	publisher = {Society for Industrial and Applied Mathematics},
	author = {Goda, Takashi and Hironaka, Tomohiko and Kitade, Wataru and Foster, Adam},
	month = feb,
	year = {2022},
	pages = {A286--A311},
}

@article{ao_estimating_2024,
	title = {On {Estimating} the {Gradient} of the {Expected} {Information} {Gain} in {Bayesian} {Experimental} {Design}},
	volume = {38},
	copyright = {Copyright (c) 2024 Association for the Advancement of Artificial Intelligence},
	issn = {2374-3468},
	url = {https://ojs.aaai.org/index.php/AAAI/article/view/30012},
	doi = {10.1609/aaai.v38i18.30012},
	language = {en},
	number = {18},
	urldate = {2024-11-20},
	journal = {Proceedings of the AAAI Conference on Artificial Intelligence},
	author = {Ao, Ziqiao and Li, Jinglai},
	month = mar,
	year = {2024},
	note = {Number: 18},
	pages = {20311--20319},
}

@article{mohamed_monte_2020,
	title = {Monte {Carlo} {Gradient} {Estimation} in {Machine} {Learning}},
	volume = {21},
	issn = {1533-7928},
	url = {http://jmlr.org/papers/v21/19-346.html},
	number = {132},
	urldate = {2024-12-11},
	journal = {Journal of Machine Learning Research},
	author = {Mohamed, Shakir and Rosca, Mihaela and Figurnov, Michael and Mnih, Andriy},
	year = {2020},
	pages = {1--62},
}

@inproceedings{smith_predictionoriented_2023,
	title = {Prediction-{Oriented} {Bayesian} {Active} {Learning}},
	issn = {2640-3498},
	url = {https://proceedings.mlr.press/v206/bickfordsmith23a.html},
	language = {en},
	urldate = {2024-12-14},
	booktitle = {Proceedings of {The} 26th {International} {Conference} on {Artificial} {Intelligence} and {Statistics}},
	publisher = {PMLR},
	author = {Bickford Smith, Freddie and Kirsch, Andreas and Farquhar, Sebastian and Gal, Yarin and Foster, Adam and Rainforth, Tom},
	month = apr,
	year = {2023},
	pages = {7331--7348},
}

@inproceedings{huang_amortized_2024,
author = {Huang, Daolang and Guo, Yujia and Acerbi, Luigi and Kaski, Samuel},
title = {Amortized Bayesian experimental design for decision-making},
year = {2024},
isbn = {9798331314385},
publisher = {Curran Associates Inc.},
address = {Red Hook, NY, USA},
booktitle = {Proceedings of the 38th International Conference on Neural Information Processing Systems},
articleno = {3475},
numpages = {27},
location = {Vancouver, BC, Canada},
series = {NIPS '24}
}

@article{goda_multilevel_2020,
	title = {Multilevel {Monte} {Carlo} estimation of expected information gains},
	volume = {38},
	issn = {0736-2994},
	url = {https://doi.org/10.1080/07362994.2019.1705168},
	doi = {10.1080/07362994.2019.1705168},
	number = {4},
	urldate = {2025-01-17},
	journal = {Stochastic Analysis and Applications},
	publisher = {Taylor \& Francis},
	author = {Goda, Takashi and Hironaka, Tomohiko and Iwamoto, Takeru},
	month = jul,
	year = {2020},
	pages = {581--600},
}

@article{coons_multifidelity_2025,
author = {Coons, Thomas E. and Huan, Xun},
title = {A Multifidelity Estimator of the Expected Information Gain for Bayesian Optimal Experimental Design},
journal = {SIAM/ASA Journal on Uncertainty Quantification},
volume = {13},
number = {4},
pages = {1990-2021},
year = {2025},
doi = {10.1137/25M1731812},
URL = {https://doi.org/10.1137/25M1731812},
eprint = {https://doi.org/10.1137/25M1731812}
}

@article{lindley_measure_1956,
	title = {On a {Measure} of the {Information} {Provided} by an {Experiment}},
	volume = {27},
	issn = {0003-4851, 2168-8990},
	url = {https://projecteuclid.org/journals/annals-of-mathematical-statistics/volume-27/issue-4/On-a-Measure-of-the-Information-Provided-by-an-Experiment/10.1214/aoms/1177728069.full},
	doi = {10.1214/aoms/1177728069},
	number = {4},
	urldate = {2025-04-01},
	journal = {The Annals of Mathematical Statistics},
	publisher = {Institute of Mathematical Statistics},
	author = {Lindley, D. V.},
	month = dec,
	year = {1956},
	pages = {986--1005},
}

@article{chaloner_bayesian_1995,
	title = {Bayesian {Experimental} {Design}: {A} {Review}},
	volume = {10},
	issn = {0883-4237, 2168-8745},
	shorttitle = {Bayesian {Experimental} {Design}},
	url = {https://projecteuclid.org/journals/statistical-science/volume-10/issue-3/Bayesian-Experimental-Design-A-Review/10.1214/ss/1177009939.full},
	doi = {10.1214/ss/1177009939},
	number = {3},
	urldate = {2025-04-01},
	journal = {Statistical Science},
	publisher = {Institute of Mathematical Statistics},
	author = {Chaloner, Kathryn and Verdinelli, Isabella},
	month = aug,
	year = {1995},
	pages = {273--304},
}

@article{myung_tutorial_2013,
	title = {A tutorial on adaptive design optimization},
	volume = {57},
	issn = {0022-2496},
	url = {https://www.sciencedirect.com/science/article/pii/S0022249613000503},
	doi = {10.1016/j.jmp.2013.05.005},
	number = {3},
	urldate = {2025-04-03},
	journal = {Journal of Mathematical Psychology},
	author = {Myung, Jay I. and Cavagnaro, Daniel R. and Pitt, Mark A.},
	month = jun,
	year = {2013},
	pages = {53--67},
}

@inproceedings{kleinegesse_efficient_2019,
	title = {Efficient {Bayesian} {Experimental} {Design} for {Implicit} {Models}},
	issn = {2640-3498},
	url = {https://proceedings.mlr.press/v89/kleinegesse19a.html},
	language = {en},
	urldate = {2025-04-03},
	booktitle = {Proceedings of the {Twenty}-{Second} {International} {Conference} on {Artificial} {Intelligence} and {Statistics}},
	publisher = {PMLR},
	author = {Kleinegesse, Steven and Gutmann, Michael U.},
	month = apr,
	year = {2019},
	pages = {476--485},
}

@inproceedings{kleinegesse_bayesian_2020,
	title = {Bayesian {Experimental} {Design} for {Implicit} {Models} by {Mutual} {Information} {Neural} {Estimation}},
	issn = {2640-3498},
	url = {https://proceedings.mlr.press/v119/kleinegesse20a.html},
	language = {en},
	urldate = {2025-04-03},
	booktitle = {Proceedings of the 37th {International} {Conference} on {Machine} {Learning}},
	publisher = {PMLR},
	author = {Kleinegesse, Steven and Gutmann, Michael U.},
	month = nov,
	year = {2020},
	pages = {5316--5326},
}

@article{huan_simulationbased_2013,
	title = {Simulation-based optimal {Bayesian} experimental design for nonlinear systems},
	volume = {232},
	issn = {0021-9991},
	url = {https://doi.org/10.1016/j.jcp.2012.08.013},
	doi = {10.1016/j.jcp.2012.08.013},
	number = {1},
	urldate = {2025-04-03},
	journal = {J. Comput. Phys.},
	author = {Huan, Xun and Marzouk, Youssef M.},
	month = jan,
	year = {2013},
	pages = {288--317},
}

@article{shen_bayesian_2023,
	title = {Bayesian {Sequential} {Optimal} {Experimental} {Design} for {Nonlinear} {Models} {Using} {Policy} {Gradient} {Reinforcement} {Learning}},
	volume = {416},
	issn = {00457825},
	url = {http://arxiv.org/abs/2110.15335},
	doi = {10.1016/j.cma.2023.116304},
	urldate = {2025-04-04},
	journal = {Computer Methods in Applied Mechanics and Engineering},
	author = {Shen, Wanggang and Huan, Xun},
	month = nov,
	year = {2023},
	note = {arXiv:2110.15335 [cs]},
	pages = {116304},
}

@misc{kerrigan_geometric_2025,
	title = {A {Geometric} {Approach} to {Optimal} {Experimental} {Design}},
	url = {http://arxiv.org/abs/2510.14848},
	doi = {10.48550/arXiv.2510.14848},
	urldate = {2025-12-01},
	publisher = {arXiv},
	author = {Kerrigan, Gavin and Naesseth, Christian A. and Rainforth, Tom},
	month = oct,
	year = {2025},
	note = {arXiv:2510.14848 [stat]},
}

@inproceedings{huang_aline_2025,
author = {Huang, Daolang and Wen, Xinyi and Bharti, Ayush and Kaski, Samuel and Acerbi, Luigi},
title = {ALINE: joint amortization for bayesian inference and active data acquisition},
year = {2026},
publisher = {Curran Associates Inc.},
address = {Red Hook, NY, USA},
url = {https://doi.org/10.52202/085713-1805},
doi = {10.52202/085713-1805},
booktitle = {Proceedings of the 39th International Conference on Neural Information Processing Systems},
articleno = {1630},
numpages = {35},
location = {San Diego, CA, United StatesMexico City, Mexico, United States},
series = {NIPS '25}
}

@article{dawid2007geometry,
  title={The geometry of proper scoring rules},
  author={Dawid, A Philip},
  journal={Annals of the Institute of Statistical Mathematics},
  volume={59},
  number={1},
  pages={77--93},
  year={2007},
  publisher={Springer}
}

@article{gneiting2007strictly,
  title={Strictly proper scoring rules, prediction, and estimation},
  author={Gneiting, Tilmann and Raftery, Adrian E},
  journal={Journal of the American statistical Association},
  volume={102},
  number={477},
  pages={359--378},
  year={2007},
  publisher={Taylor \& Francis}
}

@article{lindley1982scoring,
  title={Scoring rules and the inevitability of probability},
  author={Lindley, Dennis V},
  journal={International statistical review/revue internationale de statistique},
  pages={1--11},
  year={1982},
  publisher={JSTOR}
}

@article{fort2017mcmc,
  title={MCMC design-based non-parametric regression for rare event. Application to nested risk computations},
  author={Fort, Gersende and Gobet, Emmanuel and Moulines, Eric},
  journal={Monte Carlo Methods and Applications},
  volume={23},
  number={1},
  pages={21--42},
  year={2017},
  publisher={De Gruyter}
}

@article{dong_variational_2025,
	title = {Variational {Bayesian} optimal experimental design with normalizing flows},
	volume = {433},
	issn = {00457825},
	url = {https://linkinghub.elsevier.com/retrieve/pii/S0045782524007126},
	doi = {10.1016/j.cma.2024.117457},
	language = {en},
	urldate = {2026-01-21},
	journal = {Computer Methods in Applied Mechanics and Engineering},
	author = {Dong, Jiayuan and Jacobsen, Christian and Khalloufi, Mehdi and Akram, Maryam and Liu, Wanjiao and Duraisamy, Karthik and Huan, Xun},
	month = jan,
	year = {2025},
	pages = {117457},
}

@inproceedings{iollo_bayesian_2025,
 author = {Iollo, Jacopo and Heinkel\'{e}, Christophe and Alliez, Pierre and Forbes, Florence},
 booktitle = {International Conference on Learning Representations},
 editor = {Y. Yue and A. Garg and N. Peng and F. Sha and R. Yu},
 pages = {73125--73151},
 title = {Bayesian Experimental Design Via Contrastive Diffusions},
 url = {https://proceedings.iclr.cc/paper_files/paper/2025/file/b5cc526f12164b2144bb2e06f2e84864-Paper-Conference.pdf},
 volume = {2025},
 year = {2025}
}

@article{cox1946probability,
  title={Probability, frequency and reasonable expectation},
  author={Cox, Richard T},
  journal={American journal of physics},
  volume={14},
  number={1},
  pages={1--13},
  year={1946}
}

@article{murphy2024reinforcement,
  title={Reinforcement learning: an overview},
  author={Murphy, Kevin},
  journal={arXiv preprint arXiv:2412.05265},
  year={2024}
}

@article{kuratowski1965general,
  title={A general theorem on selectors},
  author={Kuratowski, Kazimierz and Ryll-Nardzewski, Czes{\l}aw},
  journal={Bull. Acad. Polon. Sci. S{\'e}r. Sci. Math. Astronom. Phys},
  volume={13},
  number={6},
  pages={397--403},
  year={1965}
}

@article{williams_simple_1992a,
	title = {Simple {Statistical} {Gradient}-{Following} {Algorithms} for {Connectionist} {Reinforcement} {Learning}},
	volume = {8},
	issn = {0885-6125},
	url = {https://doi.org/10.1007/BF00992696},
	doi = {10.1007/BF00992696},
	number = {3-4},
	urldate = {2026-02-24},
	journal = {Mach. Learn.},
	author = {Williams, Ronald J.},
	month = may,
	year = {1992},
	pages = {229--256},
}

@article{huan_gradientbased_2012,
	title = {Gradient-based stochastic optimization methods in {Bayesian} experimental design},
	volume = {4},
	doi = {10.1615/Int.J.UncertaintyQuantification.2014006730},
	journal = {International Journal for Uncertainty Quantification},
	author = {Huan, Xun and Marzouk, Youssef},
	month = dec,
	year = {2012},
}

@inproceedings{barber_information_2003,
	title = {Information {Maximization} in {Noisy} {Channels} : {A} {Variational} {Approach}},
	volume = {16},
	shorttitle = {Information {Maximization} in {Noisy} {Channels}},
	url = {https://papers.nips.cc/paper_files/paper/2003/hash/a6ea8471c120fe8cc35a2954c9b9c595-Abstract.html},
	urldate = {2026-02-24},
	booktitle = {Advances in {Neural} {Information} {Processing} {Systems}},
	publisher = {MIT Press},
	author = {Barber, David and Agakov, Felix},
	year = {2003},
}

@misc{bracher_jadai_2025,
	title = {{JADAI}: {Jointly} {Amortizing} {Adaptive} {Design} and {Bayesian} {Inference}},
	shorttitle = {{JADAI}},
	url = {http://arxiv.org/abs/2512.22999},
	doi = {10.48550/arXiv.2512.22999},
	urldate = {2026-02-26},
	publisher = {arXiv},
	author = {Bracher, Niels and Kühmichel, Lars and Ivanova, Desi R. and Intes, Xavier and Bürkner, Paul-Christian and Radev, Stefan T.},
	month = dec,
	year = {2025},
	note = {arXiv:2512.22999 [stat]},
}

@inproceedings{hedman_stepdad_2025,
	title = {Step-{DAD}: {Semi}-{Amortized} {Policy}-{Based} {Bayesian} {Experimental} {Design}},
	issn = {2640-3498},
	shorttitle = {Step-{DAD}},
	url = {https://proceedings.mlr.press/v267/hedman25a.html},
	language = {en},
	urldate = {2026-02-26},
	booktitle = {Proceedings of the 42nd {International} {Conference} on {Machine} {Learning}},
	publisher = {PMLR},
	author = {Hedman, Marcel and Ivanova, Desi R. and Guan, Cong and Rainforth, Tom},
	month = oct,
	year = {2025},
	pages = {22904--22923},
}

@inproceedings{huang_lossdriven_2026,
	title = {Loss-{Driven} {Bayesian} {Active} {Learning}},
	author = {Huang, Zhuoyue and Bickford Smith, Freddie and Rainforth, Tom},
	booktitle={International Conference on Artificial Intelligence and Statistics},
	year = {2026}
}

@book{fisher1935design,
  title={The Design of Experiments.},
  author={Fisher, RA},
  year={1935},
  publisher={Oliver \& Boyd}
}

@book{lindley_1972,
	title = {Bayesian statistics},
	url = {https://epubs.siam.org/doi/abs/10.1137/1.9781611970654},
	doi = {10.1137/1.9781611970654},
	publisher = {Society for Industrial and Applied Mathematics},
	author = {Lindley, D. V.},
	year = {1972},
}

@book{berger_statistical_1985,
	title = {Statistical {Decision} {Theory} and {Bayesian} {Analysis}},
	isbn = {978-0-387-96098-2},
	language = {en},
	publisher = {Springer Science \& Business Media},
	author = {Berger, James O.},
	month = aug,
	year = {1985},
}

@article{degroot1962uncertainty,
  title={Uncertainty, information, and sequential experiments},
  journal={The Annals of Mathematical Statistics},
  author= {DeGroot, MH},
  volume={33},
  number={2},
  pages={404--419},
  year={1962},
  publisher={Institute of Mathematical Statistics}
}

@inproceedings{kingma_autoencoding_2013,
  title={Auto-Encoding Variational Bayes},
  author={Diederik P. Kingma and Max Welling},
  booktitle={International Conference on Learning Representations},
  year={2013},
  url={https://api.semanticscholar.org/CorpusID:216078090}
}

@article{shannon_mathematical_1948,
    title = {A {Mathematical} {Theory} of {Communication}},
    volume = {27},
    copyright = {© 1948 The Bell System Technical Journal},
    issn = {1538-7305},
    url = {https://onlinelibrary.wiley.com/doi/abs/10.1002/j.1538-7305.1948.tb01338.x},
    doi = {10.1002/j.1538-7305.1948.tb01338.x},
    language = {en},
    number = {3},
    urldate = {2025-04-01},
    journal = {Bell System Technical Journal},
    author = {Shannon, C. E.},
    year = {1948},
    pages = {379--423},
}

@article{dawid1998coherent,
  title={Coherent measures of discrepancy, uncertainty and dependence, with applications to {B}ayesian predictive experimental design},
  author={Dawid, A Philip},
  journal={Department of Statistical Science, University College London. http://www. ucl. ac. uk/Stats/research/abs94. html, Tech. Rep},
  volume={139},
  year={1998}
}

@article{shen2025variational,
  title={Variational sequential optimal experimental design using reinforcement learning},
  author={Shen, Wanggang and Dong, Jiayuan and Huan, Xun},
  journal={Computer Methods in Applied Mechanics and Engineering},
  volume={444},
  pages={118068},
  year={2025},
  publisher={Elsevier}
}

@article{zhong_goaloriented_2026,
    title = {Goal-{Oriented} {Bayesian} {Optimal} {Experimental} {Design} for {Nonlinear} {Models} using {Markov} {Chain} {Monte} {Carlo}},
    volume = {14},
    issn = {2166-2525},
    url = {http://arxiv.org/abs/2403.18072},
    doi = {10.1137/24M1649344},
    number = {1},
    urldate = {2026-05-01},
    journal = {SIAM/ASA Journal on Uncertainty Quantification},
    author = {Zhong, Shijie and Shen, Wanggang and Catanach, Tommie and Huan, Xun},
    month = mar,
    year = {2026},
    note = {arXiv:2403.18072 [stat]},
    pages = {19--47},
}

@inproceedings{neiswanger_generalizing_2022,
    address = {Red Hook, NY, USA},
    series = {{NIPS} '22},
    title = {Generalizing {Bayesian} optimization with decision-theoretic entropies},
    isbn = {978-1-7138-7108-8},
    urldate = {2026-05-01},
    booktitle = {Proceedings of the 36th {International} {Conference} on {Neural} {Information} {Processing} {Systems}},
    publisher = {Curran Associates Inc.},
    author = {Neiswanger, Willie and Yu, Lantao and Zhao, Shengjia and Meng, Chenlin and Ermon, Stefano},
    month = nov,
    year = {2022},
    pages = {21016--21029},
}

@article{filstroff_targeted_2024,
    title = {Targeted {Active} {Learning} for {Bayesian} {Decision}-{Making}},
    issn = {2835-8856},
    url = {https://openreview.net/forum?id=KxPjuiMgmm},
    language = {en},
    urldate = {2026-05-01},
    journal = {Transactions on Machine Learning Research},
    author = {Filstroff, Louis and Sundin, Iiris and Mikkola, Petrus and Tiulpin, Aleksei and Kylmäoja, Juuso and Kaski, Samuel},
    month = feb,
    year = {2024},
}

@misc{kleinegesse_gradientbased_2021,
    title = {Gradient-based {Bayesian} {Experimental} {Design} for {Implicit} {Models} using {Mutual} {Information} {Lower} {Bounds}},
    url = {http://arxiv.org/abs/2105.04379},
    doi = {10.48550/arXiv.2105.04379},
    urldate = {2026-05-02},
    publisher = {arXiv},
    author = {Kleinegesse, Steven and Gutmann, Michael U.},
    month = may,
    year = {2021},
    note = {arXiv:2105.04379 [stat]},
}

@book{aliprantis2006infinite,
  title={Infinite Dimensional Analysis: A Hitchhiker's Guide},
  author={Aliprantis, Charalambos D. and Border, Kim C.},
  edition={3},
  publisher={Springer},
  year={2006}
}

@inproceedings{blau_cross-entropy_2023,
    title={Cross-Entropy Estimators for Sequential Experiment Design with Reinforcement Learning},
    author={Tom Blau and Iadine Chades and Amir Dezfouli and Daniel M Steinberg and Edwin V. Bonilla},
    booktitle={NeurIPS 2023 Workshop on Adaptive Experimental Design and Active Learning in the Real World},
    year={2023},
    url={https://openreview.net/forum?id=kPWO1v0slD}
}

@inproceedings{barlas_performance_2025,
    title = {Performance Comparisons of Reinforcement Learning Algorithms for Sequential Experimental Design},
    month = {December},
    booktitle = {Workshop on Generalization in Planning (GenPlan), AAAI 2025},
    note = {In Press},
    year = {2024},
    author = {Barlas, Y. and Salako, K.},
    url = {https://aair-lab.github.io/genplan25/}
}

@misc{chakraborty_likelihood-free_2024,
    title = {A {Likelihood}-{Free} {Approach} to {Goal}-{Oriented} {Bayesian} {Optimal} {Experimental} {Design}},
    url = {https://www.semanticscholar.org/paper/A-Likelihood-Free-Approach-to-Goal-Oriented-Optimal-Chakraborty-Huan/a2663fd8456a00a25405b89e0f786741b768a3f4},
    urldate = {2026-05-06},
    note = {arXiv preprint arXiv:2408.09582}, 
    author = {Chakraborty, Atlanta and Huan, Xun and Catanach, Tommie A.},
    month = aug,
    year = {2024},
}

@incollection{neal2011mcmc,
  author    = {Neal, Radford M.},
  title     = {MCMC Using Hamiltonian Dynamics},
  booktitle = {Handbook of Markov Chain Monte Carlo},
  editor    = {Brooks, Steve and Gelman, Andrew and Jones, Galin L. and Meng, Xiao-Li},
  publisher = {Chapman and Hall/CRC},
  year      = {2011},
  pages     = {113--162}
}

@article{huang2026efficient,
  title={Efficient Adaptive Data Acquisition via Pretrained Belief Representations},
  author={Huang, Daolang and Huang, Zhuoyue and Hassan, Conor and Acerbi, Luigi and Kaski, Samuel and Rainforth, Tom},
  journal={Advances in {Neural} {Information} {Processing} {Systems}},
  year={2026}
}

@InProceedings{phillips2026bayesian,
  title = 	 {{Bayesian} Experimental Design via Score Matching},
  author =       {Phillips, Angus and Kerrigan, Gavin and Rainforth, Tom},
  booktitle = 	 {Proceedings of the 42nd Conference on Uncertainty in Artificial Intelligence},
  pages = 	 {5393--5426},
  year = 	 {2026},
  editor = 	 {Perković, Emilija and Malinsky, Daniel},
  volume = 	 {337},
  series = 	 {Proceedings of Machine Learning Research},
  month = 	 {17--21 Aug},
  publisher =    {PMLR},
  url = 	 {https://proceedings.mlr.press/v337/phillips26a.html}
}

@misc{kerrigan2026testbed,
      title={Test{BED}: An Open-Source Toolkit for Bayesian Experimental Design}, 
      author={Gavin Kerrigan and Marcel Hedman and Sahel Iqbal and Tom Rainforth},
      year={2026},
      eprint={2611.06633},
      archivePrefix={arXiv},
      primaryClass={stat.ML},
      note={arXiv preprint arXiv:2611.06633}, 
      url={https://arxiv.org/abs/2611.06633}, 
}

@book{raiffa2000applied,
  title={Applied statistical decision theory},
  author={Raiffa, Howard and Schlaifer, Robert},
  year={1961},
  publisher={John Wiley \& Sons}
}

@incollection{wald1950statistical,
  title={Statistical decision functions},
  author={Wald, Abraham},
  booktitle={Breakthroughs in Statistics: Foundations and Basic Theory},
  pages={342--357},
  year={1950},
  publisher={Springer}
}

\newpage

\appendix


\newpage
\section{Lifting Bayes-Optimal Actions to Action Policies}
\label{app:action_policy_lifting}

In this section, we formalise the equivalence between the nested optimisation over Bayes-optimal downstream actions in Definition \ref{def:obed_policy}  and the joint optimisation over design and action policies used in \Cref{eqn:joint_opt_problem}. The key point is that, under standard measurability and integrability assumptions, a pointwise optimisation over actions after observing the experimental history can be lifted to an optimisation over measurable action policies. We establish this equivalence in the more general setting where the downstream loss may depend explicitly on the observed history $h_T$, a dependence omitted in the main text for simplicity.

\paragraph{Setup.} 
Let \((\Omega,\mathcal F,\mathbb P)\) be the underlying probability space, and \(\pi_d\in\Pi_d\) be a fixed design policy. We assume that \(h_T\) takes values in a measurable history space
\((\mathcal H_T,\mathcal B(\mathcal{H}_T))\). We consider a downstream loss $\ell : \mathcal{A} \times \Theta \times \mathcal{H}_T \to \bar{\mathbb R}$, which may depend explicitly on the observed history. For a fixed history \(h\in\mathcal H_T\), we formally define the conditional risk $\forall a \in \mathcal{A}$ as,  
\[
    L(a,h)
    :=
    \mathbb E\!\left[\ell(a, \theta, h)\mid h_T=h\right].
\]
The Bayes-optimal action after observing \(h\) is any solution of
\[
    a^B(h) \in \argmin_{a\in\mathcal A} L(a,h).
\]
The corresponding Bayes expected future loss for a design policy \(\pi_d\) is
\[
    \mathrm{EFL}_\ell^B(\pi_d)
    =
    \mathbb E_{p(h_T;\pi_d)}
    \left[
        \min_{a\in\mathcal A}
        L(a,h_T)
    \right].
\]
We now show that this nested formulation is equivalent to an optimisation over
measurable action policies.

\paragraph{Standing assumptions.}
We work under standard regularity assumptions ensuring that the pointwise Bayes-optimal actions can be selected measurably as functions of the observed history. The history space \((\mathcal H_T,\mathcal B(\mathcal H_T))\) is standard Borel, and the action space \(\mathcal A\) is a compact metric space with Borel \(\sigma\)-algebra \(\mathcal B(\mathcal A)\). For every
\(\pi_d\in\Pi_d\), the conditional risk $L(a,h)$ admits a jointly measurable version on \(\mathcal A\times\mathcal H_T\), and
\(a\mapsto L(a,h)\) is lower semicontinuous for every \(h\in\mathcal H_T\).
We also assume the relevant expectations are well-defined, for example through
an integrable envelope \(m(h_T,\theta)\) satisfying
\[
    \forall (a, \theta, h_T), \quad
    |\ell(a, \theta, h_T)|\leq m(h_T,\theta)
    \quad
    \land
    \quad
    \mathbb E[m(h_T,\theta)]<\infty.
\]

Let
\begin{equation}
    \Pi_a^{\mathrm{meas}}
    :=
    \left\{
    \pi_a:\mathcal H_T\to\mathcal A
    \;:\;
    \pi_a
    \text{ is }
    \mathcal B(\mathcal H_T)/\mathcal B(\mathcal A)\text{-measurable}
    \right\}
\end{equation}
denote the set of admissible measurable action policies. The restriction to measurable action policies makes the optimisation
probabilistically well-defined. Under the regularity assumptions above, the Bayes-action correspondence admits a measurable selector, so the pointwise Bayes action belongs to  \(\Pi_a^{\mathrm{meas}}\). It yields the desired
equivalence, as stated in the following proposition.

\begin{proposition}[Lifting pointwise Bayes actions to action policies]
\label{prop:bayes_action_policy_lifting}
Under the previous assumptions, for every fixed design policy \(\pi_d\in\Pi_d\),
\begin{equation}
    \mathbb E_{p(h_T;\pi_d)}
    \left[
        \min_{a\in\mathcal A} L(a,h_T)
    \right]
    =
    \min_{\pi_a\in\Pi_a^{\mathrm{meas}}}
    \mathbb E_{p(h_T;\pi_d)}
    \left[
        L(\pi_a(h_T),h_T)
    \right].
\end{equation}
Equivalently,
\begin{equation}
\label{eq:min_outside_expectation}
    \mathrm{EFL}_\ell^B(\pi_d)
    =
    \min_{\pi_a\in\Pi_a^{\mathrm{meas}}}
    \mathbb E_{p(h_T;\pi_d)p(\theta\mid h_T)}
    \left[
        \ell(\pi_a(h_T), \theta, h_T)
    \right].
\end{equation}
\end{proposition}

\begin{proof}
Fix \(\pi_d\in\Pi_d\). For any measurable action policy
\(\pi_a\in\Pi_a^{\mathrm{meas}}\), we directly have pointwise
\[
    \min_{a\in\mathcal A} L(a,h)
    \leq
    L(\pi_a(h),h).
\]
Therefore,
\[
    \mathbb E_{p(h_T;\pi_d)}
    \left[
        \min_{a\in\mathcal A} L(a,h_T)
    \right]
    \leq
    \mathbb E_{p(h_T;\pi_d)}
    \left[
        L(\pi_a(h_T),h_T)
    \right].
\]
Taking the infimum over \(\pi_a\in\Pi_a^{\mathrm{meas}}\) gives
\[
    \mathbb E_{p(h_T;\pi_d)}
    \left[
        \min_{a\in\mathcal A} L(a,h_T)
    \right]
    \leq
    \inf_{\pi_a\in\Pi_a^{\mathrm{meas}}}
    \mathbb E_{p(h_T;\pi_d)}
    \left[
        L(\pi_a(h_T),h_T)
    \right].
\]

It remains to prove the reverse inequality. Since \(L(\cdot,h)\) is lower semicontinuous on the compact action space
\(\mathcal A\), the lower semicontinuous version of Weierstrass' theorem implies that \(\argmin_{a\in\mathcal A}L(a,h)\) is non-empty for every \(h\in\mathcal H_T\). Hence the argmin correspondence
\[
    \Gamma(h)
    :=
    \argmin_{a\in\mathcal A} L(a,h)
\]
is non-empty and compact-valued, as a closed subset of $\mathcal{A}$. Since \(L\) is jointly measurable and lower
semicontinuous in the action variable, the measurable maximum theorem from \citep{aliprantis2006infinite} implies that the value function
\[
    h\mapsto \min_{a\in\mathcal A} L(a,h)
\]
is measurable and that \(\Gamma\) is a measurable correspondence. By the
Kuratowski--Ryll-Nardzewski measurable selection theorem from \citep{kuratowski1965general}, there exists a
measurable selector \(\pi_a^B:\mathcal H_T\to\mathcal A\) such that
\[
    \pi_a^B(h)\in\Gamma(h)
    \qquad\text{for all }h\in\mathcal H_T.
\]
Therefore, $\forall h \in \mathcal H_T$ 
\[
    L(\pi_a^B(h),h)
    =
    \min_{a\in\mathcal A}L(a,h),
\]
and hence
\[
    \mathbb E_{p(h_T;\pi_d)}
    \left[
        L(\pi_a^B(h_T),h_T)
    \right]
    =
    \mathbb E_{p(h_T;\pi_d)}
    \left[
        \min_{a\in\mathcal A}L(a,h_T)
    \right].
\]
Thus the infimum over measurable policies is attained by \(\pi_a^B\), giving
the desired equality.
\end{proof}

\paragraph{From Bayes-action policies to joint design--action optimisation.}

Having shown the equality (\ref{eq:min_outside_expectation}), we can directly rewrite it as
\begin{equation}
     \mathrm{EFL}_\ell^B(\pi_d)
    =
    \min_{\pi_a\in\Pi_a^{\mathrm{meas}}}
    \mathbb E_{p(\theta)p(y_{1:T}\mid\theta;\pi_d)}
    \left[
        \ell(\pi_a(h_T), \theta, h_T)
    \right].
\end{equation}
Consequently, optimising the Bayes expected future loss over design policies can be written as 
\[
\begin{aligned}
    \pi_d^B
    &\in
    \argmin_{\pi_d\in\Pi_d}
    \mathrm{EFL}_\ell^B(\pi_d) \\
    &=
    \argmin_{\pi_d\in\Pi_d}
    \min_{\pi_a\in\Pi_a^{\mathrm{meas}}}
    \mathbb E_{p(\theta)p(y_{1:T}\mid\theta;\pi_d)}
    \left[
        \ell(\pi_a(h_T), \theta, h_T)
    \right], 
\end{aligned}
\]
which ultimately leads to the core objective used by \textsc{Action-BED} as described in equation (\ref{eqn:joint_opt_problem}). 

\paragraph{Remark. }
The compactness and lower semicontinuity assumptions above are standard
sufficient conditions ensuring that exact Bayes actions exist and admit
measurable selectors, which is enough for the equivalence used in this work.
When these assumptions fail, the result can be stated in a weaker but more
general form using infima and measurable \(\varepsilon\)-optimal selectors.

Specifically, if for every \(\varepsilon>0\) there exists a measurable policy
\(\pi_a^\varepsilon:\mathcal H_T\to\mathcal A\) satisfying
\[
    L(\pi_a^\varepsilon(h),h)
    \leq
    \inf_{a\in\mathcal A}L(a,h)+\varepsilon,
\]
then
\[
    \mathbb E_{p(h_T;\pi_d)}
    \left[
        \inf_{a\in\mathcal A}L(a,h_T)
    \right]
    =
    \inf_{\pi_a\in\Pi_a^{\mathrm{meas}}}
    \mathbb E_{p(h_T;\pi_d)}
    \left[
        L(\pi_a(h_T),h_T)
    \right].
\]
This extension is useful when the action space is non-compact or when the
conditional risk is not known to attain its pointwise infimum.

\newpage
\section{Equivalence Between BED Viewpoints}
\label{appendix:proper_scoring_rule_link}



In the following Section we expand on the discussion on the equivalence between EUR and EFL perspectives of BED by providing a proof of \Cref{thm:equiv_BED}. In \cref{app:variational_recovery} we additionally elaborate on the equivalence between \textsc{Action-BED} and existing variational approaches in a particular case.

\subsection{Proof of Theorem \ref{thm:equiv_BED}}
\label{appendix:equivalence_between_visions}

\paragraph{Assumptions.} 
Let $\Theta$ and $\mathcal H_T$ be standard Borel spaces, and let $\mathcal P_\theta$ be a measurable class of probability distributions on $\Theta$ containing the posterior distributions $p(\theta\mid \y; \pi_d)$, up to null sets, for every admissible design policy. We assume that all scoring rules considered are jointly measurable and proper for every fixed history $h\in\mathcal H_T$. We assume that the posterior expected self-scores, Bayes risks, and their expectations over experimental histories are finite for every admissible design policy. For any \(p,q\in\mathcal P_\theta\) and history \(h\), cross-score expectations \(\mathbb E_{\theta\sim p}[s(q,h,\theta)]\) are assumed to be well-defined, possibly taking the value \(+\infty\). In particular, all comparisons below are understood in the extended real numbers without indeterminate expressions.

For a downstream loss $\ell$, we assume that $\ell(a,\theta,h)$ is jointly measurable and that the corresponding Bayes risks are well-defined. Whenever Bayes-optimal actions are invoked, we assume that the relevant infima are attained. For the construction of a scoring rule from a downstream loss, we additionally assume that, for every $(q_\theta,h)\in\mathcal P_\theta\times\mathcal H_T$, the Bayes-action set is non-empty and admits a jointly measurable selector
\[
\tilde{\pi}_a^B(h,q_\theta)
\in\arg\min_{a\in\mathcal A}
\mathbb E_{\theta\sim q_\theta}
[\ell(a,\theta,h)].
\]
Unlike Appendix~\ref{app:action_policy_lifting}, we do not require compactness of the action space or lower semicontinuity of the loss. The conditions of Appendix~\ref{app:action_policy_lifting} provide sufficient conditions for measurable Bayes-action selection in the policy-lifting result, but are not required for the equivalence established here.

We restate the theorem for completeness in a more general form, allowing both the downstream loss and the scoring rule to depend explicitly on the observed data.
\newtheorem*{theoremrep}{Theorem \ref{thm:equiv_BED}}
\begin{theoremrep}[Equivalence of EFL \& EPU]
    \label{thm_rep:equiv_BED}
    Let
    $s : \mathcal{P}_{\theta} \times \mathcal H_T \times \Theta \to \bar{\mathbb{R}}$
    be a proper scoring rule on $\Theta$ for every fixed history $h_T \in \mathcal H_T$. 
    Then any
    \begin{equation}
    \label{eq_rep:bayes_opt_psr}
    \pi_d^B
    \in
    \arg\min_{\pi_d \in \Pi_d}
    \mathbb{E}_{p(\theta)p(\y \mid \theta;\pi_d)}
    \big[
    s\big(p_\theta(\cdot \mid h_T), h_T, \theta\big)
    \big]
    \end{equation}
    is a Bayes-optimal design policy for the belief model $p(\theta)p(\y | \theta;\pi_d)$ and for any loss $\ell \in \mathcal{L}_s$, where
    \begin{equation}
    \label{eq_rep:loss_set}
    \mathcal{L}_s
    =
    \Big\{
    \ell : 
    \forall (q_\theta, h) \in \mathcal P_\theta \times \mathcal H_T,\,
    \inf_{a \in \mathcal{A}} \E_{\theta \sim q_\theta}\big[\ell(a, \theta, h)\big]
    = \E_{\theta \sim q_\theta}\big[s(q_\theta, h, \theta)\big]
    \Big\}, \quad 
    \end{equation}
    Furthermore, the set $\mathcal{L}_s$ is never empty, and any loss function $\ell$ satisfying the previous integrability conditions and admitting measurable Bayes-action selectors induces one or more corresponding proper scoring rules $s$, though not necessarily strictly proper, of the form
    \begin{equation}
    \label{eq_rep:induced_scoring}
    s(q_\theta,h,\theta) =
    \ell\big(\tilde{\pi}_a^B(h,q_\theta),
    \theta,h\big),
    \end{equation}
    where
    \begin{equation}
    \label{eq_rep:bayes_opt_act}
    \tilde{\pi}_a^B(h,q_\theta)
    \in
    \arg\min_{a\in\mathcal{A}}
    \mathbb{E}_{\theta\sim q_\theta}
    [\ell(a,\theta,h)].
    \end{equation}
    Here, the action space $\mathcal{A}$ is understood to be part of the specification of each loss function \(\ell\). If the Bayes-action set contains multiple elements for some \((h,q_\theta)\), different measurable selections may induce distinct proper scoring rules for the same loss function. All expectations and minimizers are assumed to be well-defined.
\end{theoremrep}

\begin{proof}
We start by noting that, for any $\ell \in \mathcal{L}_s$,
\begin{align}
    \mathbb{E}_{p(\theta)p(y_{1:T}\mid \theta;\pi_d)}
    \left[
        s\!\left(
            p_\theta(\cdot \mid h_T),
            h_T, 
            \theta
        \right)
    \right]
    & =
    \mathbb{E}_{p(h_T;\pi_d)}
    \left[
        \E_{p(\theta \mid h_T)}
            \left[
                s\!\left(
                    p_\theta(\cdot \mid h_T),
                    h_T, 
                    \theta
                \right)
            \right]
    \right]
    \\
    & =
   \mathbb{E}_{p(h_T;\pi_d)}
    \left[
            \inf_{a\in\mathcal{A}}
            \mathbb{E}_{p(\theta\mid h_T)}
        \left[
            \ell(a,\theta,h_T)
        \right]
    \right].
\end{align}
The last equality follows directly from the definition of $\mathcal L_s$. Therefore, for any $\ell\in\mathcal{L}_s$, the minimizers of \Cref{eq_rep:bayes_opt_psr} coincide with the Bayes-optimal experimental design policies associated with $\ell$, as defined in \Cref{def:obed_policy}, allowing for explicit dependence on $h_T$. Hence, $\pi_d^B$ is a Bayes-optimal experimental design policy for $\ell$ under the belief model $p(\theta)p(y_{1:T}\mid \theta;\pi_d)$.


Next, we show that for any $\ell$ satisfying the stated assumptions, any corresponding $s$ defined as per Equation~\eqref{eq_rep:induced_scoring} must be a proper scoring rule for all $h_T \in \mathcal H_T$. To do this, let $\tilde{\pi}_a^B : \mathcal{H}_T \times \mathcal{P}_\theta \to \mathcal{A}$ denote an action policy that deterministically returns, for any given data and belief state $(h_T, p_\theta)$, a corresponding Bayes-optimal action as per Equation~\eqref{eq_rep:bayes_opt_act}. We note that different choices for $\tilde{\pi}_a^B$ may lead to different $s$. 
By the definition of $s$, we have
\begin{equation*}
    \inf_{q_\theta\in\mathcal{P}_\theta}
    \mathbb{E}_{p(\theta\mid y_{1:T};\pi_d)}
    \left[
        s(q_\theta, h_T, \theta)
    \right]
    =
    \inf_{q_\theta\in\mathcal{P}_\theta}
    \mathbb{E}_{p(\theta\mid y_{1:T};\pi_d)}
    \left[
        \ell\!\left(
            \tilde{\pi}_a^B(h_T, q_\theta), 
            \theta, 
            h_T
        \right)
    \right].
\end{equation*}
Since $q_\theta$ influences the expected loss only through the choice of action
$\tilde{\pi}_a^B(h_T, q_\theta)$, it follows that
\begin{equation*}
    \inf_{q_\theta\in\mathcal{P}_\theta}
    \mathbb{E}_{p(\theta\mid y_{1:T};\pi_d)}
    \left[
        \ell\!\left(
            \tilde{\pi}_a^B(h_T, q_\theta),
            \theta, 
            h_T
        \right)
    \right]
    \geq
    \inf_{a\in\mathcal{A}}
    \mathbb{E}_{p(\theta\mid y_{1:T};\pi_d)}
    \left[
        \ell(a, \theta, h_T)
    \right].
\end{equation*}
Finally, using the definition of $\tilde{\pi}_a^B$ that returns correspondingly Bayes-optimal actions, the right-hand side is attained by choosing the true posterior belief
$q_\theta = p_\theta(\cdot \mid y_{1:T};\pi_d)$, so that
\begin{align*}
    \inf_{a\in\mathcal{A}}
    \mathbb{E}_{p(\theta\mid y_{1:T};\pi_d)}
    \left[
        \ell(a, \theta, h_T)
    \right]
    &=
    \mathbb{E}_{p(\theta\mid y_{1:T};\pi_d)}
    \left[
        \ell\!\left(
            \tilde{\pi}_a
            \left(
                h_T, 
                p_\theta(\cdot \mid y_{1:T};\pi_d)
            \right),
            \theta, 
            h_T
        \right)
    \right] \\
    &=
    \mathbb{E}_{p(\theta\mid y_{1:T};\pi_d)}
    \left[
        s\!\left(
            p_\theta(\cdot \mid y_{1:T};\pi_d),
            h_T, 
            \theta
        \right)
    \right].
\end{align*}
Thus, $p_\theta(\cdot \mid y_{1:T};\pi_d)$ is a minimizer of $\mathbb{E}_{p(\theta\mid y_{1:T};\pi_d)} \left[s(q_\theta, h_T, \theta)\right]$, and so $s$ satisfies the definition of a proper scoring rule on $\Theta$. It need not be strictly proper as the minimiser is not necessarily unique.

Finally, we show that $\mathcal{L}_s$ can never be an empty set for a given scoring rule $s$. Starting from Equation~(\ref{eq_rep:bayes_opt_psr}), we have
\begin{align}
    \pi_d^B
    &=
    \argmin_{\pi_d\in\Pi_d}
    \mathbb{E}_{p(\theta)p(y_{1:T}\mid \theta;\pi_d)}
    \left[
        s\!\left(
            p_\theta(\cdot \mid y_{1:T};\pi_d),
            h_T, 
            \theta
        \right)
    \right]
    \\
    &=
    \argmin_{\pi_d\in\Pi_d}
    \mathbb{E}_{p(y_{1:T};\pi_d)}
    \left[
        \mathbb{E}_{p(\theta\mid y_{1:T};\pi_d)}
        \left[
            s\!\left(
                p_\theta(\cdot \mid y_{1:T};\pi_d),
                h_T, 
                \theta
            \right)
        \right]
    \right].
\end{align}
Since $s$ is a proper scoring rule, this is equal to
\begin{align}
    \pi_d^B
    &=
    \argmin_{\pi_d\in\Pi_d}
    \mathbb{E}_{p(y_{1:T};\pi_d)}
    \left[
        \min_{q_\theta\in\mathcal{P}_\theta}
        \mathbb{E}_{p(\theta\mid y_{1:T};\pi_d)}
        \left[
            s(q_\theta, h_T, \theta)
        \right]
    \right].
\end{align}
This is the general form of a Bayes-optimal experimental design policy with action space
$\mathcal{A}=\mathcal{P}_\theta$ and downstream loss
\begin{equation*}
    \ell(q_\theta, \theta, h_T)=s(q_\theta, h_T, \theta).
\end{equation*}
Therefore, $s$ is itself an element of $\mathcal{L}_s$, and consequently $\mathcal{L}_s$ can never be empty. 
\end{proof}

\subsection{Recovering Variational Approaches}
\label{app:variational_recovery}

As noted in \Cref{sec:unification}, the EIG itself is equivalent, up to an additive constant and a change of sign, to an EPU objective with the log score $s(q_\theta(\cdot), \theta) = - \log q_\theta(\theta)$. That is, up to a constant offset, the EIG is equivalent to 
\begin{equation*}
    \mathbb{E}_{p(\theta)p(\y|\theta; \pi_d)} \left[- \log p(\theta | h_T)\right].
\end{equation*}

We can find an equivalent EFL formulation by choosing $\ell \in \mathcal{L}_s$ according to \Cref{thm:equiv_BED}. The canonical choice of loss function $\ell \in \mathcal{L}_s$ for the log score is the log-loss on actions $a(\cdot) \in \mathcal{P}_\theta$, that is $\ell(a(\cdot), \theta) = - \log a(\theta)$. The optimal action under this loss is a standard result:
\begin{equation*}
    \tilde{\pi}_a = \argmin_{a(\cdot) \in \mathcal{P}_\theta} \E_{\theta \sim p(\cdot | h_T)}\left[-\log a(\theta)\right] = \argmin_{a(\cdot) \in \mathcal{P}_\theta} \text{KL}(p(\theta | h_T) \| a(\theta)) = p(\theta | h_T),
\end{equation*}
i.e. true Bayesian posterior. Therefore \textsc{Action-BED} solves the following joint optimisation problem:
\begin{align}
\begin{split}
    \label{eqn:joint_opt_problem_EIG}
    (\pi_a, \pi_d)^B &= \argmin_{(\pi_a, \pi_d) \in \Pi_a^{\mathcal{P}_\theta} \times \Pi_d} \mathrm{EFL}_{\text{log-loss}}(\pi_d,\pi_a), \\
    \mathrm{where} \quad
    \mathrm{EFL}_{\text{log-loss}}(\pi_d,\pi_a) &= \E_{p(\theta)p(\y|\theta;\pi_d)} \left[-\log (\pi_a(h_T)(\theta))\right],
    \end{split}
\end{align}
and $\Pi_a^{\mathcal{P}_\theta} = \{\pi \mid \pi : \mathcal{Y}^T \times \Xi^T \to \mathcal{P}_\theta\}$. In practice, we must parameterise $\Pi_a^{\mathcal{P}_\theta}$, or a subset thereof, in order to pose a tractable parametric optimisation problem. We do this by specifying an amortised posterior approximation \emph{policy} $\pi_a^\psi: \mathcal{Y}^T \times \Xi^T \to \mathcal{P}_\theta$, defined as:
\begin{equation*}
    \pi_a^\psi(h_T)(\cdot) = q_\psi(\cdot | h_T),
\end{equation*}
where $\psi \in \Psi$ parametrises an amortised variational family. Thus, \textsc{Action-BED} solves the following parametric optimisation:
\begin{equation*}
    \min_{\phi, \psi} \mathbb{E}_{p(\theta)p(\y|\theta; \pi_d^\phi)} \left[ -\log \pi_a^\psi(h_T)(\theta)\right] = \min_{\phi, \psi} \mathbb{E}_{p(\theta)p(\y|\theta; \pi_d^\phi)} \left[ -\log q_\psi(\theta | h_T)\right].
\end{equation*}

This objective is equivalent to the variational posterior objectives introduced by \citet{foster2020unified, foster_variational_2019}, who solve the following problem:
\begin{equation*}
    \max_{\phi, \psi} \mathbb{E}_{p(\theta)p(\y|\theta; \pi_d^\phi)} \left[\log q_\psi(\theta | h_T) - \log p(\theta)\right].
\end{equation*}

\newpage
\section{Reparameterisation of Simulator and Regularity Conditions}
\label{app:reparam}

In this Section, we provide details of the reparameterisation function used in our EFL gradient expression \Cref{eqn:EFL_reparam_gradient}. We further provide the precise regularity conditions required to differentiate under the integral sign in the reparameterisation trick. 

We recall the application of the reparameterisation trick as follows:
\begin{align}
    \nabla_{\phi, \psi} \text{EFL}(\phi,\psi)
    & = \nabla_{\phi, \psi} \mathbb{E}_{p(\theta)\,p(\y \mid \theta;\pi_d^\phi)}
    \left[
    \ell\big(\pi_a^\psi(h_T), \theta\big)
    \right] \\ 
    &=
    \mathbb{E}_{p(\theta)q(\varepsilon_{1:T})}
    \left[ \nabla_{\phi, \psi}
        \ell\big(
        \pi_a^\psi({h}_T(\theta,\varepsilon_{1:T};\pi_d^\phi)),
        \theta
        \big)
    \right], \label{eqn_rep:EFL_reparam_gradient}
\end{align}

where the simulator $p(\y \mid \theta; \pi_d^\phi)$ was reparameterised in order to write $h_T = h_T(\theta,\varepsilon_{1:T};\pi_d^\phi)$ for $\varepsilon_{1:T}\sim q(\varepsilon_{1:T})$. We firstly provide the exact form of the reparameterisation function in \cref{app:reparam_fn} before providing regularity conditions on the loss and reparameterisation functions in \cref{app:reg_conditions_loss} and consequently on the BED model in \cref{app:reg_conditions_BED}.

\subsection{Reparametrisation Functions}
\label{app:reparam_fn}

Recall the data generating process under a given design policy $\pi_d^\phi$ is given by $p(\y \mid \theta; \pi_d^\phi) = \prod_{t=1}^T p(y_t \mid \theta, \xi_t = \pi_d^\phi(h_{t-1}), h_{t-1})$. Furthermore recall the data $h_T$ is defined as $h_T = \{\y, \x\}$ and designs $\xi_t$ are recovered auto-regressively using the design policy and the previous data $\xi_t = \pi_d^\phi(h_{t-1})$. 

We can therefore build the reparameterisation of $h_T$ recursively as follows. At the first step, we have $h_1 = \{y_1, \xi_1\}$ where 
\begin{equation*}
    y_1\sim p(\cdot \mid \theta, \xi_1), \quad \xi_1 = \pi_d^\phi(\emptyset).
\end{equation*}
Introducing noise variable $\varepsilon_1 \sim q_1(\cdot)$, we can directly write the pair $(y_1, \xi_1)$ as a function of $(\theta, \epsilon_1; \pi_d^\phi)$, specifically:
\begin{equation*}
    (y_1, \xi_1) = \left(\tilde{g}(\theta, \epsilon_1; \pi_d^\phi), \pi_d^\phi(\emptyset)\right) =: h_1(\theta, \epsilon_1; \pi_\phi^d).
\end{equation*}
At the second time step, we now have
\begin{equation*}
    y_2\sim p(\cdot \mid \theta, \xi_2, h_1), \quad \xi_2 = \pi_d^\phi(h_1).
\end{equation*}
Again introducing noise variable $\varepsilon_2 \sim q_2(\cdot)$, we can express the pair $(y_2, \xi_2)$ as a function of $(\theta, \epsilon_2, h_1; \pi_d^\phi)$ as follows:
\begin{equation*}
    (y_2, \xi_2) = \left(\tilde{g}(\theta, \epsilon_2, h_1; \pi_d^\phi), \pi_d^\phi(h_1)\right) =: \tilde{h}_2(\theta, \varepsilon_2, h_1; \pi_d^\phi)
\end{equation*}
By substituting the reparameterisation of the previous step, i.e. $h_1 = h_1(\theta, \epsilon_1; \pi_\phi^d)$ for $\varepsilon_1 \sim q_1(\cdot)$ and stacking with $h_1$, we obtain:
\begin{equation*}
        h_2 = (y_1, \xi_1, y_2, \xi_2) = (h_1, y_2, \xi_2) = (h_1, \tilde{h}_2) =: h_2(\theta, \epsilon_{1:2}; \pi_d^\phi).
\end{equation*}
We continue this argument recursively, arriving at a reparameterisation function for the entire experimental history $h_T = h_T(\theta, \epsilon_{1:T}; \pi_d^\phi)$ for $\varepsilon_{1:T} \sim \prod_{t=1}^T q_t(\cdot)$. Note that we may sometimes abuse notation slightly and write this as $h_T = h_T(\theta, \varepsilon_{1:T}, \phi)$, with the dependence on the policy $\pi_d$ implicit.

\subsection{Conditions for Differentiation Under the Integral Sign}
\label{app:reg_conditions_loss}

Let $(\Omega,\mathcal{F},\mu)$ be a measure space on which the random variables $\theta$ and $\varepsilon$ are defined. Let $\Phi \times \Psi \subset \mathbb{R}^{d_\phi} \times \mathbb{R}^{d_\psi}$ be an open set. Equation~\ref{eqn_rep:EFL_reparam_gradient} is justified under the following conditions:
\begin{enumerate}
\item The distributions $p(\theta)$ and $p(\varepsilon_{1:T})$ are independent of $(\phi,\psi)$;

\item For $\mu$-almost every $(\theta,\varepsilon_{1:T}) \in \Omega$, the mapping
\[
(\phi,\psi) \mapsto
\ell\big(\pi_a^\psi(h_T(\theta, \varepsilon_{1:T}; \pi_d^\phi)),\theta\big)
\]
is $\mathcal{F}$-measurable and integrable:
\[
\int_\Omega
\left|
\ell\big(\pi_a^\psi(h_T(\theta, \varepsilon_{1:T}; \pi_d^\phi)),\theta\big)
\right|
\,\mathrm{d}\mu(\theta, \varepsilon_{1:T})
<
\infty.
\]

\item For $\mu$-almost every $(\theta,\varepsilon_{1:T}) \in \Omega$, the mapping is differentiable in $(\phi, \psi)$ and the partial derivatives
\[
\partial_\phi \ell\big(\pi_a^\psi(h_T(\theta, \varepsilon_{1:T}; \pi_d^\phi)),\theta\big),
\quad
\partial_\psi \ell\big(\pi_a^\psi(h_T(\theta, \varepsilon_{1:T}; \pi_d^\phi)),\theta\big)
\]
exist (which holds under standard smoothness assumptions on $h_T$, $\pi_a^\psi$, and $\ell$);

\item For each compact $K \subset \Phi \times \Psi$, there exist integrable functions
\[
g_\phi : \Omega \to [0,\infty),
\quad
g_\psi : \Omega \to [0,\infty)
\]
such that for all $(\phi, \psi) \in K$ and $\mu$-almost every $(\theta, \varepsilon_{1:T}) \in \Omega$,
\begin{equation}
\big|
\partial_\phi \
\ell\big(\pi_a^\psi(h_T(\theta, \varepsilon_{1:T}; \pi_d^\phi)),\theta\big)
\big|
\le g_\phi(\omega),
\quad\text{and}\quad
\int_\Omega g_\phi(\theta, \varepsilon_{1:T})\,\mathrm{d}\mu(\theta, \varepsilon_{1:T}) < \infty,
\end{equation}
and
\begin{equation}
\big|
\partial_\psi \
\ell\big(\pi_a^\psi(h_T(\theta, \varepsilon_{1:T}; \pi_d^\phi)),\theta\big)
\big|
\le g_\psi(\omega),
\quad\text{and}\quad
\int_\Omega g_\psi(\theta, \varepsilon_{1:T})\,\mathrm{d}\mu(\theta, \varepsilon_{1:T}) < \infty.
\end{equation}
\end{enumerate}

These conditions follow from the differentiation-under-the-integral theorem. The following section explains how these sufficient conditions for differentiating under the integral sign can be derived from assumptions on the model.

\subsection{Conditions on the BED model}
\label{app:reg_conditions_BED}

Let $\Phi \subset \mathbb{R}^{d_\phi}$ and $\Psi \subset \mathbb{R}^{d_\psi}$ be open parameter sets. In order to have the map 
    \[
    (\phi,\psi) \mapsto \ell(\pi_a^\psi(h_T(\theta, \varepsilon_{1:T}; \pi_d^\phi)),\theta)
    \]
    differentiable, we have to satisfy the following conditions: 

\begin{enumerate}
    \item \textbf{Differentiable design network}: 
    The design-induced history
    \begin{align*}
        h_T &: \Theta \times \mathcal{E} \times \Phi \to \mathcal{H}_T \\
        &(\theta,\varepsilon_{1:T}, \phi) \mapsto h_T(\theta,\varepsilon_{1:T}; \pi_d^\phi),
    \end{align*}
    is measurable in $(\theta,\varepsilon_{1:T})$ and differentiable in $\phi$ for almost all $(\theta,\varepsilon_{1:T})$, with
    \[
        \nabla_\phi h_T(\theta,\varepsilon_{1:T}, \phi)
    \]
    existing and being finite. In particular, in all deep learning architectures considered to model $h_T$, this is implemented as a composition of differentiable operations with respect to~$\phi$ (e.g.\ affine maps and smooth or piecewise-smooth activations), so that it is differentiable almost everywhere.

    \item \textbf{Differentiable prediction network}:
    The downstream prediction network
    \[
        \pi_a^\psi : \mathcal{H}_T \to \mathcal{A}
    \]
    is measurable in its input and differentiable in $(\psi,h_T)$, i.e.
    \[
        (\psi,h) \mapsto \pi_a^\psi(h)
    \]
    is differentiable for all $(\psi,h)$, or at least almost everywhere, with Jacobians
    \[
        \nabla_\psi \pi_a^\psi(h), \qquad \nabla_h \pi_a^\psi(h)
    \]
    existing and finite. This holds in particular when $\pi_a^\psi$ is a neural network with standard smooth (or piecewise-linear) activations such as $\mathrm{tanh}$, $\mathrm{softplus}$, $\mathrm{ReLU}$ or $\mathrm{leaky-ReLU}$, which are differentiable almost everywhere.

    \item \textbf{Differentiable loss}:
    The downstream loss
    \[
        \ell : \mathcal{A} \times \Theta \to \mathbb{R}, \quad (a,\theta) \mapsto \ell(a,\theta),
    \]
    is differentiable in its first argument $a$ for all $(a,\theta)$ (or almost everywhere), with gradient
    \[
        \nabla_a \ell(a,\theta)
    \]
    existing and finite (e.g.\ mean-squared error, cross-entropy, etc.).
\end{enumerate}
Under these conditions, for almost all $(\theta,\varepsilon_{1:T})$, the composite map
\[
(\phi,\psi) \longmapsto \ell\big(\pi_a^\psi(h_T(\theta,\varepsilon_{1:T},\phi)),\theta\big)
\]
is differentiable, and its gradient can be computed by the chain rule. 

\newpage
\section{Gradient Computations}

\subsection{Gradient Computations for Separated Networks}
\label{appendix:gradient_computations}

We denote, for a single draw $(\theta,\varepsilon_{1:T}) \sim p(\theta)q(\varepsilon_{1:T})$,
\begin{equation*}
L(\phi,\psi;\theta,\varepsilon_{1:T})
:=
\ell\big(\pi_a^\psi(h_T(\theta,\varepsilon_{1:T};\pi_d^\phi)),\,\theta\big),
\end{equation*}
so that
\begin{equation}
\widehat{\nabla_{\phi,\psi}}\,\mathrm{EFL}(\phi,\psi)
=
\nabla_{\phi,\psi} L(\phi,\psi;\theta,\varepsilon_{1:T}).
\end{equation}

\paragraph{Gradient with respect to $\psi$.}
Since $h_T$ does not depend on $\psi$, we have
\begin{align}
\nabla_\psi L(\phi,\psi;\theta,\varepsilon_{1:T})
&=
\nabla_\psi \,\ell\big(\pi_a^\psi(h_T(\theta,\varepsilon_{1:T};\pi_d^\phi)),\,\theta\big) \\
&=
\nabla_a \ell\big(a,\theta\big)\Big|_{a=\pi_a^\psi(h_T)}
\;\nabla_\psi \pi_a^\psi\big(h_T(\theta,\varepsilon_{1:T};\pi_d^\phi)\big),
\end{align}
where $\nabla_a \ell(a,\theta)$ denotes the gradient of the loss with respect to its first argument.

\paragraph{Gradient with respect to $\phi$.}
Here the dependence on $\phi$ is only through the history $h_T(\theta,\varepsilon_{1:T};\pi_d^\phi)$, so the chain rule yields
\begin{align}
\nabla_\phi L(\phi,\psi;\theta,\varepsilon_{1:T})
&=
\nabla_\phi \,\ell\big(\pi_a^\psi(h_T(\theta,\varepsilon_{1:T};\pi_d^\phi)),\,\theta\big) \\
&=
\nabla_a \ell\big(a,\theta\big)\Big|_{a=\pi_a^\psi(h_T)}
\;\nabla_{h_T} \pi_a^\psi\big(h_T(\theta,\varepsilon_{1:T};\pi_d^\phi)\big)
\;\nabla_\phi h_T(\theta,\varepsilon_{1:T};\pi_d^\phi).
\end{align}

\paragraph{Joint gradient estimator.}
Stacking the two components, a single-sample Monte Carlo estimator of the full gradient is thus
\[
\widehat{\nabla_{\phi,\psi}}\,\mathrm{EFL}(\phi,\psi)
=
\begin{pmatrix}
\nabla_\phi L(\phi,\psi;\theta,\varepsilon_{1:T}) \\
\nabla_\psi L(\phi,\psi;\theta,\varepsilon_{1:T})
\end{pmatrix}.
\]
In practice, automatic differentiation computes precisely these quantities by backpropagating through the computational graph
\[
(\phi,\psi,\theta,\varepsilon_{1:T})
\mapsto h_T(\theta,\varepsilon_{1:T};\pi_d^\phi)
\mapsto \pi_a^\psi(h_T)
\mapsto \ell(\pi_a^\psi(h_T),\theta).
\]

\subsection{Gradient Estimator in case of weight sharing}
\label{appendix:gradient_weight_sharing}

We now make explicit the form of the Monte Carlo gradient estimator in the case where both the acquisition and downstream policies share the same encoder network, with a sum-pooling aggregation step of the encodings and a specific prediction head for both. It is a common choice of architecture for policy-based methods. For a single draw $(\theta,\varepsilon_{1:T}) \sim p(\theta)q(\varepsilon_{1:T})$, we consider
the loss
\[
L(\phi_1,\phi_2,\psi;\theta,\varepsilon_{1:T})
:=
\ell\big(\pi_a^\psi(h_T(\theta,\varepsilon_{1:T};\pi_d^{\phi_1,\phi_2})),\,\theta\big),
\]
where the downstream prediction network is defined via a shared DeepSets-style encoder:
\begin{align}
e_t &= E_{\phi_1}(\xi_t,y_t) \in \mathbb{R}^d, \qquad t=1,\dots,T, \\
R(h_T) &= \sum_{t=1}^T e_t 
       = \sum_{t=1}^T E_{\phi_1}(\xi_t,y_t), \\
\pi_a^\psi(h_T) &= f_\psi\big(R(h_T)\big).
\end{align}
A single-sample Monte Carlo estimator of the gradient of the expected future loss is then
\[
\widehat{\nabla_{\phi_1,\phi_2,\psi}}\,\mathrm{EFL}(\phi_1,\phi_2,\psi)
=
\nabla_{\phi_1,\phi_2,\psi}
\,L(\phi_1,\phi_2,\psi;\theta,\varepsilon_{1:T}).
\]

\paragraph{Gradient with respect to the prediction head $\psi$.}
The dependence on $\psi$ is only through $a = f_\psi(R(h_T))$, so the chain rule yields
\begin{align}
\nabla_\psi L(\phi_1,\phi_2,\psi;\theta,\varepsilon_{1:T})
&=
\nabla_\psi\,\ell\big(f_\psi(R(h_T)),\theta\big) \\
&=
\nabla_a \ell\big(a,\theta\big)\Big|_{a=f_\psi(R(h_T))}
\;\nabla_\psi f_\psi\big(R(h_T)\big).
\end{align}

\paragraph{Gradient with respect to the shared encoder parameters $\phi_1$.}
Here, $\phi_1$ affects the loss both through the encoder outputs
$e_t = E_{\phi_1}(\xi_t,y_t)$ and, through the design policy, through the generated history
$h_T(\theta,\varepsilon_{1:T};\pi_d^{\phi_1,\phi_2})$.
Applying the chain rule gives
\begin{align}
\nabla_{\phi_1} L(\phi_1,\phi_2,\psi;\theta,\varepsilon_{1:T})
&=
\nabla_a \ell\big(a,\theta\big)\Big|_{a=f_\psi(R(h_T))}
\;\nabla_{R} f_\psi\big(R(h_T)\big)
\;\nabla_{\phi_1} R(h_T).
\end{align}
The term $\nabla_{\phi_1}R(h_T)$ includes both the direct dependence of the encoder
$E_{\phi_1}$ on $\phi_1$ and the indirect dependence through the history
$h_T(\theta,\varepsilon_{1:T};\pi_d^{\phi_1,\phi_2})$. In particular,
\begin{align}
\nabla_{\phi_1} R(h_T)
&=
\nabla_{\phi_1}
\left[
\sum_{t=1}^T E_{\phi_1}(\xi_t,y_t)
\right],
\end{align}
where $(\xi_t,y_t)_{t=1}^T$ are generated by the design policy
$\pi_d^{\phi_1,\phi_2}$.

\paragraph{Gradient with respect to the design head $\phi_2$.}
The parameters $\phi_2$ enter through the design policy that generates the history
$h_T(\theta,\varepsilon_{1:T};\pi_d^{\phi_1,\phi_2})$. Denoting this dependence explicitly,
the chain rule gives
\begin{align}
\nabla_{\phi_2} L(\phi_1,\phi_2,\psi;\theta,\varepsilon_{1:T})
&=
\nabla_{h_T}\,\ell\big(f_\psi(R(h_T)),\theta\big)
\;\nabla_{\phi_2} h_T(\theta,\varepsilon_{1:T};\pi_d^{\phi_1,\phi_2}),
\end{align}
where the term $\nabla_{h_T}\ell(\cdot)$ itself decomposes through the encoder and pooling:
\begin{align}
\nabla_{h_T}\,\ell\big(f_\psi(R(h_T)),\theta\big)
&=
\nabla_a \ell\big(a,\theta\big)\Big|_{a=f_\psi(R(h_T))}
\;\nabla_{R} f_\psi\big(R(h_T)\big)
\;\nabla_{h_T} R(h_T).
\end{align}
In practice, automatic differentiation computes exactly these gradients by backpropagating 
through the computational graph
\[
(\phi_1,\phi_2,\psi,\theta,\varepsilon_{1:T})
\mapsto h_T(\theta,\varepsilon_{1:T};\pi_d^{\phi_1,\phi_2})
\mapsto \{E_{\phi_1}(\xi_t,y_t)\}_{t=1}^T
\mapsto R(h_T)
\mapsto f_\psi(R(h_T))
\mapsto \ell(\cdot,\theta).
\]
Finally, a Monte Carlo estimator of the full gradient is obtained by averaging these
single-sample gradients over i.i.d.\ draws $(\theta^{(i)},\varepsilon_{1:T}^{(i)})$.

\newpage
\section{Network Architecture}
\label{appendix:architectures}

Our co-optimization framework, introduced in Section~\ref{meth}, is agnostic to the specific parameterization of the design policy $\pi_d^\phi$ and the downstream action policy $\pi_a^\psi$. Any differentiable architecture can be used, provided that the resulting computation graph supports pathwise optimization through the reparameterized simulator. Nevertheless, architectural choices play an important role in practice, since they determine how the experimental history is represented and which inductive biases are made available to the policy. In particular, the architecture can encode exchangeability, temporal ordering, Markovian structure, spatial symmetries, or more general dependencies between past designs and outcomes.

\paragraph{Design policy.}
The design policy $\pi_d^\phi$ is the central component of the acquisition mechanism. At each experimental step, it maps the current history
\[
    h_t = \{(\xi_s,y_s)\}_{s=1}^t
\]
to the next design, or to a distribution over the designs, 
\[
\begin{cases}
    \xi_{t+1} = \pi_d^\phi(h_t) \\
     \xi_{t+1} \sim p_\phi(\xi \mid h_t) =: \pi_d^\phi(h_t)
\end{cases}
\]
In deep adaptive design, this mapping is commonly parameterized through an encoder--aggregator--emitter architecture. First, each observed design--outcome pair $(\xi_s,y_s)$ is embedded by a local encoder,
\[
    r_s = E_{\phi_1}(\xi_s,y_s),
\]
or, more generally, by separate design and outcome encoders followed by a fusion module. The resulting sequence or set of pair embeddings $(r_1,\ldots,r_t)$ is then compressed into a global history representation,
\[
    z_t = A_{\phi_2}(r_1,\ldots,r_t),
\]
where $A_{\phi_2}$ is a history aggregation module. Finally, an emitter network maps this history state to the next design,
\[
    \xi_{t+1} = G_{\phi_3}(z_t),
\]
or to a distribution over candidate designs in stochastic acquisition policies.

The main decision-making capacity of the design policy is often concentrated in the aggregation module. The encoder produces local representations of individual experiments, whereas the aggregator determines how these pieces of evidence are combined to support the next acquisition decision. Different choices of $A_{\phi_2}$ encode different assumptions about the structure of the history. For exchangeable probabilistic models of histories, permutation-invariant architectures such as DeepSets~\citep{zaheer2017deepsets,foster_deep_2021} are natural: they aggregate pair embeddings through a symmetric pooling operation and therefore make the design decision independent of the ordering of observations. For temporally ordered or dynamically evolving systems, recurrent architectures such as LSTMs~\citep{hochreiter1997lstm,ivanova_implicit_2021} can encode the sequential dependence between past interventions and outcomes. Attention-based architectures, including Set Transformers \citep{lee2019settransformer} and Transformer Neural Process (TNP)-style encoders \citep{tnp_ashman24a}, provide a more flexible alternative by allowing the policy to learn interactions between all past design--outcome pairs and to condition the next design on the most informative parts of the history.

\paragraph{Downstream action network.}
The downstream action policy maps the final experimental history to a terminal action,
\[
    a_T = \pi_a^\psi(h_T).
\]
This action may represent a point estimate, a classifier output, a decision rule, or any task-specific prediction required after data acquisition. Since the downstream policy is only evaluated after the acquisition phase and needs to extract information contained in the history, $\pi_a^\psi$ can be implemented with a lighter architecture, or can follow the same encoder--aggregator--emitter structure as the design policy:
\[
    r_s = E_{\psi_1}(\xi_s,y_s), 
    \qquad
    z_T = A_{\psi_2}(r_1,\ldots,r_T),
    \qquad
    a_T = G_{\psi_3}(z_T).
\]
This shared architectural template is especially useful when the downstream action depends on subtle interactions across the entire history. In such cases, using a rich history encoder for $\pi_a^\psi$ can improve the quality of the terminal decision and provide stronger gradients to guide the design policy during co-optimization.

\paragraph{Enhancing coupling through weight sharing.}
When the design policy and downstream action policy are built from compatible encoder--aggregator architectures, one can further strengthen their coupling by sharing part of their representations. For example, both policies may use a common local encoder $E_\omega$ for design--outcome pairs, while retaining separate aggregation and output heads:
\[
    r_s = E_\omega(\xi_s,y_s),
\]
\[
    \xi_{t+1}
    =
    G_d^{\phi}
    \left(
        A_d^{\phi}(r_1,\ldots,r_t)
    \right),
    \qquad
    a_T
    =
    G_a^{\psi}
    \left(
        A_a^{\psi}(r_1,\ldots,r_T)
    \right).
\]
A stronger form of sharing can also be used, where the local encoder and part of the history aggregator are shared, while task-specific heads remain distinct:
\[
    z_t = A_\omega(E_\omega(\xi_1,y_1),\ldots,E_\omega(\xi_t,y_t)),
\]
\[
    \xi_{t+1}=G_d^\phi(z_t),
    \qquad
    a_T=G_a^\psi(z_T).
\]
This form of weight sharing has several advantages. First, it encourages both the design policy and the downstream predictor to build compatible representations of the experimental history. The design policy is then trained to collect data that are useful under the same representational geometry used by the downstream action network. Second, sharing parameters strengthens the coupling between the two objectives: gradients from the downstream loss directly shape the representation used for acquisition. Third, it reduces the number of free parameters, which can improve statistical efficiency and regularize training, especially when the number of simulated trajectories is limited.

\subsection{TNP-Like Policy Architecture}
\label{appendix:design_policy_architecture}

In the experiments, we use a modular neural architecture for the adaptive design policy $\pi_d^\phi$. The architecture follows the general encoder--aggregator--emitter decomposition described above, with a Transformer-based module used to summarize the experimental history. Its purpose is to map the information collected so far into the next design decision, while remaining flexible enough to handle different experimental tasks. The architecture can be viewed as a TNP-like design policy where past design--outcome pairs play the role of context observations, and a learned decision token queries this context in order to produce the next design. We detail below the different modules composing the architecture used throughout the experiments. Experiment-specific dimensions and hyperparameters are denoted by variables, whereas numerical values indicate settings shared across all experiments.

\paragraph{Encoding design--outcome pairs.}
Each past design $\xi_t$ and outcome $y_t$ is first mapped to a latent representation. In practice, the design and outcome are encoded separately,
\[
    e^\xi_t = f_\xi(\xi_t),
    \qquad
    e^y_t = f_y(y_t),
\]

\begin{table}[H]
\centering
\scriptsize
\begin{minipage}{0.46\textwidth}
\centering
\caption{Design encoder architecture.}
\label{tab:design_encoder_architecture}
\begin{tabular}{llll}
\toprule
\textbf{Layer} & \textbf{Description} & \textbf{Dimension} & \textbf{Activation} \\
\midrule
Input & $\xi$ & $\dim_\xi$ & -- \\
H1 & Fully connected & 64 & ReLU \\
Output & Fully connected & $d_{\mathrm{enc},\xi}$ & -- \\
\bottomrule
\end{tabular}
\end{minipage}
\hspace{0.02\textwidth}
\vrule width 0.4pt
\hspace{0.02\textwidth}
\begin{minipage}{0.46\textwidth}
\centering
\caption{Outcome encoder architecture.}
\label{tab:encoder_architecture}
\begin{tabular}{llll}
\toprule
\textbf{Layer} & \textbf{Description} & \textbf{Dimension} & \textbf{Activation} \\
\midrule
Input & $y$ & $\dim_y$ & -- \\
H1 & Fully connected & 64 & ReLU \\
Output & Fully connected & $d_{\mathrm{enc},y}$ & -- \\
\bottomrule
\end{tabular}
\end{minipage}
\end{table}

and then combined into a single representation of the experiment performed at step $t$:
\[
    r_t = f_{\mathrm{pair}}\left([e^\xi_t ; e^y_t]\right).
\]
\begin{table}[H]
\centering
\caption{Fusion module architecture.}
\label{tab:fusion_module_architecture}
\begin{tabular}{llll}
\toprule
\textbf{Layer} & \textbf{Description} & \textbf{Dimension} & \textbf{Activation} \\
\midrule
Input & $[e^\xi; e^{y}]$ & $d_{\mathrm{enc},y} + d_{\mathrm{enc},\xi}$ & -- \\
H1 & Fully connected & $64$ & GELU \\
Output & Fully connected & $d_{\mathrm{enc},r}$ & -- \\
\bottomrule
\end{tabular}
\end{table}
This pair representation captures the local information revealed by one design--outcome interaction. Applying the same construction to all previous steps produces a sequence of history tokens
\[
    R_t = (r_1,\ldots,r_t).
\]
When the task is temporally structured, these tokens can also be augmented with time information, so that the policy can distinguish early and late observations. When the task is exchangeable, the same representation can instead be interpreted as an unordered set of observed experiments. In all experiments reported in this work, we augment these tokens with learned time embeddings, allowing the policy to distinguish observations according to their position in the acquisition trajectory.

\paragraph{Aggregating the experimental history.}
The sequence of pairwise representations $(r_1,\ldots,r_t)$ is processed by a Transformer-based history aggregator. Each representation is first projected to a common embedding dimension, then enriched with a learned type embedding indicating that it corresponds to a history token. A learned temporal embedding is also added to preserve the order of the acquisition trajectory. A learned decision token is appended at the end of the sequence and is assigned its own type embedding, together with the time index corresponding to the next decision step. The resulting token sequence is normalized and passed through a stack of Transformer encoder layers.

The final representation of the decision token is used as the summary of the experimental history. This token acts as a learnable query that attends to all previous design--outcome pairs and extracts the information relevant for choosing the next design. Denoting this representation by $h_t$, the aggregation step can be written as
\[
    h_t
    =
    f_{\mathrm{hist}}(r_1,\ldots,r_t),
    \qquad
    h_t \in \mathbb{R}^{d_h}.
\]
This attention-based aggregation avoids imposing a fixed pooling rule over the experimental history. It allows the policy to model interactions between past observations, account for the temporal order of the acquisitions, and focus on the parts of the trajectory that are most informative for the next experimental decision.

\begin{table}[H]
\centering
\small
\caption{History aggregator architecture. The module maps the sequence of pair representations $(r_1,\ldots,r_t)$ to a single history representation $h_t$.}
\label{tab:history_aggregator_architecture}
\begin{tabular}{lll}
\toprule
\textbf{Component} & \textbf{Configuration} & \textbf{Output shape} \\
\midrule
Input sequence 
& Pair representations $(r_1,\ldots,r_t)$, $r_i \in \mathbb{R}^{d_{\mathrm{enc},r}}$ 
& $t \times d_{\mathrm{enc},r}$ \\

Input adapter 
& Linear layer $d_{\mathrm{enc},r} \rightarrow d_{\mathrm{enc},h}$ 
& $t \times d_{\mathrm{enc},h}$ \\

History token embeddings 
& Add learned history type embeddings 
& $t \times d_{\mathrm{enc},h}$ \\

Time embeddings 
& Add learned time embeddings, enabled 
& $t \times d_{\mathrm{enc},h}$ \\

Decision token 
& Append learned query token with decision type embedding 
& $(t+1) \times d_{\mathrm{enc},h}$ \\

Transformer encoder 
& $2$ layers, $4$ heads, feedforward dimension $128$ 
& $(t+1) \times d_{\mathrm{enc},h}$ \\

Readout 
& Final decision-token representation 
& $d_{\mathrm{enc},h}$ \\
\bottomrule
\end{tabular}
\end{table}

\paragraph{Emitting the next design.}
Finally, an emitter network maps the history representation to the next design:
\[
    \xi_{t+1} = f_{\mathrm{emit}}(h_t).
\]

\begin{table}[H]
\centering
\caption{Emitter network architecture.}
\label{tab:emitter_architecture}
\begin{tabular}{llll}
\toprule
\textbf{Layer} & \textbf{Description} & \textbf{Dimension} & \textbf{Activation} \\
\midrule
Input & History representation $h_t$ & $d_{\mathrm{enc},h}$ & -- \\
H1 & Fully connected & $64$ & GELU \\
Output & Fully connected & $\dim_\xi$ & -- \\
\bottomrule
\end{tabular}
\end{table}

\paragraph{Complete and modular design policy.}
The complete design policy is obtained by composing the modules. Given the experimental history $\mathcal{H}_t=\{(\xi_i,y_i)\}_{i=1}^t$, the policy can be written as
\[
    \pi_\phi(\mathcal{H}_t)
    =
    f_{\mathrm{emit}}
    \circ
    f_{\mathrm{hist}}
    \circ
    f_{\mathrm{pair}}
    \circ
    (f_\xi,f_y)
    (\{(\xi_i,y_i)\}_{i=1}^t).
\]
A key advantage of this architecture is its modularity. The task-specific encoders $f_\xi$ and $f_y$ can be adapted to the structure of each experiment, while the fusion, history aggregation, and emission modules are kept fixed. In low-dimensional settings, the encoders are simple multilayer perceptrons, whereas for structured or high-dimensional observations, $f_y$ can be replaced by a convolutional or patch-based encoder. This architecture therefore provides a common backbone across all experiments, both for our task-driven method and for the DAD/iDAD baselines. The same policy network can be trained end-to-end by backpropagating gradients either through the downstream task loss or through the sequential PCE objective.

\begin{figure}[t]
\centering
\resizebox{\textwidth}{!}{%
\begin{tikzpicture}[
    font=\small,
    >=Latex,
    box/.style={
        draw,
        rounded corners,
        align=center,
        minimum height=1.1cm,
        text width=2.8cm
    },
    io/.style={box, fill=gray!10},
    enc/.style={box, fill=blue!8},
    mod/.style={box, fill=green!8},
    outbox/.style={box, fill=red!8},
    group/.style={draw, dashed, rounded corners, inner sep=8pt},
    arr/.style={->, thick},
    line/.style={thick}
]

\node[io, text width=3.2cm, minimum height=1.35cm] (hist) at (0,0)
{
    \textbf{Experimental history}\\[1mm]
    $\mathcal{H}_t=\{(\xi_i,y_i)\}_{i=1}^t$
};

\node[enc, text width=3.0cm, minimum height=1.1cm] (xienc) at (4.4,0.9)
{
    \textbf{Design encoder}\\[1mm]
    $e_i^\xi=f_\xi(\xi_i)$
};

\node[enc, text width=3.0cm, minimum height=1.1cm] (yenc) at (4.4,-0.9)
{
    \textbf{Outcome encoder}\\[1mm]
    $e_i^y=f_y(y_i)$
};

\node[mod, text width=3.1cm, minimum height=1.25cm] (fusion) at (8.6,0)
{
    \textbf{Pair fusion}\\[1mm]
    $r_i=f_{\mathrm{pair}}(e_i^\xi,e_i^y)$
};

\node[mod, text width=3.7cm, minimum height=1.65cm] (agg) at (13.6,0)
{
    \textbf{History aggregation}\\[0.8mm]
    Decision-token Transformer\\[0.8mm]
    $h_t=f_{\mathrm{hist}}(r_1,\ldots,r_t)$
};

\node[mod, text width=2.9cm, minimum height=1.25cm] (emit) at (18.0,0)
{
    \textbf{Emitter}\\[1mm]
    $\xi_{t+1}=f_{\mathrm{emit}}(h_t)$
};

\node[outbox, text width=2.3cm, minimum height=1.15cm] (nextd) at (21.7,0)
{
    \textbf{Next design}\\[1mm]
    $\xi_{t+1}$
};

\draw[arr] (hist.20) to[out=0,in=180] (xienc.190);
\draw[arr] (hist.-20) to[out=0,in=180] (yenc.170);

\draw[arr] (xienc.east) to[out=0,in=180] (fusion.170);
\draw[arr] (yenc.east)  to[out=0,in=180] (fusion.190);

\draw[arr] (fusion.east) -- node[above] {$r_1,\ldots,r_t$} (agg.west);
\draw[arr] (agg.east) -- node[above] {$h_t$} (emit.west);
\draw[arr] (emit.east) -- (nextd.west);

\node[
    group,
    fit=(xienc)(yenc),
    label={[align=center]below:\textbf{Task-specific encoders}}
] {};

\node[
    group,
    fit=(fusion)(agg)(emit),
    label={[align=center]below:\textbf{Shared policy modules}}
] {};

\end{tikzpicture}
}
\caption{Modular architecture of the design policy. Past designs and outcomes are encoded separately, fused into pair representations, aggregated by a decision-token Transformer, and mapped by the emitter to the next design.}
\label{fig:design_policy_architecture}
\end{figure}
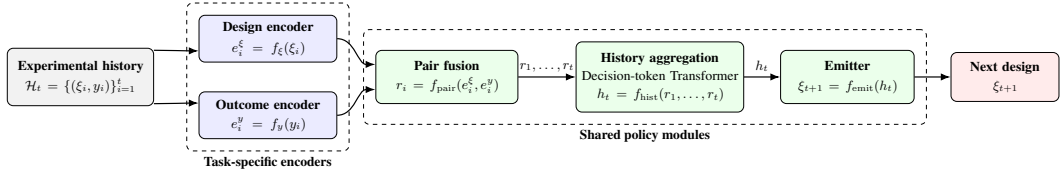

\subsection{Downstream warm-up}
\label{app:downstream_warmup}

In practice, before running the joint optimisation described in \cref{alg:main_alg}, we may initialise the action policy with a short downstream warm-up phase. During this phase, the design policy is kept fixed and is not used to generate designs. Instead, design trajectories are sampled from a task-relevant random distribution \(r(\xi_{1:T})\), and only the downstream-action policy \(\pi_a^\psi\) is trained to minimise the target downstream loss. Given samples \(\theta^b \sim p(\theta)\), simulator noise \(\varepsilon_{1:T}^b \sim q(\varepsilon_{1:T})\), and random designs \(\xi_{1:T}^b \sim r(\xi_{1:T})\), the warm-up gradient estimator is
\begin{equation}
\label{eqn:mc_gradient_est_warmup}
\widehat{\nabla_{\psi}}\,\mathcal{L}_{\mathrm{down}}(\psi)
=
\frac{1}{B} \sum_{b=1}^B
\nabla_{\psi}
\ell\big(
\pi_a^\psi(h_T(\theta^b,\varepsilon_{1:T}^b;\xi_{1:T}^b)),
\theta^b
\big).
\end{equation}
Although the resulting random designs are generally suboptimal, they provide an initial supervised signal for the action policy. Empirically, this warm-up can improve the subsequent joint optimisation: when the design policy starts being trained, the downstream network already maps histories to meaningful actions, yielding more informative gradients for the design policy and leading to a better coupling between acquisition and downstream prediction.

\begin{algorithm}[H]
\caption{\textsc{Action-BED} with downstream warm-up}
\label{alg:main_alg_warmup}
\begin{tcolorbox}
\begin{algorithmic}[1]
\vspace{-4pt}
\Require Policies $\pi_d^\phi, \pi_a^\psi$, initialisations $\phi_0, \psi_0$
\Require Warm-up steps $K_{\mathrm{warm}}$, joint steps $K$, random design distribution $r(\xi_{1:T})$

\Statex \textbf{Downstream warm-up}
\For{$k = 0,1,\dots,K_{\mathrm{warm}}-1$}
    \State $\theta^{1:B}, \varepsilon_{1:T}^{1:B}, \xi_{1:T}^{1:B}
    \sim p(\theta)q(\varepsilon_{1:T})r(\xi_{1:T})$
    \State $g_k^\psi \gets
    \widehat{\nabla_\psi}\,
    \mathcal{L}_{\mathrm{down}}(\psi_k)$ \ (eq. \ref{eqn:mc_gradient_est_warmup})
    \State $\psi_{k+1} \gets \textsc{Update}_a(\psi_k, g_k^\psi)$
\EndFor
\State $\psi_0 \gets \psi_{K_{\mathrm{warm}}}$

\Statex \textbf{Joint Action-BED optimisation}
\For{$k = 0,1,\dots,K-1$}
    \State $\theta^{1:B}, \varepsilon_{1:T}^{1:B} \sim p(\theta)q(\varepsilon_{1:T})$
    \State $g^\phi_k, g_k^\psi \gets
    \widehat{\nabla_{\phi,\psi}}\,
    \mathrm{EFL}(\phi_k, \psi_k)$
    \;(\cref{eqn:mc_gradient_est})
    \State $\phi_{k+1}, \psi_{k+1}
    \gets
    \textsc{Update}(\phi_k, g^\phi_k, \psi_k, g^\psi_k)$
\EndFor
\State \Return $\phi_K, \psi_K$
\end{algorithmic}
\end{tcolorbox}
\end{algorithm}

\newpage
\section{Bayes-optimal action under MSE and permutation-invariant MSE}
\label{app:posterior-mean-optimality}

We provide here a self-contained derivation of the Bayes-optimal action associated with the downstream decision problem considered in the source location finding experiment in Section~\ref{sec:slf_exp}. Since the general Bayes-action formulation is introduced in the main text, we only recall the minimal setup needed for the derivation. We then derive the posterior-mean optimality result for the standard squared error loss and extend it to the permutation-invariant MSE used in the experiment.

\subsection{Bayes-action setup}

Let $\theta \in \Theta$ denote the latent parameter and let $h_T \in \mathcal{H}_T$ denote the experimental trajectory induced by a fixed data-acquisition policy $\pi_d^\phi$. A downstream action policy is a measurable function
\[
\pi_a^\psi : \mathcal{H}_T \to \mathcal{A},
\]
implemented in practice by the downstream-action network.

For a loss function $\ell : \mathcal{A} \times \Theta \to \mathbb{R}_+$, and a fixed trajectory $h_T$, we call the Bayes-optimal action the minimizer of the posterior expected loss:
\begin{equation}
a^\star(h_T)
=
\arg\min_{a \in \mathcal{A}}
\mathbb{E}_{p(\theta \mid h_T)}
\left[
\ell(a,\theta)
\right].
\label{eq:bayes_action_general_appendix}
\end{equation}

\subsection{Standard MSE loss}

We first specialize to the standard squared error loss. We assume that $\Theta = \mathcal{A} \subseteq \mathbb{R}^d$, which is the case for the source location finding task, and we define
\[
\ell_{\mathrm{MSE}}(\hat{\theta},\theta)
=
\|\hat{\theta}-\theta\|_2^2.
\]
For a fixed trajectory $h_T$, the conditional risk is
\begin{equation}
\mathcal{R}_{\mathrm{MSE}}(\hat{\theta}\mid h_T)
=
\mathbb{E}_{p(\theta \mid h_T)}
\left[
\|\hat{\theta}-\theta\|_2^2
\right].
\end{equation}
Expanding the squared norm gives
\begin{align}
\mathcal{R}_{\mathrm{MSE}}(\hat{\theta}\mid h_T)
&=
\mathbb{E}
\left[
\|\hat{\theta}\|_2^2
-2\langle \hat{\theta},\theta\rangle
+\|\theta\|_2^2
\,\middle|\,
h_T
\right] \nonumber \\
&=
\|\hat{\theta}\|_2^2
-2
\left\langle
\hat{\theta},
\mathbb{E}[\theta \mid h_T]
\right\rangle
+
\mathbb{E}
\left[
\|\theta\|_2^2
\mid h_T
\right].
\end{align}
The last term does not depend on $\hat{\theta}$. Differentiating with respect to $\hat{\theta}$ and setting the gradient to zero yields
\[
\nabla_{\hat{\theta}}
\mathcal{R}_{\mathrm{MSE}}(\hat{\theta}\mid h_T)
=
2\hat{\theta}
-
2\mathbb{E}[\theta \mid h_T]
=
0.
\]
Hence,
\begin{equation}
\hat{\theta}^\star_{\mathrm{MSE}}(h_T)
=
\mathbb{E}[\theta \mid h_T].
\label{eq:posterior_mean_appendix}
\end{equation}
Thus, under the standard MSE loss, the Bayes-optimal action is the posterior mean.

\subsection{Permutation-invariant MSE loss}
\label{app:posterior-mean-optimality-invariant}

However, in our setting the latent parameter is not naturally ordered. For example, in the source location finding problem with $K$ exchangeable sources, the parameter can be written as
\[
\theta = (\theta_1,\dots,\theta_K) \in \mathbb{R}^{K \times d},
\]
but the ordering of the sources is arbitrary. In that case, penalizing an estimate $\hat{\theta}$ using the standard MSE would incorrectly distinguish between two estimates that differ only by a permutation of the sources. Let $\mathcal{S}_K$ denote the symmetric group over $K$ elements. For a permutation $\sigma \in \mathcal{S}_K$, write
\[
\sigma \cdot \theta
=
(\theta_{\sigma(1)},\dots,\theta_{\sigma(K)}).
\]
The permutation-invariant MSE is defined as
\begin{equation}
\ell_{\mathrm{PI\text{-}MSE}}(\hat{\theta},\theta)
=
\min_{\sigma \in \mathcal{S}_K}
\frac{1}{Kd}
\sum_{k=1}^K
\left\|
\hat{\theta}_k
-
\theta_{\sigma(k)}
\right\|_2^2.
\label{eq:pi_mse_loss_appendix}
\end{equation}
Equivalently, identifying permutations with permutation matrices $P \in \mathcal{P}_K$, this can be written as
\begin{equation}
\ell_{\mathrm{PI\text{-}MSE}}(\hat{\theta},\theta)
=
\frac{1}{Kd}
\min_{P \in \mathcal{P}_K}
\left\|
\hat{\theta}
-
P\theta
\right\|_F^2,
\label{eq:pi_mse_matrix_appendix}
\end{equation}
where $\|\cdot\|_F$ denotes the Frobenius norm.

For a fixed trajectory $h_T$, the Bayes-optimal action under this loss is therefore
\begin{equation}
\hat{\theta}^\star_{\mathrm{PI\text{-}MSE}}(h_T)
=
\arg\min_{\hat{\theta} \in \mathbb{R}^{K \times d}}
\mathbb{E}_{p(\theta \mid h_T)}
\left[
\frac{1}{Kd}
\min_{P \in \mathcal{P}_K}
\left\|
\hat{\theta}
-
P\theta
\right\|_F^2
\right].
\label{eq:pi_mse_bayes_action_appendix}
\end{equation}

This is the Fréchet mean of the posterior distribution on the quotient space $\mathbb{R}^{K \times d}/\mathcal{S}_K$, where configurations that differ only by a permutation are identified. To make the optimality condition explicit, define, for a given estimate $\hat{\theta}$ and latent parameter $\theta$, an optimal alignment
\begin{equation}
P_{\hat{\theta}}(\theta)
\in
\arg\min_{P \in \mathcal{P}_K}
\left\|
\hat{\theta}
-
P\theta
\right\|_F^2.
\label{eq:optimal_alignment_appendix}
\end{equation}
Assuming that this optimal alignment is unique almost surely under $p(\theta \mid h_T)$, the conditional risk can be written locally as
\begin{equation}
\mathcal{R}_{\mathrm{PI\text{-}MSE}}(\hat{\theta}\mid h_T)
=
\frac{1}{Kd} \cdot
\mathbb{E}_{p(\theta \mid h_T)}
\left[
\left\|
\hat{\theta}
-
P_{\hat{\theta}}(\theta)\theta
\right\|_F^2
\right].
\end{equation}
By the envelope theorem, differentiating the minimized objective with respect to $\hat{\theta}$ gives
\begin{align}
\nabla_{\hat{\theta}}
\mathcal{R}_{\mathrm{PI\text{-}MSE}}(\hat{\theta}\mid h_T)
&=
\frac{2}{Kd} \cdot
\mathbb{E}_{p(\theta \mid h_T)}
\left[
\hat{\theta}
-
P_{\hat{\theta}}(\theta)\theta
\right] \nonumber \\
&=
\frac{2}{Kd} \cdot
\left(
\hat{\theta}
-
\mathbb{E}_{p(\theta \mid h_T)}
\left[
P_{\hat{\theta}}(\theta)\theta
\right]
\right).
\end{align}
Therefore, any differentiable Bayes-optimal action satisfies the fixed-point condition
\begin{equation}
\hat{\theta}^\star_{\mathrm{PI\text{-}MSE}}(h_T)
=
\mathbb{E}_{p(\theta \mid h_T)}
\left[
P_{\hat{\theta}^\star_{\mathrm{PI\text{-}MSE}}(h_T)}(\theta)\theta
\right].
\label{eq:pi_mse_fixed_point_appendix}
\end{equation}

Equation~\eqref{eq:pi_mse_fixed_point_appendix} shows that the Bayes-optimal action under permutation-invariant MSE is the posterior mean after optimally aligning each posterior sample to the estimate itself. The Bayes-optimal action is also not unique in the usual Euclidean sense: if $\hat{\theta}^\star$ is optimal, then $P\hat{\theta}^\star$ is also optimal for any permutation matrix $P \in \mathcal{P}_K$. Hence, the estimator is identifiable only up to permutation, which is precisely the desired invariance property.

\subsection{Interpretation for the downstream network}

Under the standard MSE loss, minimizing the empirical Bayesian risk trains the downstream-action network to approximate the posterior mean mapping
\[
h_T
\longmapsto
\mathbb{E}[\theta \mid h_T].
\]
Under the permutation-invariant MSE, the downstream network instead learns a representative of the posterior Fréchet mean on the quotient space induced by the permutation symmetry:
\[
h_T
\longmapsto
\arg\min_{\hat{\theta}}
\mathbb{E}_{p(\theta \mid h_T)}
\left[
\min_{P \in \mathcal{P}_K}
\frac{1}{Kd}
\left\|
\hat{\theta}
-
P\theta
\right\|_F^2
\right].
\]
Equivalently, it learns an unordered set-valued estimate of the latent sources. During training and evaluation, the optimal matching in
Equation~\eqref{eq:pi_mse_loss_appendix} removes the arbitrary label ordering of the sources, ensuring that two estimates that differ only by a permutation incur the same loss.

\newpage
\section{Experiment Details}
\label{appendix:experiments}

The implementation was carried out in part using the PyTorch library \cite{Paszke2019PyTorch}, and the code is publicly available. Experiments were conducted on a server equipped with an NVIDIA Titan RTX GPU (24GB, Turing architecture) and 24 CPU cores. Each experiment was run on a single GPU.

\subsection{Source Location Finding}
\label{appendix:locationfinding}

The \emph{source location finding} problem is a canonical benchmark in BED, in which the objective is to infer the unknown spatial positions of $K$ signal-emitting sources $\theta_1,\ldots,\theta_K \in \mathbb{R}^2$ from a sequence of noisy measurements. Following the experimental setting of \citep{foster_deep_2021}, at each design step $t$, an agent selects a sensor location $\xi_t \in \mathbb{R}^2$, and observes a scalar response $y_t \in \mathbb{R}$ corresponding to the aggregation of the signals emitted by the $K$ sources, whose expected magnitudes decay with the squared distance to the sensor according to a known physical attenuation model. A commonly used formulation assumes
\begin{equation}
    \log (y_t) = \log \left(b + \sum_{k=1}^{K} \frac{A}{m + \| \xi_t - \theta_k \|^2}\right) + \sigma\,\varepsilon_t,
    \qquad \varepsilon_t \sim \mathcal{N}(0,1),
    \label{eq:lf-model}
\end{equation}
where $A>0$ denotes the signal amplitude, $b$ is a background offset, $m>0$ is a stabilising constant preventing singularities at zero distance, and $\sigma$ controls the measurement noise level. Equivalently, denoting the noiseless aggregated intensity by
\[
    \mu(\theta,\xi)
    =
    b + \sum_{k=1}^{K} \frac{A}{m + \|\xi-\theta_k\|^2},
\]
the observation model can be written in terms of noisy measurements of the log total intensity, yielding the prior--likelihood specification
\[
\theta_k \overset{\text{i.i.d.}}{\sim} \mathcal{N}(0_2, I_2),
\qquad
\log y \mid \theta, \xi \sim \mathcal{N}\!\big(\log \mu(\theta,\xi), \sigma^2\big),
\]
where $\theta=(\theta_1,\ldots,\theta_K)\in\mathbb{R}^{K\times 2}$. This formulation clearly shows that the model satisfies the reparameterization assumption~\ref{app:reparam_fn}. The goal is to localize the sources with the highest possible accuracy: we seek to collect data from which a downstream model can be trained to identify all sources as effectively as possible. In our experimental setting, we consider the case $K=2$.
\\ \\
The task is challenging because information is spatially structured and highly non-uniform: sensor placements far from the source yield almost uninformative responses, while overly greedy exploitation may lead to premature localisation around suboptimal regions. From an inference perspective, the posterior distribution $p(\theta \mid y_{1:T}, \xi_{1:T})$ is typically multi-modal at early stages and becomes sharply concentrated as informative measurements accumulate. Thus, in this setting, amortised methods such as DAD aim to query maximally informative sensors and should generate sequences of designs that are effective for the prediction task. 

\begin{table}[h]
    \caption{\textbf{Location Finding Parameters}}
    \centering
    \label{tab:lf-mlp-hparams}
    \begin{tabular}{l c}
        \toprule
        \textbf{Parameter} & \textbf{Value} \\
        \midrule
        Number of sources, $K$ & 2 \\
        $\dim_\xi$ & 2 \\
        $\dim_y$ & 1 \\
        $T$ & 30 \\
        Base signal, $b$ & $10^{-1}$ \\
        Max. signal, $m$ & $10^{-4}$ \\
        $A$ & $1$ \\
        Signal noise, $\sigma$ & 0.5 \\
        \bottomrule
    \end{tabular}
\end{table}

\paragraph{Downstream Objective.}
In this experimental setting, the downstream objective is to predict the source locations. Prediction quality is evaluated using two losses: the \textit{Mean Squared Error} (MSE) and a logarithmic variant, denoted \textit{Quadratic-Logarithmic Loss} (QLL). 
For problems with multiple sources, the source labels become exchangeable, meaning that permuting their ordering leaves both the likelihood and the signal observed by the sensor unchanged. Hence, the downstream task is intrinsically permutation-invariant. What matters is the predicted set of source locations, not the arbitrary ordering assigned to its elements. Consequently, the standard MSE loss is not the most relevant evaluation criterion in this setting, since it imposes a fixed correspondence between predicted and true sources and may penalize predictions that are correct up to a permutation. Therefore, we use in this experiment the permutation invariant MSE over the losses, defined as
\begin{align}
\ell_{\mathrm{MSE}}(\theta, \hat{\theta}) 
&= \min_{\sigma \in \mathcal{S}_K}
\frac{1}{2K} \cdot
\sum_{k=1}^K
\|\hat{\theta}_{\sigma(k)} - \theta_k\|^2,
\end{align}
where $\mathcal{S}_K$ is the set of permutations of $\{1, \cdots, K\}$ and $\hat{\theta}$ the point prediction of source locations. 
This optimization problem is significantly more complex than in the single-source case. The expectation involves a minimization over permutations inside the loss, leading to a highly non-linear and generally non-convex objective. This phenomenon poses significant challenges during training: 
\begin{itemize}
    \item \textbf{Label switching and multimodality.}  
    The posterior distribution is invariant under permutations of the sources, which typically results in a multimodal posterior with $K!$ symmetric modes.
    
    \item \textbf{Failure of the posterior mean.}  
    Averaging over these modes leads to a posterior mean that does not correspond to any meaningful configuration of sources, often placing predicted sources between true locations or collapsing them toward the center of mass.
\end{itemize}
As a result, under this permutation-invariant loss, the Bayes estimator is not the usual posterior mean on the labelled parameter space, but rather the posterior Fréchet mean induced by the optimal-assignment metric. Equivalently, the Bayes-optimal action is a fixed point where each predicted source is equal to the posterior mean of the true sources optimally matched to it; see Appendix~\ref{app:posterior-mean-optimality-invariant}. This loss is adopted because it directly reflects the objective of the experiment, namely recovering the unordered set of source locations. For the \textit{Quadratic-Logarithmic Loss}, we similarly account for the permutation symmetry of the source configuration by minimizing over all possible assignments between the predicted and true source locations. Specifically, we define the permutation-invariant QLL as 
\begin{equation}
\ell_{\mathrm{QLL}}
\bigl(\theta,\hat{\theta}\bigr)
=
\min_{\sigma \in \mathcal{S}_K}\frac{1}{KD}
\sum_{k=1}^{K}
\sum_{d=1}^{D}
\left[
\left(
\theta_{\sigma(k),d}
-
\hat{\theta}_{k,d}
\right)^2
+
\log\!\left(
\left(
\theta_{\sigma(k),d}
-
\hat{\theta}_{k,d}
\right)^2
+
\epsilon
\right)
\right]
\label{eq:log_mse}
\end{equation}

where $\epsilon = 10^{-6}$ is used for numerical stability. By selecting the permutation that minimizes the loss, this formulation accounts for the fact that source identities are interchangeable, without imposing an arbitrary ordering on the source locations. Consequently, the downstream network is trained to predict the unordered source configuration from the observed trajectory $\{(\xi_i,y_i)\}_{i=1}^T$.

\paragraph{Experimental Details.}
In this experiment, we consider three baseline methods: 
\begin{enumerate}
    \item \textbf{\textsc{Random}}, where designs are naively sampled at random from the marginal prior distribution of a single source location.
    \item \textbf{\textsc{DAD}}, the standard deep policy-based method introduced in \citep{foster_deep_2021}. It parameterises the adaptive design policy with a neural network and is trained using its native expected-information-gain objective. To ensure a fair comparison, we replace the architectural backbone of the original implementation with the same policy network used in our method.
    \item \textbf{\textsc{ALINE}}, a recent policy-based method that achieves among the strongest reported performance in the literature. Its effectiveness is driven by a TNP-like policy architecture, a dense trajectory-level reward objective, and an internal posterior-approximation mechanism that amortises posterior inference. We retain the architecture proposed in the original paper, which is already close in spirit to the transformer-based policy network used in our method.
\end{enumerate}
To compare methods on this task, each baseline is first trained using its native design objective, derived from an EIG-based criterion, thereby learning an acquisition policy independently of the downstream predictor. Once this policy is fixed, a downstream network is trained on trajectories generated by the learned policy to predict the source locations from the resulting design--observation history. All methods are then evaluated through the predictive accuracy of this downstream network, ensuring that the comparison reflects performance on the final source-localization task rather than on the surrogate design objective alone.

The objective of this experiment is threefold. First, we evaluate our method from the perspective of downstream predictive performance. In particular, we assess whether the learned design policy produces informative trajectories that enable the downstream network to accurately solve the source-location point-estimation problem. Second, we evaluate the quality of the learned design policy using the usual lower and upper bounds on the EIG, estimated respectively by sPCE and sNMC. 
 Third, we compare methods in terms of computational cost, measured by the wall-clock training time on a single GPU. We distinguish between the cost of learning the design policy itself and the total training cost, which, for the baselines, also includes the additional training of the downstream prediction network.

\paragraph{Architectures. }
We use the same policy-network architecture for \textsc{Action-BED} and \textsc{DAD}, detailed in Appendix~\ref{appendix:design_policy_architecture}. At the start of each rollout, before any design--outcome observations are available, the initial history encoding is parameterised as a learnable vector. The only architectural hyperparameters varied across methods are the encoding dimensions, which are reported in Table~\ref{tab:architecture-dimensions}
\begin{table}[!htbp]
\centering
\small
\setlength{\tabcolsep}{6.0pt}
\renewcommand{\arraystretch}{1.15}
\begin{threeparttable}
\caption{\small \textbf{Policy architecture encoding dimensions}.}
\label{tab:architecture-dimensions}
\begin{tabular}{lcccc}
\toprule
\textbf{Method} 
& $d_{\mathrm{enc}, \xi}$ 
& $d_{\mathrm{enc}, y}$
& $d_{\mathrm{enc}, r}$ 
& $d_{\mathrm{enc}, h}$ \\
\midrule
\textsc{DAD} 
& 32 & 32 & 32 & 64 \\
\textsc{Action-BED} 
& 32 & 32 & 32 & 64 \\
\bottomrule
\end{tabular}
\end{threeparttable}
\end{table}

In this task, weight sharing between the design policy and the downstream network did not yield a clear performance improvement without further optimisation. We therefore use a simple fully connected MLP for the downstream network, which maps the design--outcome history $\{(\xi_i, y_i)\}_{i=1}^T$ to a point estimate of the source locations. The architecture is detailed in Table~\ref{tab:downstream_net_architecture}.
\begin{table}[!htbp]
\centering
\caption{Downstream network architecture.}
\begin{tabular}{llll}
\toprule
\textbf{Layer} & \textbf{Description} & \textbf{Dimension} & \textbf{Activation} \\
\midrule
Input & $h_T$ & $T \times (\dim_\xi + \dim_y) = 90$ & -- \\
Hidden layer 1 & Fully connected & 512 & GELU \\
Hidden layer 2 & Fully connected & 256 & GELU \\
Hidden layer 3 & Fully connected & 128 & GELU \\
Output & $\hat{\theta}$ & $4$ & -- \\
\bottomrule
\end{tabular}
\label{tab:downstream_net_architecture}
\end{table}

\paragraph{Training procedure.}
During training, all policy networks are optimised from data generated online by rollouts of the corresponding design policy. At each iteration, latent parameters $\theta$ are sampled from the prior distribution $p(\theta)$, and the policy is rolled out autoregressively over the experimental horizon. 
The hyperparameters used for the policy networks are reported in Table~\ref{tab:lfmlphparams_baselines}. ALINE is implemented using its native architecture and trained with the hyperparameters reported in the original paper for the same source-location finding task. In particular, ALINE is trained for $100\,000$ iterations with a joint design policy and posterior approximation network, after an initial burning phase of $20\,000$ iterations during which only the posterior approximation network is optimised. DAD and Action-BED are implemented using the hyperparameters reported in the same table. For Action-BED, we additionally use a warm-up phase in which the downstream prediction network is trained in isolation to infer the source locations from randomly generated designs. After this warm-up, the design and downstream networks are optimised jointly using a shared optimizer. This pretraining stage provides a more informative learning signal to the design policy at the beginning of joint optimization, improves the coupling between the acquisition policy and the downstream predictor, and empirically leads to convergence toward better optima.

\begin{table}[t]
\centering
\scriptsize
\setlength{\tabcolsep}{3.0pt}
\renewcommand{\arraystretch}{1.12}
\caption{\small \textbf{Hyperparameters for source location finding.}
Training and optimization hyperparameters used for the policy networks and the downstream prediction networks. ExpLR denotes an exponential learning-rate scheduler. The \textsc{Random} downstream network uses the same hyperparameters as \textsc{DAD}.}
\label{tab:lfmlphparams_baselines}

\begin{threeparttable}
\begin{minipage}[t]{0.32\textwidth}
\centering
\textbf{DAD}\\[0.35em]
\begin{tabular}{lr}
\toprule
\textbf{Field} & \textbf{Value} \\
\midrule
\multicolumn{2}{l}{\textbf{Policy training}} \\
Batch size & $1024$ \\
Contrastive samples $L$ & $2000$ \\
Policy steps & $150\mathrm{K}$ \\
Policy seed & $42$ \\
\midrule
\multicolumn{2}{l}{\textbf{Policy optimization}} \\
Optimizer & Adam \\
Adam betas & $(0.8,0.998)$ \\
Policy LR & $10^{-3}$ \\
Scheduler & ExpLR \\
Decay $\gamma$ & $0.95$ \\
Decay period & $2000$ \\
Weight decay & $0$ \\
\midrule
\multicolumn{2}{l}{\textbf{Downstream training}} \\
Downstream steps & $150\mathrm{K}$ \\
Downstream seed & $42$ \\
Downstream LR & $5{\times}10^{-4}$ \\
Downstream sched. & ExpLR \\
\bottomrule
\end{tabular}
\end{minipage}
\hfill
\begin{minipage}[t]{0.32\textwidth}
\centering
\textbf{ALINE}\\[0.35em]
\begin{tabular}{lr}
\toprule
\textbf{Field} & \textbf{Value} \\
\midrule
\multicolumn{2}{l}{\textbf{Policy training}} \\
Batch size & $200$ \\
Policy steps & $100\mathrm{K}$ \\
Burn-in steps & $20\mathrm{K}$ \\
Burn-in training & Posterior only \\
Query pool, burn-in & $30$ \\
Query pool, train & $200$ \\
Query proposal & Uniform \\
Reward discount & $1$ \\
Policy weight $\alpha$ & $1$ \\
Policy seed & $123$ \\
\midrule
\multicolumn{2}{l}{\textbf{Policy optimization}} \\
Optimizer & Adam \\
Adam betas & $(0.9,0.999)$ \\
Policy LR & $10^{-3}$ \\
Policy sched. & Cosine \\
Grad. clipping & $\|\nabla\|_{\infty}\leq 1$ \\
\midrule
\multicolumn{2}{l}{\textbf{Downstream training}} \\
Downstream steps & $100\mathrm{K}$ \\
Downstream LR & $5{\times}10^{-4}$ \\
Downstream sched. & ExpLR \\
Downstream decay $\gamma$ & $0.95$ \\
Decay period & $2000$ \\
Downstream seed & 42 \\
\bottomrule
\end{tabular}
\end{minipage}
\hfill
\begin{minipage}[t]{0.32\textwidth}
\centering
\textbf{Action-BED}\\[0.35em]
\begin{tabular}{lr}
\toprule
\textbf{Field} & \textbf{Value} \\
\midrule
\multicolumn{2}{l}{\textbf{Joint training}} \\
Batch size & $2000$ \\
Joint steps & $150\mathrm{K}$ \\
Warm-up steps & $50\mathrm{K}$ \\
Warm-up training & Downstream only \\
Warm-up designs & Random \\
Seed & $42$ \\
\midrule
\multicolumn{2}{l}{\textbf{Optimization}} \\
Optimizer & Adam \\
Adam betas & $(0.8,0.998)$ \\
Policy LR & $5{\times}10^{-4}$ \\
Scheduler & ExpLR \\
Decay $\gamma$ & $0.95$ \\
Decay period & $2000$ \\
Weight decay & $0$ \\
\bottomrule
\end{tabular}
\end{minipage}

\end{threeparttable}
\vspace{-1em}
\end{table}

\paragraph{Evaluation procedure.}
After training, we evaluate the downstream action networks associated with each learned design policy, either jointly trained with the acquisition policy in Action-BED, or trained independently after the design policy has been learned for other baselines. We evaluate downstream predictors using the two downstream losses described above. Evaluation rollouts are generated by first sampling latent source parameters $\theta \sim p(\theta)$ from the prior and then rolling out the corresponding adaptive design policy autoregressively over the full experimental horizon. The downstream network receives the resulting terminal history $h_T$ and outputs a point prediction of the source configuration. All reported downstream losses are computed over $2048$ independent evaluation rollouts, and uncertainty is reported as $\pm 1$ standard error in Table~\ref{tab:location_table}.


\paragraph{EIG-bound evaluation.}
For each trained design policy, we additionally estimate lower and upper bounds on the EIG. This provides a standard, task-agnostic measure of design quality commonly used in the BED literature. It is not the downstream objective considered in our experiment, but DAD and ALINE are explicitly trained with information-theoretic objectives related to EIG maximisation. We recall the definition of these bounds: 
\begin{align}
    \mathcal L_{\mathrm{sPCE}}
    =
    \mathbb{E}
    \left[
        \log
        \frac{
            p(h_T \mid \theta_0;\pi_d^\phi)
        }{
            \frac{1}{L+1}
            \sum_{\ell=0}^{L}
            p(h_T \mid \theta_\ell;\pi_d^\phi)
        }
    \right], \\
    \mathcal U_{\mathrm{sNMC}}
    =
    \mathbb{E}
    \left[
        \log
        \frac{
            p(h_T \mid \theta_0;\pi_d^\phi)
        }{
            \frac{1}{L}
            \sum_{\ell=1}^{L}
            p(h_T \mid \tilde\theta_\ell;\pi_d^\phi)
        }
    \right],
\end{align}
with $\theta_0 \sim p(\theta)$, $\theta_{1:L}\sim p(\theta)$ are independent Monte Carlo samples used to approximate the marginal likelihood, and $L$ being the number of contrastive samples drawn. For all methods, we use a large number of contrastive samples in both estimators so that the resulting bounds are sufficiently tight for comparison, see Table \ref{tab:eig_bound_eval_protocol}. Reported uncertainties correspond to $\pm 1$ standard error over independent rollouts.

\begin{table}[t]
\centering
\scriptsize
\setlength{\tabcolsep}{4.0pt}
\renewcommand{\arraystretch}{1.15}  
\caption{\small \textbf{Source location finding.} Evaluation protocol for design policies. $L$ denotes the number of contrastive samples used to estimate the sequential PCE and NMC bounds. The corresponding batch size is the number of rollout histories used for these EIG-bound estimates. Downstream losses are evaluated on independent rollout batches.}
\label{tab:eig_bound_eval_protocol}

\begin{tabular}{llcccc}
\toprule
\textbf{Method} 
& \textbf{Training objective}
& \multicolumn{2}{c}{\textbf{sPCE / sNMC}}
& \multicolumn{2}{c}{\textbf{Downstream losses}} \\
\cmidrule(lr){3-4}
\cmidrule(lr){5-6}
&
& \textbf{$L$}
& \textbf{Batch size}
& \textbf{Batch size}
 \\
\midrule

Action-BED (MSE)
& Downstream loss
& $10^6$
& $2000$
& $2000$
 \\

Action-BED (QLL)
& Downstream loss
& $10^8$
& $200$
& $2000$
 \\

DAD
& EIG-oriented objective
& $10^6$
& $2000$
& $2000$
 \\

ALINE
& EIG-oriented objective
& $5 \times 10^6$
& $500$
& $2000$
 \\

Random
& None
& $10^5$
& $2000$
& $2000$
 \\

\bottomrule
\end{tabular}



\end{table}

\subsubsection{Qualitative Results}
\label{appendix:qualitative_results_lf}

We complement the quantitative evaluation with qualitative trajectory plots of the learned design policies. We roll out each trained policy and visualize the sequence of selected designs over the experimental horizon. This comparison illustrates the characteristic behavior induced by each method: how quickly it explores the design space, whether it concentrates measurements around informative regions, and how its sampling strategy adapts as observations are collected. These trajectories can provide an interpretable diagnostic of the experimental strategies learned by the different policies.

\begin{figure}[t]
    \centering
    \small
    \begin{subfigure}[t]{0.23\textwidth}
        \centering
        \includegraphics[width=\linewidth]{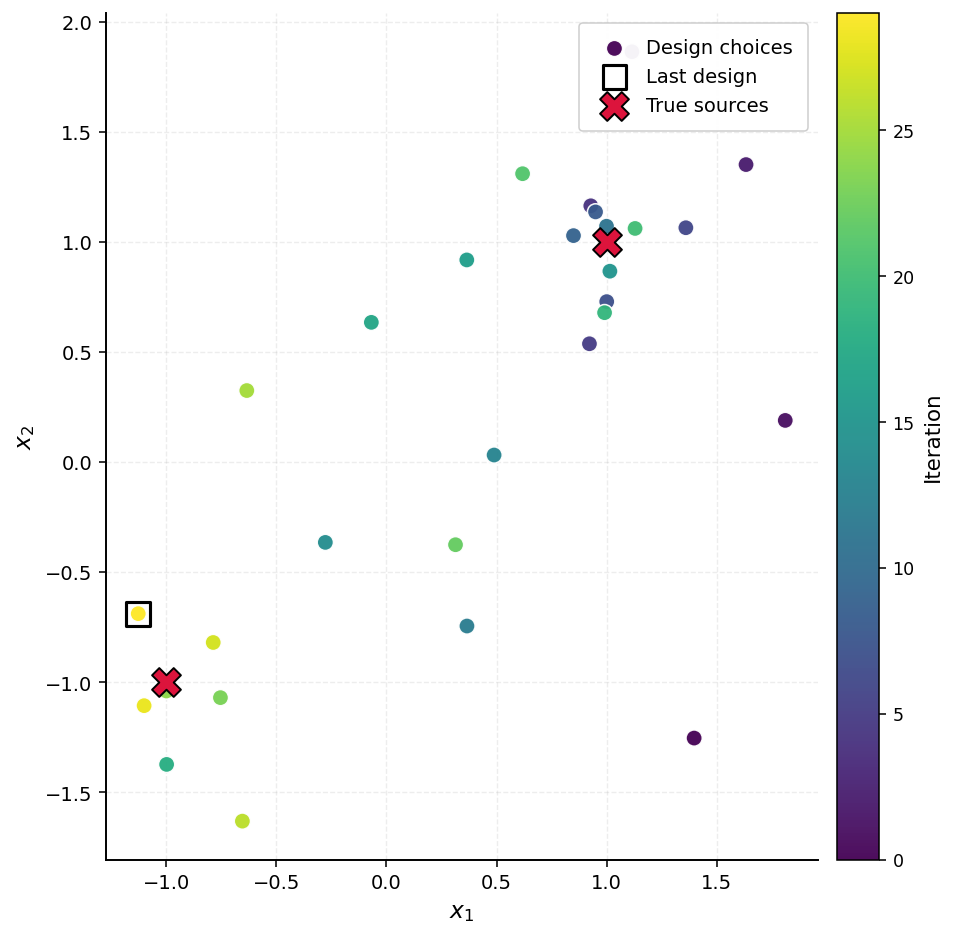}
        \caption{\textsc{DAD}}
    \end{subfigure}
    \hfill
    \begin{subfigure}[t]{0.25\textwidth}
        \centering
        \includegraphics[width=\linewidth]{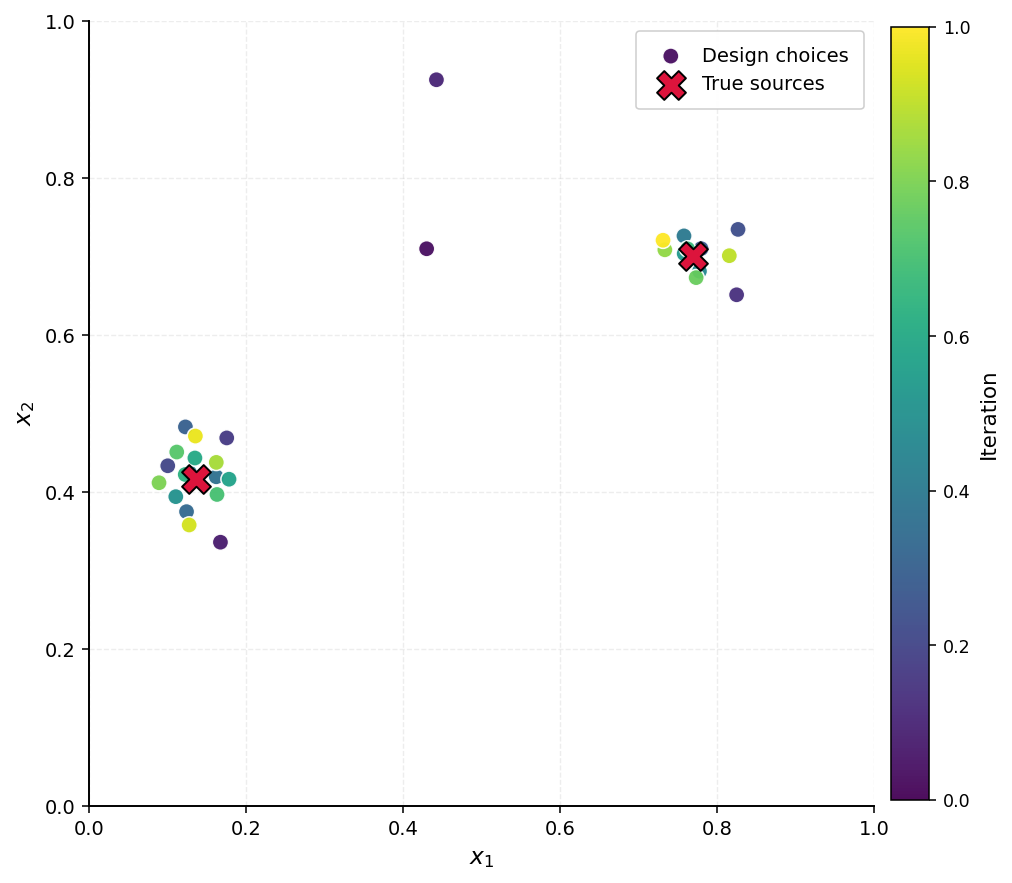}
        \caption{\textsc{ALINE}}
    \end{subfigure}
    \hfill
    \begin{subfigure}[t]{0.26\textwidth}
        \centering
        \includegraphics[width=\linewidth]{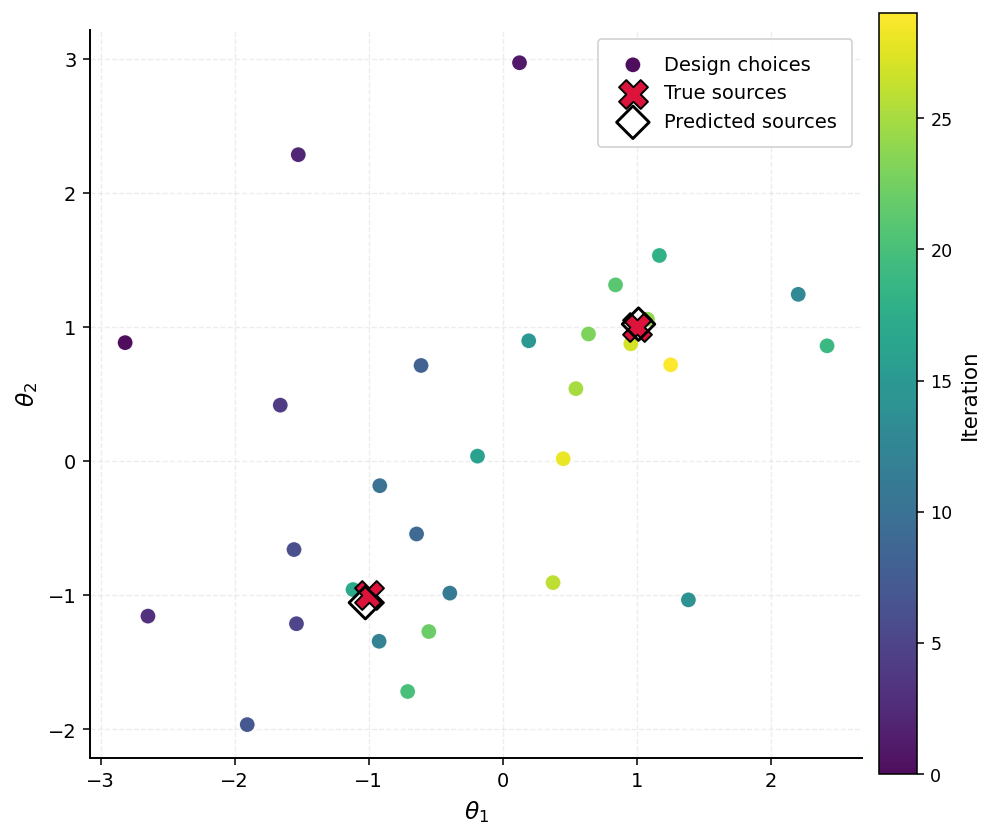}
        \caption{\textsc{Action-BED (MSE)}}
    \end{subfigure}
    \hfill
    \begin{subfigure}[t]{0.23\textwidth}
        \centering
        \includegraphics[width=\linewidth]{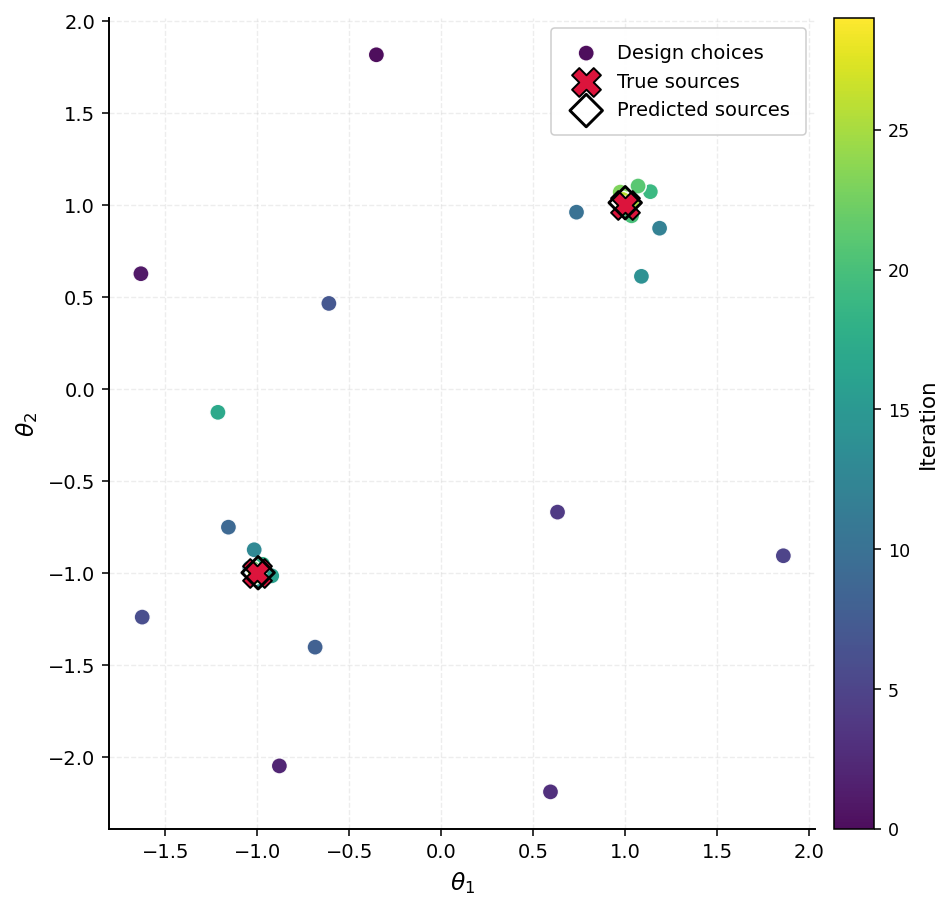}
        \caption{\textsc{Action-BED (QLL)}}
    \end{subfigure}

    \caption{\small \textbf{Comparison of four policy rollouts.}
    Each panel shows one experimental design trajectory under the trained design policy.}
    \label{fig:four_policy_rollouts}
\end{figure}

\subsubsection{Amortising Downstream Bayes-action Selection}
\label{app:amortised_posterior_actions}

This section details the experimental protocol used to compare trained design policies on the downstream task, in terms of the quality of the Bayes-optimal action induced by the experimental trajectories they generate. This serves two purposes. First, it provides a direct way to compare design policies under the decision-theoretic objective of interest, namely the losses incurred by the action selected after observing the experimental history. Second, it illustrates a key practical advantage of \textsc{Action-BED}: since the design policy is trained jointly with an amortised downstream action policy, it directly learns an action network that predicts a near-optimal decision, thereby amortising the otherwise costly posterior-inference step. Indeed, for design policies trained on the EIG, the downstream action-selection step is not directly available from training. Therefore, in order to evaluate such policies on downstream performance, one must proceed in two stages: first generate histories using the fixed trained design policy, and then infer or learn the corresponding downstream action from these histories. This is precisely the comparison performed here.

In this experiment, we evaluate the posterior Fréchet mean under both the MSE and QLL losses. For the MSE objective, this estimator coincides with the Bayes action; see Appendix~\ref{app:posterior-mean-optimality}. This provides a useful explicit posterior-based decision rule against which amortised downstream prediction can be compared. For each trained design policy, we compute the posterior Fréchet mean from posterior samples conditioned on the observed trajectory, and report its empirical downstream performance over the evaluation batch under both loss functions.

\paragraph{Inference protocol.} To compute the explicit posterior Bayes action, we consider three posterior-mean estimators. The first one is based on importance sampling with the prior as proposal distribution. For each experimental history $h_T = (\xi_{1:T}, y_{1:T})$ generated by a fixed design policy, we sample particles $\theta^{(1)},\ldots,\theta^{(N)} \sim p(\theta)$ independently from the prior distribution. These particles are then reweighted according to the likelihood of the observed history,
\[
    w^{(i)}(h_T)
    \;\propto\;
    p(h_T \mid \theta^{(i)}),
    \qquad
    \sum_{i=1}^N w^{(i)}(h_T)=1.
\]
This yields a weighted empirical approximation of the posterior distribution,
\[
    p(\theta \mid h_T)
    \approx
    \sum_{i=1}^N w^{(i)}(h_T)\,\delta_{\theta^{(i)}}.
\]
However, because the latent sources are exchangeable, the direct posterior mean is not Bayes-optimal for the permutation-invariant MSE loss considered and may yield a degenerate estimate of the underlying sources by averaging across permutations. To account for this symmetry, we first optimally permute the sources of each particle. Given a reference ordering $\bar{\theta}$, for instance the current empirical estimate or an initial point estimate, each particle $\theta^{(i)} = (\theta^{(i)}_1,\theta^{(i)}_2)$ is aligned by solving
\[
    \sigma_i^\star
    \in
    \arg\min_{\sigma \in \mathcal{S}_2}
    \sum_{k=1}^2
    \left\|
        \theta^{(i)}_{\sigma(k)} - \bar{\theta}_k
    \right\|^2,
\]
where $\mathcal{S}_K$ denotes the set of permutations of the latent sources. We then define the aligned particle
\[
    \widetilde{\theta}^{(i)}
    =
    \left(
        \theta^{(i)}_{\sigma_i^\star(1)},
        \theta^{(i)}_{\sigma_i^\star(2)}
    \right),
\]
and estimate the MSE Bayes action by the weighted posterior Fréchet mean
\[
    \widehat{a}_{\mathrm{IS}}(h_T)
    =
    \sum_{i=1}^N
    w^{(i)}(h_T)\,
    \widetilde{\theta}^{(i)}.
\]
This aligned posterior mean is then used as the downstream prediction, and its empirical loss is evaluated on the held-out batch. The numerical parameters used for this importance-sampling estimator are reported in Table~\ref{tab:posterior_mean_estimation_params}.

The second estimator is based on Hamiltonian Monte Carlo (HMC)~\citep{neal2011mcmc}. HMC constructs a Markov chain targeting the posterior distribution by simulating Hamiltonian dynamics in an augmented position--momentum space. For each history $h_T$, we use HMC to draw approximate samples from the posterior distribution. In practice, the HMC chain is initialised from a coarse approximation of the optimal action obtained by a lighter version of the prior-proposal importance-sampling procedure described above. This provides a reasonable starting point near a high-posterior-probability region and stabilises the subsequent sampling stage. After burn-in and thinning according to the parameters reported in Table~\ref{tab:posterior_mean_estimation_params}, we obtain posterior samples for each observed trajectory. The HMC-based Bayes action is then computed using the same post-processing procedure as for importance sampling. The empirical posterior mean is then computed over the aligned samples and used as the downstream MSE Bayes-action estimate. We finally evaluate the empirical downstream MSE incurred by this prediction over the evaluation batch.

The third estimator is based on a Laplace approximation of the posterior distribution. For each history $h_T$, we first compute a maximum-a-posteriori estimate of the latent parameters by optimising the log posterior,
\[
    \widehat{\theta}_{\mathrm{MAP}}(h_T)
    \in
    \arg\max_{\theta}
    \left\{
        \log p(h_T \mid \theta)
        +
        \log p(\theta)
    \right\}.
\]
In practice, the optimisation is initialised from a coarse particle-based approximation in order to start near a high-posterior-probability region, and multiple restarts are used to improve robustness. The posterior is then locally approximated by a Gaussian distribution centred at the MAP, with precision given by the negative Hessian of the log posterior at this point,
\[
    p(\theta \mid h_T)
    \approx
    \mathcal{N}
    \left(
        \widehat{\theta}_{\mathrm{MAP}}(h_T),
        \left(H(h_T)+\lambda_{\mathrm{ridge}} I\right)^{-1}
    \right),
\]
where $H(h_T)$ denotes the local posterior precision and $\lambda_{\mathrm{ridge}}$ is a small ridge regularisation parameter added for numerical stability. We then draw samples from this Gaussian approximation using a Cholesky decomposition whenever the regularised precision matrix is positive definite, as described in Table~\ref{tab:posterior_mean_estimation_params}. The Laplace-based Bayes action is finally computed using the same post-processing procedure as for importance sampling and HMC: the Gaussian samples are aligned to account for the exchangeability of the latent sources, and the empirical posterior mean over the aligned samples is used as the downstream MSE Bayes-action estimate. We finally evaluate the empirical downstream MSE incurred by this prediction over the evaluation batch.

\begin{table}[t]
\centering
\scriptsize
\setlength{\tabcolsep}{6.5pt}
\renewcommand{\arraystretch}{1.15}

\begin{threeparttable}

\begin{subtable}[t]{0.40\textwidth}
\centering
\caption{\small Hamiltonian Monte Carlo}
\label{tab:posterior_mean_hmc_params}
\begin{tabular}{lc}
\toprule
\textbf{Parameter} & \textbf{Value} \\
\midrule
Evaluation batch size & $2000$ \\
HMC samples & $5 \times 10^4$ \\
Burn-in steps & $10^4$ \\
Step size $\epsilon^\star$ & $[10^{-3},\,10^{-2}]$ \\
Leapfrog steps & $10$ \\
Thinning interval & $1$ \\
PM init. particles $N_{\mathrm{PM}}$ & $10^5$ \\
\bottomrule
\end{tabular}
\end{subtable}
\hfill
\begin{subtable}[t]{0.50\textwidth}
\centering
\caption{\small Laplace approximation}
\label{tab:posterior_mean_laplace_params}
\begin{tabular}{lc}
\toprule
\textbf{Parameter} & \textbf{Value} \\
\midrule
Evaluation batch size & $2000$ \\
MAP optimisation steps & $1000$ \\
Laplace samples & $5 \times 10^4$ \\
Ridge regularisation $\lambda_{\mathrm{ridge}}$ & $10^{-3}$ \\
Factorisation & Cholesky Decomposition \\
\bottomrule
\end{tabular}
\end{subtable}

\par\vspace{1em}

\begin{subtable}[t]{0.32\textwidth}
\centering
\caption{\small Prior-proposal IS}
\label{tab:posterior_mean_is_params}
\begin{tabular}{lc}
\toprule
\textbf{Parameter} & \textbf{Value} \\
\midrule
Particles $N$ & $5 \times 10^6$ \\
Evaluation batch size & $2000$ \\
\bottomrule
\end{tabular}
\end{subtable}
\vspace{8pt}
\caption{
\textbf{Hyperparameters for posterior Bayes-action estimation.}
We report the parameters used for the three Bayes-action estimators:
prior-proposal IS, Laplace and HMC.
The HMC step size $\epsilon$ is adapted separately for each method so as
to match a target acceptance rate, while remaining constrained to the
interval reported in the table.
}
\label{tab:posterior_mean_estimation_params}

\end{threeparttable}
\end{table}

\begin{table}[t]
\centering
\scriptsize
\setlength{\tabcolsep}{4.6pt}
\renewcommand{\arraystretch}{1.15}
\caption{\small \textbf{Deployment-time Bayes-action estimation under fixed acquisition policies.}
For each pretrained design policy, we compare amortised action networks with posterior-based Bayes-action estimators to assess downstream decision-making from the acquired histories.
Retrained networks are trained post-hoc on histories generated by the fixed design policy to minimise the relevant loss, whereas native \textsc{Action-BED} networks are jointly trained with their design policies. 
Uncertainties are reported as one standard error over $2048$ rollouts.}
\label{tab:deployment_fixed_policy_light}
\begin{tabular}{llcccc}
\toprule
\textbf{Acquisition} & \textbf{Decision Rule}
& \multicolumn{2}{c}{\textbf{Performance} ($\pm$ 1 s.e.)}
\\
\cmidrule(lr){3-4}
& & MSE $(10^{-2})$ $\downarrow$
& QLL $\downarrow$ \\
\midrule

\textsc{Action-BED} (MSE)
& Native network (MSE)       & 0.55 $\pm$ 0.06  & -6.83 $\pm$ 0.03 \\
& Retrained network (MSE)       & 0.86 $\pm$ 0.09  & -6.93 $\pm$ 0.01  \\
& Bayes-act. (IS)               & 0.44 $\pm$ 0.03  & -7.25 $\pm$ 0.03 \\
& Bayes-act. (HMC)              & 0.61 $\pm$ 0.04  & -7.11 $\pm$ 0.02  \\
& Bayes-act. (Laplace)          & 0.46 $\pm$ 0.03      & -7.33 $\pm$ 0.04  \\
\addlinespace[1pt]
\midrule

\textsc{Action-BED} (QLL)
& Native network (QLL)       & 2.8 $\pm$ 0.4    & -10.10 $\pm$ 0.03  \\
& Retrained network (QLL)       & 1.7 $\pm$ 0.2    & -9.97 $\pm$ 0.03 \\
& Bayes-act. (IS)               & 1.3 $\pm$ 0.6    & -9.37 $\pm$ 0.04 \\
& Bayes-act. (HMC)              & 2.00 $\pm$ 0.43  & -10.12 $\pm$ 0.04 \\
& Bayes-act. (Laplace)          & 2.31 $\pm$ 0.71      & \textbf{-10.32} $\pm$ \textbf{0.05} \\
\addlinespace[1pt]
\midrule

DAD
& Retrained network             & 5.6 $\pm$ 0.6    & -6.58 $\pm$ 0.05 \\
& Bayes-act. (IS)               & 3.9 $\pm$ 0.2    & -8.06 $\pm$ 0.05 \\
& Bayes-act. (HMC)              & 3.6 $\pm$ 0.5    & -8.11 $\pm$ 0.05 \\
& Bayes-act. (Laplace)          & 3.5 $\pm$ 0.8      & -8.05 $\pm$ 0.06 \\
\addlinespace[1pt]
\midrule

ALINE
& Retrained network             & 0.21 $\pm$ 0.03  & -8.11 $\pm$ 0.07 \\
& Bayes-act. (IS)               & 0.16 $\pm$ 0.05      & -9.22 $\pm$ 0.03 \\
& Bayes-act. (HMC)              & 0.35 $\pm$ 0.21  & -9.17 $\pm$ 0.03  \\
& Bayes-act. (Laplace)          & \textbf{0.13} $\pm$ \textbf{0.04}      & -9.74 $\pm$ 0.07  \\
\addlinespace[1pt]
\midrule

Random
& Retrained network             & 22.4 $\pm$ 0.8   & -3.16 $\pm$ 0.04 \\
& Bayes-act. (IS)               & 6.1 $\pm$ 0.7    & -4.36 $\pm$ 0.05 \\
& Bayes-act. (HMC)              & 15.7 $\pm$ 2.5   & -3.45 $\pm$ 0.08 \\
& Bayes-act. (Laplace)          & 32.4 $\pm$ 3.1      & -3.28 $\pm$ 0.09 \\
\bottomrule
\end{tabular}
\vspace{-1em}
\end{table}

\begin{table}[t]
\centering
\scriptsize
\setlength{\tabcolsep}{6.0pt}
\renewcommand{\arraystretch}{1.15}
\caption{\small \textbf{Indicative test-time cost of posterior-action inference.}
We report the approximate wall-clock time required to infer actions and evaluate the downstream loss on a batch of $2000$ trajectories. Timings are measured on a separate evaluation machine and are intended only to indicate the approximate order of magnitude of the deployment cost.}
\label{tab:posterior_action_runtime}
\begin{tabular}{lc}
\toprule
\textbf{Action Estimator} & \textbf{Approx. time (seconds)} $\downarrow$ \\
\midrule
Amortised downstream network & $\boldsymbol{\approx 6}$ \\
Importance-sampling posterior mean & $\approx 7600$ \\
Laplace posterior approximation & $\approx 250$ \\
HMC posterior approximation & $\approx 3100$ \\
\bottomrule
\end{tabular}
\vspace{-1em}
\end{table}

\paragraph{Indicative inference cost.}
Since posterior-action estimators require test-time inference for each observed trajectory, we additionally report, in Table~\ref{tab:posterior_action_runtime}, the wall-clock time required to infer actions and evaluate the downstream loss on a batch of $2000$ trajectories. These timings were obtained on the evaluation machine used for this posterior-inference study, which differs from the hardware used for training. They should therefore be interpreted as indicative deployment costs rather than hardware-normalised runtime benchmarks. Nevertheless, they highlight the qualitative distinction between amortised action prediction, which only requires a forward pass through the downstream network, and posterior-inference baselines, which require iterative or sampling-based computation at test time.

\paragraph{Loss-dependent acquisition behaviour.}
The results reported in Table~\ref{tab:deployment_fixed_policy_light} highlight an important distinction between the two downstream losses considered in this work. Although both losses evaluate the quality of a final action inferred from the experimental history, they encode different downstream objectives. Consequently, the corresponding Bayes actions are not the same, and neither are the design policies learned by \textsc{Action-BED}. This distinction is crucial when interpreting the posterior-action evaluations of fixed design policies. The experimental trajectories produced by the policy are shaped by the final decision criterion they are meant to serve. \textsc{Action-BED} trained with the QLL, as shown in Table~\ref{tab:location_table}, achieves strong QLL performance, and its learned designs often concentrate tightly around the latent sources. One might therefore expect that, under these highly localised designs, the MSE Bayes action obtained by posterior inference would also be especially accurate, perhaps even outperforming the policies trained with other objectives. However, this is not what we observe in practice. 

As illustrated in Figure~\ref{fig:sorted_rollout_losses}, and quantified in Table~\ref{tab:location_loss_quantiles}, the distribution of rollout-wise MSE losses differs substantially between \textsc{Action-BED (MSE)} and \textsc{Action-BED (QLL)}. Under the same evaluation conditions, \textsc{Action-BED (QLL)} achieves markedly lower losses for a large fraction of rollouts with a median of
$0.008 \times 10^{-2}$ compared to $0.27 \times 10^{-2}$ for the MSE-trained policy. However, this improvement comes at the cost of a much heavier upper tail: at the 99th percentile, its MSE reaches $95.9 \times 10^{-2}$, compared to only $5.8 \times 10^{-2}$ under MSE training. Thus, although the QLL-trained policy produces very accurate predictions on the majority of trajectories, it also induces rare rollouts on which the downstream action can fail severely. These high-loss trajectories dominate the average MSE and explain why the overall MSE performance of the QLL-trained policy is worse despite its stronger median behaviour.

This behaviour is consistent with the different training signals induced by the two losses. The logarithmic component of QLL strongly rewards reductions in already small prediction errors, encouraging precise localisation whenever achievable. In contrast, the quadratic objective places greater emphasis on large residuals, encouraging robustness against severe failures. This trade-off also helps interpret the higher EIG bounds achieved by \textsc{Action-BED} (QLL) in Figure~\ref{fig:eig_bounds_time}: its objective appears to encourage acquisition strategies that support highly accurate source identification, even though EIG is not directly optimised. Importantly, the two variants optimise different downstream objectives, rather than competing to minimise the same loss. The posterior-based evaluations in Table~\ref{tab:deployment_fixed_policy_light} further show that the QLL-trained acquisition policy remains less effective under the mean MSE criterion, even when using posterior-based MSE Bayes-action estimators instead of its amortised action network.

\begin{figure}[t]
    \centering
    \begin{subfigure}[t]{0.49\linewidth}
        \centering
        \includegraphics[width=\linewidth]{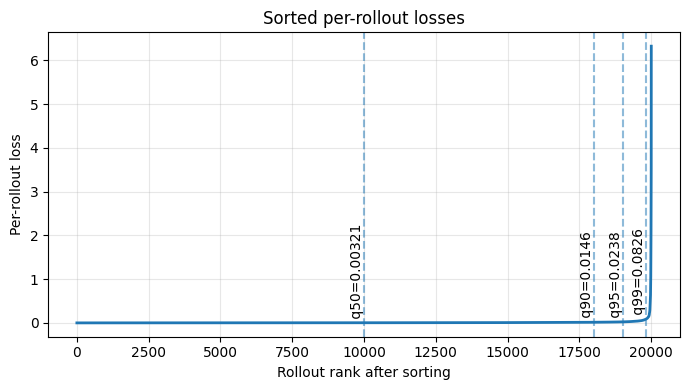}
        \caption{\small \textsc{Action-BED (MSE)}}
    \end{subfigure}
    \hfill
    \begin{subfigure}[t]{0.49\linewidth}
        \centering
        \includegraphics[width=\linewidth]{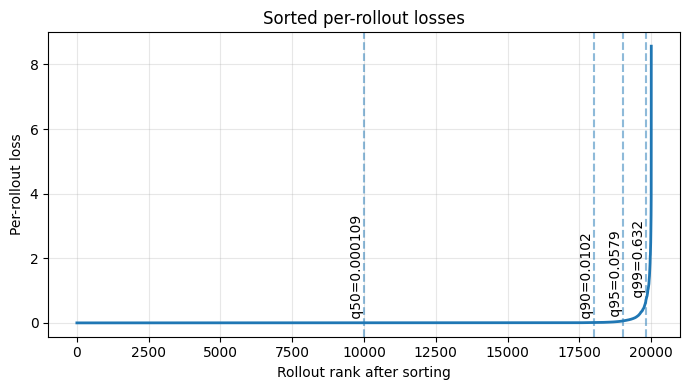}
        \caption{\small \textsc{Action-BED (QLL)}}
    \end{subfigure}
    \caption{ \small \textbf{Distribution of rollout-wise MSE losses for loss-specialised \textsc{Action-BED} policies.} We report per-rollout downstream MSE losses sorted from smallest to largest for \textsc{Action-BED} trained with the MSE objective and with the QLL objective over $20000$ rollouts. The QLL-trained policy achieves smaller errors on most rollouts but exhibits a substantially heavier upper tail, as further quantified in Table~\ref{tab:location_loss_quantiles}.}
    \label{fig:sorted_rollout_losses}
\end{figure}

\begin{table}[t]
    \centering
    \caption{
    \textbf{Source Location Finding: Quantiles of per-rollout
    downstream losses.}
    For each trained policy, we evaluate 4096 rollouts
    and compute the loss quantiles separately for each
    evaluation metric. 
    }
    \label{tab:location_loss_quantiles}
    \small
    \begin{tabular}{llcccc}
        \toprule
        Method & Evaluation loss
        & $q_{50}$ & $q_{90}$ & $q_{95}$ & $q_{99}$ \\
        \midrule
        \textsc{Action-BED} (MSE)
        & MSE ($10^{-2}$)
        & 0.27 & 1.1 & \textbf{1.9} & \textbf{5.8} \\
        & QLL
        & -6.83 & -5.37 & -4.91 & -3.96 \\
        \midrule
        \textsc{Action-BED} (QLL)
        & MSE ($10^{-2}$)
        & \textbf{0.008} & \textbf{0.72} & 4.2 & 95.9 \\
        & QLL
        & \textbf{-10.43} & \textbf{-7.43}
        & \textbf{-6.23} & \textbf{-4.12} \\
        \bottomrule
    \end{tabular}
\end{table}

\subsubsection{Additional Results}
\label{appendix:additional_results_lf}

\paragraph{Training time decomposition.} All training-time comparisons are performed on the same GPU, so that reported differences reflect the computational cost of the methods rather than hardware variability. In addition to the main performance metrics, we report in Table \ref{tab:compute_time} a decomposition of the training time into policy and downstream action-network training. This decomposition shows that DAD has a slightly larger total training time than Action-BED, although its policy network alone is faster to train than the jointly trained Action-BED model. This is likely due to the batch sizes used in this experiment and to the relatively large neural architectures, for which backpropagation represents a substantial fraction of the total cost. Consequently, the computational advantage of the singly intractable objective is partly diluted by the cost of the network updates. ALINE is markedly more expensive: policy training alone already exceeds the full joint-training time of Action-BED, and its downstream stage remains computationally demanding because adaptive rollouts are costly to generate. Due to its simplicity, the random-design baseline is comparatively inexpensive to train.

\paragraph{Evolution of downstream performance.}
Figure~\ref{fig:losses_against_time} reports an additional time-resolved comparison of downstream performance during policy training. For each method, we periodically freeze the design-policy weights and evaluate the downstream losses using the same evaluation protocol as above. For DAD and ALINE, this requires training a downstream action network from scratch at each frozen policy checkpoint, using the same architecture and training hyperparameters as reported in Tables~\ref{tab:downstream_net_architecture} and~\ref{tab:lfmlphparams_baselines}. For Action-BED, we instead evaluate the downstream action network associated with the joint model at the corresponding checkpoint. Importantly, the horizontal axis only accounts for the elapsed policy-training time, or joint-training time for Action-BED, including the potential warm-up time. Thus, it excludes the additional cost of training independent downstream action networks for the baselines.  Checkpoints are shown up to convergence. For the joint Action-BED models, both downstream losses are evaluated at every checkpoint, even when the downstream predictor was not explicitly optimised for the loss under consideration.

This experiment illustrates how quickly each design policy produces histories that are useful for downstream prediction. Under the permutation-invariant MSE, Action-BED rapidly reaches a regime in which its action network performs well, achieving the best performance during approximately the first six hours of training. At later checkpoints, however, downstream networks trained on ALINE policies achieve stronger performance. Under the QLL, Action-BED trained with the QLL objective yields substantial gains very early in training, outperforming the other methods by a large margin.

\begin{figure}[t]
    \centering
    \includegraphics[width=0.85\linewidth]{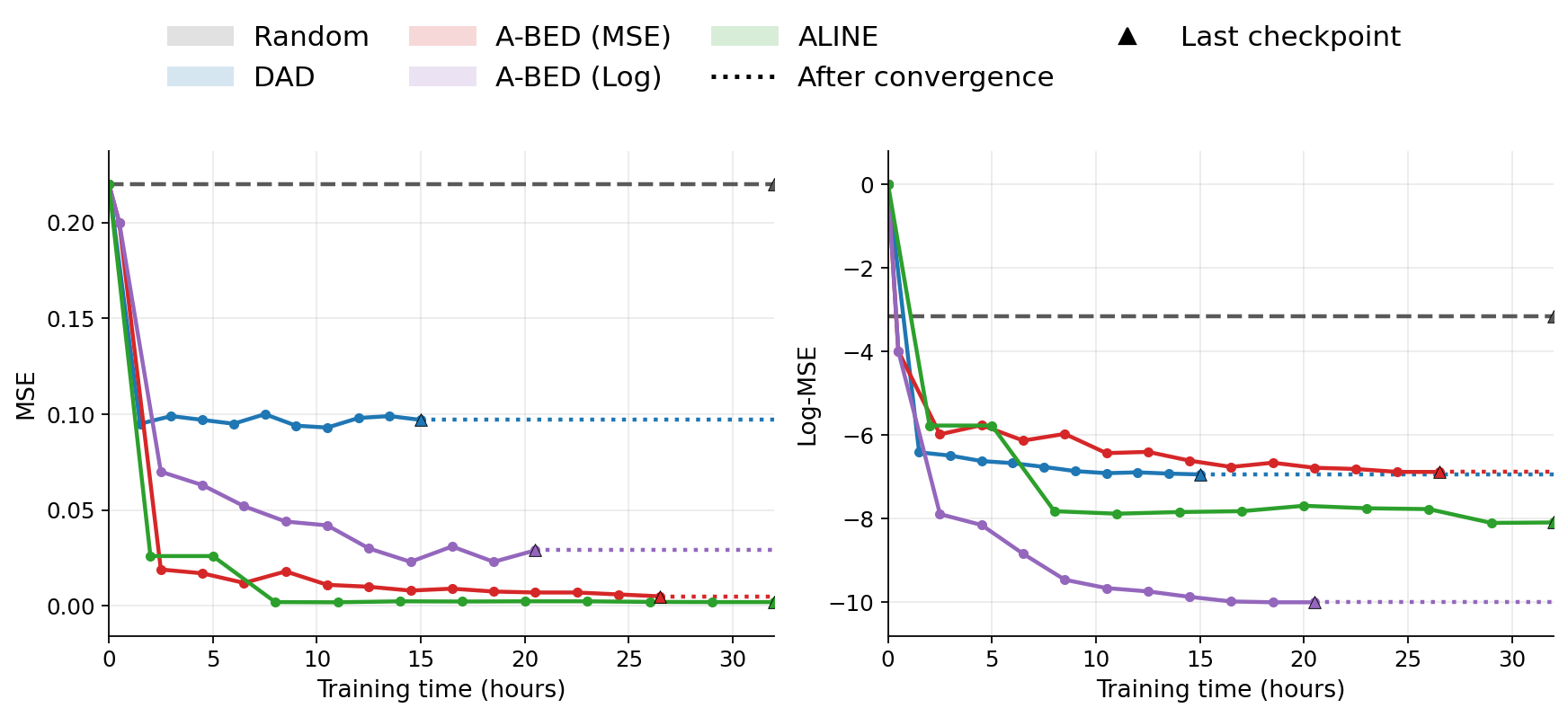}
    \caption{\textbf{Comparison of downstream losses} as a function of design policy training time, until convergence.}
    \label{fig:losses_against_time}
\end{figure}

\begin{table}[t]
\centering
\scriptsize
\setlength{\tabcolsep}{5.0pt}
\renewcommand{\arraystretch}{1.15}

\caption{
\small \textbf{Training Time Decomposition}.
Wall-clock GPU times are reported in minutes.
For \textsc{Action-BED}, design and downstream action networks are trained jointly, so only total training time is reported.
}
\label{tab:compute_time}

\begin{tabular}{lccc}
\toprule
\textbf{Method}
& \textbf{Design}
& \textbf{Downstream}
& \textbf{Total} \\
\midrule

\textsc{Action-BED} (MSE)
& --
& --
& 1171 \\

\textsc{Action-BED} (QLL)
& --
& --
& 1184 \\

\addlinespace[1pt]
\midrule

DAD (MSE)
& 879
& 335
& 1214 \\

DAD (QLL)
& 879
& 337
& 1216 \\

\addlinespace[1pt]
\midrule

ALINE (MSE)
& 1780
& 1182
& 2962 \\

ALINE (QLL)
& 1780
& 1190
& 2970 \\

\addlinespace[1pt]
\midrule

Random
& 0
& 31
& 31 \\

\bottomrule
\end{tabular}
\vspace{-1em}
\end{table}

\paragraph{Effect of Joint Training on Policy Co-Adaptation.}
We further examine whether joint training induces a useful co-adaptation between the design policy network and the downstream action policy network. To isolate this effect, we take a fully trained \textsc{Action-BED} design policy, keep it fixed, and retrain a new downstream action network under the same optimisation budget and hyperparameters as in the original joint training procedure; see Tables~\ref{tab:downstream_net_architecture} and~\ref{tab:lfmlphparams_baselines}. We perform this retraining separately with the MSE and QLL objectives.

The relevant comparison is the same-loss setting: for each \textsc{Action-BED} model, we compare the native jointly trained action network to a post-hoc action network retrained with the same loss used to learn the design policy. As shown in Table~\ref{tab:native_vs_retrained_abed}, this retrained network performs slightly worse than the native joint network, despite being trained on histories generated by an already fully optimised design policy. This suggests that the gain of joint training is due to the optimisation path through which the design and action networks become mutually adapted. In particular, the design policy learns to generate histories that are informative in a form the action network can exploit, while the action network simultaneously learns to decode both the observed outcomes and the information implicit in the adaptive sequence of designs.

\begin{table}[H]
\centering
\small
\setlength{\tabcolsep}{5.2pt}
\renewcommand{\arraystretch}{1.15}
\caption{\small \textbf{Native versus retrained downstream networks for Action-BED.} 
Native networks are obtained directly from joint Action-BED training.
Retrained networks are trained post-hoc after freezing the learned acquisition policy, using the same downstream loss.
Uncertainties are reported as one standard error.}
\label{tab:native_vs_retrained_abed}
\begin{tabular}{llcccc}
\toprule
\textbf{Acquisition} & \textbf{Action Network}
& \multicolumn{3}{c}{\textbf{Performance} ($\pm$ 1 s.e.)} \\
\cmidrule(lr){3-4}
\cmidrule(lr){5-5}
& & MSE $(10^{-2})$ $\downarrow$
& QLL $\downarrow$ \\
\midrule

\textsc{Action-BED} (MSE)
& Native network    & $\mathbf{0.55 \pm 0.06}$ & -6.83 $\pm$ 0.03\\
& Retrained network       & 0.86 $\pm$ 0.09 & -6.93 $\pm$ 0.01 \\
\addlinespace[1pt]
\midrule
\textsc{Action-BED} (QLL)
& Native network    & $2.8 \pm 0.4$ & $\mathbf{-10.10 \pm 0.03}$ \\
& Retrained network       & 1.7 $\pm$ 0.2 & -9.97 $\pm$ 0.03  \\

\bottomrule
\end{tabular}
\vspace{-1em}
\end{table}

\paragraph{Cross-Evaluation of Design and Action Policies.}
We further study the coupling between acquisition and decision-making by cross-evaluating all trained downstream action networks across all design policies in the benchmark. For each trained design policy, we generate histories under that policy and evaluate the predictions produced by action networks trained with every other method with the MSE loss. For \textsc{Action-BED}, the action network is the one learned jointly with the design policy; for the other baselines, it is the network trained subsequently on trajectories from the fixed design policy. Table~\ref{tab:cross_policy_coupling_normalized} reports normalised performance ratios, measuring, for each fixed design policy, the degradation incurred when replacing its associated action network by an action network trained on another history distribution.

This experiment tests whether action policies learn transferable Bayes-optimal action, or instead become specialised to the distribution of histories induced by their own design policy. Design policies that are more accurate and concentrate their measurements tightly around the latent sources induce history distributions on which action networks trained under other acquisition strategies perform substantially worse. In this regime, the downstream predictor must exploit fine-grained, policy-specific patterns in the adaptive design sequence, rather than only aggregate generic information from the observations. Conversely, the action networks paired with these more focused design policies achieve strong performance on their own histories but transfer poorly to histories generated by other policies. This suggests that high-performing methods do not merely collect more information; they also shape the information in a form that is specifically matched to their associated downstream decision rule.

\begin{table}[H]
\centering
\small
\setlength{\tabcolsep}{5.2pt}
\renewcommand{\arraystretch}{1.15}
\caption{\small \textbf{Normalized cross-policy coupling between acquisition and action networks for the MSE loss.}
Each entry reports the normalized degradation factor for a test acquisition policy $\pi_d^i$ paired with an action network $\pi_a^j$, defined as
$\widetilde R_{ij} = R(\pi_d^i,\pi_a^j) / \min_k R(\pi_d^i,\pi_a^k)$.
Rows indicate the acquisition policies used at test time, and columns the downstream action networks.}
\label{tab:cross_policy_coupling_normalized}
\begin{tabular}{lccccc}
\toprule
\textbf{Test Acquisition}
& \multicolumn{5}{c}{\textbf{Downstream Action Network}} \\
\cmidrule(lr){2-6}
& A-BED (MSE)
& A-BED (QLL)
& DAD
& ALINE
& Random \\
\midrule

\textsc{Action-BED} (MSE)
& $\mathbf{1.00}$ & 304.55 & 114.73 & 175.64 & 13.64 \\

\textsc{Action-BED} (QLL)
& 17.11 & $\mathbf{1.00}$ & 16.07 & 34.75 & 9.21 \\

DAD
& 19.48 & 1079.98 & $\mathbf{1.00}$ & 16.89 & 6.14 \\

ALINE
& 32.86 & 18.10 & 20.03 & $\mathbf{1.00}$ & 131.91 \\

Random
& 11.83 & 1135.97 & 4.16 & 10.82 & $\mathbf{1.00}$ \\

\bottomrule
\end{tabular}
\vspace{-1em}
\end{table}

\paragraph{Impact of Warmstarting in Joint Training.}
We assess whether warmstarting the downstream action network improves joint \textsc{Action-BED} training. Table~\ref{tab:warmstart_abed} shows that, for both the MSE and QLL objectives, warmstarting leads to better final performance. This suggests that the benefit is not limited to a faster initial optimisation of the action network. Rather, by providing a more accurate downstream predictor early in training, warmstarting yields a more informative learning signal for the design policy. This improved signal changes the subsequent joint optimisation dynamics and guides the coupled design--action system toward a better final solution.

\begin{table}[H]
\centering
\scriptsize
\setlength{\tabcolsep}{5.0pt}
\renewcommand{\arraystretch}{1.15}
\begin{threeparttable}
\caption{\small \textbf{Impact of warmstarting in joint Action-BED training.} Uncertainties are reported as one standard error.}
\label{tab:warmstart_abed}
\begin{tabular}{llccccc}
\toprule
\textbf{Acquisition} & \textbf{Warmstart}
& \multicolumn{4}{c}{\textbf{Performance} ($\pm$ 1 s.e.)}
& \textbf{Compute Time} \\
\cmidrule(lr){3-6}
\cmidrule(lr){7-7}
& & MSE $(10^{-2})$ $\downarrow$
& QLL $\downarrow$
& sPCE $\uparrow$
& sNMC $\uparrow$
& GPU (min) $\downarrow$ \\
\midrule

\textsc{Action-BED} (MSE)
& Yes      & $\mathbf{0.55 \pm 0.06}$ & $-6.83 \pm 0.03$ & 11.63 $\pm$ 0.04  & 12.84 $\pm$ 0.08  & 1171 \\
& No    & 0.87 $\pm$ 0.1 & -6.82 $\pm$ 0.01 & 11.62 $\pm$ 0.08 & 12.67 $\pm$ 0.15 & \textbf{1159} \\
\addlinespace[1pt]
\midrule

\textsc{Action-BED} (QLL)
& Yes      & $2.8 \pm 0.4$ & $\mathbf{-10.10 \pm 0.03}$ & $\mathbf{17.14 \pm 0.11}$ & $\mathbf{21.07 \pm 0.45}$ & 1184 \\
& No    & 3.0 $\pm$ 0.3 & -9.33 $\pm$ 0.03 & 16.48 $\pm$ 0.18 & 18.88 $\pm$ 0.49 & 1172 \\

\bottomrule
\end{tabular}
\end{threeparttable}
\vspace{-1em}
\end{table}

\paragraph{Training-seed robustness of the main architectures.}
To further assess the robustness of our results to training initialization, we repeat the training of \textsc{Action-BED} and \textsc{DAD} over five independent random seeds using the main architectures on two benchmarks: source location finding and the stochastic single pendulum. The results are reported in Tables~\ref{tab:source_main_seeds} and \ref{tab:pendulum_main_seeds}, respectively, as mean $\pm$ standard error across training seeds, rather than across evaluation rollouts as in the main results.

\begin{table}[t]
    \centering
    \caption{
    \textbf{Source Location Finding: Training-seed variability.}
    Mean $\pm$ standard error over five independent training seeds using the main architectures. 
    \textsc{Action-BED} is evaluated with its native action network on the relevant loss, whereas for \textsc{DAD} an action network is retrained post-hoc on the relevant loss and evaluated under the same loss. 
    }
    \label{tab:source_main_seeds}
    \small
    \begin{tabular}{lcccc}
        \toprule
        Method
        & MSE ($10^{-2}$) $\downarrow$
        & QLL $\downarrow$
        & sPCE $\uparrow$
        & sNMC $\uparrow$ \\
        \midrule
        \textsc{Action-BED} (MSE)
        & $\mathbf{0.63 \pm 0.05}$
        & $-6.91 \pm 0.07$
        & $11.72 \pm 0.12$
        & $12.73 \pm 0.22$ \\

        \textsc{Action-BED} (QLL)
        & $3.57 \pm 0.12$
        & $\mathbf{-10.11 \pm 0.08}$
        & $\mathbf{17.49 \pm 0.23}$
        & $\mathbf{20.13 \pm 0.62}$ \\

        DAD
        & $5.12 \pm 0.48$
        & $-6.93 \pm 0.09$
        & $13.01 \pm 0.09$
        & $17.13 \pm 0.14$ \\
        \bottomrule
    \end{tabular}
\end{table}

\begin{table}[t]
    \centering
    \caption{
    \textbf{Stochastic Single Pendulum: Training-seed variability.}
    Mean $\pm$ standard error over five independent training seeds using the main architectures, and with the same protocol as in Table~\ref{tab:source_main_seeds}.
    }
    \label{tab:pendulum_main_seeds}
    \footnotesize
    \setlength{\tabcolsep}{5pt}
    \begin{tabular}{lccccc}
        \toprule
        Method
        & MSE ($10^{-3}$) $\downarrow$
        & W. MSE ($10^{-3}$) $\downarrow$
        & QLL $\downarrow$
        & sPCE $\uparrow$
        & sNMC $\uparrow$ \\
        \midrule
        \textsc{Action-BED} (MSE)
        & $\mathbf{6.2 \pm 0.2}$
        & $4.8 \pm 0.1$
        & $-6.57 \pm 0.01$
        & $3.66 \pm 0.01$
        & $3.69 \pm 0.03$ \\

        \textsc{Action-BED} (QLL)
        & $7.2 \pm 0.1$
        & $5.3 \pm 0.4$
        & $\mathbf{-6.73 \pm 0.02}$
        & $\mathbf{3.96 \pm 0.02}$
        & $\mathbf{3.97 \pm 0.04}$ \\

        \textsc{Action-BED} (W. MSE)
        & $13.3 \pm 1.7$
        & $\mathbf{3.6 \pm 0.4}$
        & $-5.85 \pm 0.09$
        & $3.16 \pm 0.07$
        & $3.19 \pm 0.05$ \\

        DAD
        & $7.7 \pm 0.4$
        & $5.3 \pm 0.5$
        & $-6.37 \pm 0.02$
        & $3.89 \pm 0.02$
        & $3.90 \pm 0.02$ \\
        \bottomrule
    \end{tabular}
\end{table}

Across both tasks, the results exhibit relatively limited variability across training seeds and preserve the main downstream-performance trends. In particular, each \textsc{Action-BED} variant achieves the lowest average loss under its corresponding training objective, while the QLL-trained policies consistently attain strong EIG bounds. These findings provide additional evidence that the performance gains observed in the main experiments are not specific to a single training initialization.

Due to the substantial computational cost of training the full architectures, this analysis is restricted to \textsc{Action-BED} and \textsc{DAD} on two representative benchmarks. In particular, we do not repeat the training of \textsc{ALINE} or the remaining tasks across multiple seeds. Nevertheless, these additional experiments complement the training-seed stability analysis conducted with smaller architectures in Appendix~\ref{appendix:light_archi}, and provide further evidence of the robustness of \textsc{Action-BED}.

\subsection{Dynamical Systems}

We consider sequential design problems in which the underlying simulator is a dynamical system. Unlike exchangeable observation models, each outcome depends on the current system state, so the order in which observations are collected matters. This setting is adapted from \cite{iqbal_nesting_2024} and covers both the simple- and double-pendulum experiments considered in this work.

More precisely, we view these experiments as an abstraction of dynamical systems governed by stochastic differential equations, detailed below, whose observation model is Markovian:
\[
    p(y_{t+1} \mid y_t, \xi_{t+1}, \theta),
\]
so that the next observation is generated from the previous state or observation~$y_t$, the current design~$\xi_{t+1}$, and the physical parameters~$\theta$. Hence, the likelihood is sequential and cannot be factorized as with conditionally exchangeable observations. Under these assumptions, the joint distribution of a trajectory can be written as
\begin{align*}
    p(h_T \mid \theta; \pi_d^\phi)
    &= \prod_{t=1}^T p( y_t \mid h_{t-1}, \theta) \\
    &= \prod_{t=1}^T
       p(y_t \mid y_{t-1}, \xi_t, \theta).
\end{align*}

\paragraph{Architectures.}
For both pendulum tasks, we use the same TNP-like policy architecture for \textsc{Action-BED} and \textsc{DAD}, as described in Appendix~\ref{appendix:design_policy_architecture} and with the same encoding dimensions as in Table \ref{tab:architecture-dimensions}. Since the likelihood is non-exchangeable, we add a time embedding to each history element $r_t$ before aggregation. At the beginning of a rollout, when no design--outcome pair has yet been observed, the initial history representation is parameterized as a learnable vector.

For both \textsc{Action-BED} and \textsc{DAD}, the emitter output is additionally mapped through an affine--$\tanh$ transformation,
\[
    \xi_t = a \tanh(\tilde \xi_t) + b,
\]
which ensures that the selected designs lie in the admissible control domain. The same architectural template is used for the stochastic and double pendulum experiments; only the input and output dimensions can vary, depending on the dimension of the designs and observations. 

\textsc{ALINE} is permutation-invariant by construction and is therefore not directly tailored to non-exchangeable likelihoods. To make the sequential structure explicit, we augment each context element with a time encoding and include the most recent observation in the context before encoding. This provides the model with the information needed to condition each candidate design on the current dynamical state. \textsc{Random} samples designs uniformly at random over the valid constrained space. 

In all cases, the downstream predictor is a fully connected MLP mapping the complete design--outcome history to a point estimate of~$\theta$ in the relevant space (see Table \ref{tab:downstream_net_architecture_pendulum}). No parameter sharing is used between the design policy and the downstream network.

\begin{table}[!htbp]
\centering
\caption{Downstream network architecture.}
\begin{tabular}{llll}
\toprule
\textbf{Layer} & \textbf{Description} & \textbf{Dimension} & \textbf{Activation} \\
\midrule
Input & $h_T$ & $T \times (\dim_\xi + \dim_y)$ & -- \\
Hidden layer 1 & Fully connected & 512 & GELU \\
Hidden layer 2 & Fully connected & 256 & GELU \\
Hidden layer 3 & Fully connected & 128 & GELU \\
Output & $\hat{\theta}$ & $d_\theta$ & -- \\
\bottomrule
\end{tabular}
\label{tab:downstream_net_architecture_pendulum}
\end{table}

\paragraph{Training procedure.}
For the dynamical-system experiments, we follow a training protocol closely mirroring that of the source location finding task. All policies are trained from trajectories generated by the corresponding simulator. At each iteration, latent parameters $\theta$ are sampled from the prior, the design policy is rolled out autoregressively over the experimental horizon, and each method is trained on its own objective. We compare several downstream objectives in order to assess the accuracy and flexibility of the learned predictive action network. We consider the standard mean-squared error on~$\theta$,
\[
    \ell_{\mathrm{MSE}}(\hat\theta,\theta)
    =
    \frac{1}{d_\theta}\|\hat\theta-\theta\|_2^2,
\]
the logarithmic variant,
\[
    \ell_{QLL}(\hat\theta,\theta)
    = \frac{1}{d_\theta} \sum_{j=1}^{d_\theta}
    \left[\log\left[(\hat\theta_j-\theta_j)^2 + \epsilon \right] + (\hat\theta_j-\theta_j)^2 \right]
    ,
\]
and a weighted MSE,
\[
    \ell_{\mathrm{wMSE}}(\hat\theta,\theta)
    =
    \frac{1}{d_\theta} \sum_{j=1}^{d_\theta}
    w_j(\hat\theta_j-\theta_j)^2.
\]
For the stochastic pendulum, we use weights
\[
    w=(0.1,1.0,2.0),
\]
whereas for the double pendulum we use
\[
    w=(1.0,1.0,2.0,3.0).
\]
The weighted objective allows different components of~$\theta$ to be emphasized, illustrating that the downstream target can be adapted to the scientific or operational quantity of interest.

\paragraph{Evaluation procedure.}
In addition to downstream prediction performance, we evaluate each trained policy using sequential PCE and sequential NMC bounds on the EIG, which provide a common information-theoretic measure of design quality across methods. For both dynamical systems and for all downstream losses, prediction metrics are estimated from $2048$ rollout batches and reported with $\pm 1$ standard error. For the single-pendulum experiment, the sequential PCE and NMC estimates are computed with $L=50$K contrastive samples and a batch size of $2000$. For the double-pendulum experiment, we use a larger contrastive budget, with $L=200$K and a batch size of $1000$. We also report the total GPU compute time required for training.

\subsubsection{Stochastic Pendulum}
\label{appendix:stopendulum}

The first dynamical system considered is a stochastic single pendulum. The simulator describes the trajectory of a pendulum whose state evolves over time according to noisy dynamics. At each step, the design variable corresponds to an external control applied to the system, allowing the policy to influence the subsequent motion of the pendulum and thereby collect informative observations about the latent physical parameters.

\paragraph{Task.}
We consider the problem of actively identifying the physical parameters of a stochastic controlled pendulum. The state is $x_t = (q_t,\dot q_t)$, where $q_t$ denotes the angle from the vertical and $\dot q_t$ the angular velocity. The parameters of interest are the mass and length of the pendulum, $(m,l)$, while the gravitational acceleration and damping coefficient are fixed to $g = 9.81, \ d = 0.1$. To obtain dynamics that depend linearly on the unknown parameters, we introduce the reparametrization $\theta := \left(\frac{3g}{2l}, \frac{3d}{ml^2}, \frac{3}{ml^2}\right)^\top$.  The initial state is fixed as $x_0 = (0,0)^\top$. At each step of the discretized dynamics, with time step $\Delta t = 0.05$, the policy selects a bounded control input $\xi_t \in [-1,1]$, corresponding to a torque applied to the pendulum. The observation is the state of the system at the next discretization step. The objective is to choose controls that generate trajectories informative about the latent physical parameters.

\paragraph{Stochastic Dynamics.}
In continuous time, the controlled dynamics follow the stochastic differential equation
\[
    dx_t = (\dot q_t, h(x_t,\xi_t)^\top \theta)^\top\,dt + L\,d\beta_t,
\]
where $\beta_t$ is a Brownian motion,
\[
    h(x_t,\xi_t)=(-\sin q_t,\;-\dot q_t,\;\xi_t),
    \qquad
    L = (0,0.1)^\top.
\]
Equivalently,
\begin{align*}
    dq_t &= \dot q_t\,dt, \\
    d\dot q_t &= h(x_t,\xi_t)^\top \theta\,dt + 0.1\,d\beta_t.
\end{align*}
The observations used in the experiment are the discretized states
\[
    y_t = x_t = (q_t,\dot q_t),
\]
recorded every $\Delta t=0.05$. The resulting observation process is sequential: each new observation depends on the previous state, the applied control, and the latent parameter~$\theta$.

In our setting, we place a Gaussian prior directly on the effective coefficients of $\theta$, without enforcing the constraints associated with a fixed set of physical parameters: 
\[
    p(\theta)
    =
    \mathcal N\!\left(
    \begin{bmatrix}
        14.7 \\
        0 \\
        3.0
    \end{bmatrix},
    \begin{bmatrix}
        0.1 & 0 & 0 \\
        0 & 0.01 & 0 \\
        0 & 0 & 0.1
    \end{bmatrix}
    \right).
\]
The downstream task is therefore to estimate the latent dynamical coefficients $\theta$ from the controlled trajectory. The design variable $\xi_t$ represents the external torque applied to the pendulum.

\paragraph{Hyperparameters.}
The main training hyperparameters are reported in Table~\ref{tab:dynamical_systems_hyperparameters}. \textsc{Action-BED} and \textsc{DAD} are trained for $50$K gradient steps. For \textsc{Action-BED}, we additionally use a $10$K-step warm-up phase in which the downstream network is trained on trajectories generated from random designs. \textsc{ALINE} is trained with a total budget of $100$K steps, including a $20$K-step burning phase during which only the posterior inference network is optimised.

\begin{table}[t]
\centering
\scriptsize
\setlength{\tabcolsep}{3.0pt}
\renewcommand{\arraystretch}{1.12}
\caption{\small \textbf{Hyperparameters for stochastic single pendulum.}
Training and optimization hyperparameters used for the policy networks and the downstream prediction networks. 
ExpLR denotes an exponential learning-rate scheduler. The \textsc{Random} downstream network uses the same hyperparameters as \textsc{DAD}.}
\label{tab:dynamical_systems_hyperparameters}

\begin{threeparttable}
\begin{minipage}[t]{0.32\textwidth}
\centering
\textbf{DAD}\\[0.35em]
\begin{tabular}{lr}
\toprule
\textbf{Field} & \textbf{Value} \\
\midrule
\multicolumn{2}{l}{\textbf{Policy training}} \\
Batch size & $512$ \\
Contrastive samples $L$ & $2000$ \\
Policy steps & $50\mathrm{K}$ \\
Policy seed & $42$ \\
\midrule
\multicolumn{2}{l}{\textbf{Policy optimization}} \\
Optimizer & Adam \\
Adam betas & $(0.8,0.998)$ \\
Policy LR & $10^{-4}$ \\
Scheduler & ExpLR \\
Decay $\gamma$ & $0.96$ \\
Decay period & $400$ \\
Weight decay & $0$ \\
\midrule
\multicolumn{2}{l}{\textbf{Downstream training}} \\
Downstream steps & $50\mathrm{K}$ \\
Downstream seed & $42$ \\
Downstream LR & $10^{-4}$ \\
Downstream sched. & ExpLR \\
\bottomrule
\end{tabular}
\end{minipage}
\hfill
\begin{minipage}[t]{0.32\textwidth}
\centering
\textbf{ALINE}\\[0.35em]
\begin{tabular}{lr}
\toprule
\textbf{Field} & \textbf{Value} \\
\midrule
\multicolumn{2}{l}{\textbf{Policy training}} \\
Batch size & $200$ \\
Policy steps & $80\mathrm{K}$ \\
Burn-in steps & $20\mathrm{K}$ \\
Burn-in training & Posterior only \\
Query pool, burn-in & $50$ \\
Query pool, train & $200$ \\
Query proposal & Uniform \\
Reward discount & $1$ \\
Policy weight $\alpha$ & $1$ \\
Policy seed & $123$ \\
\midrule
\multicolumn{2}{l}{\textbf{Policy optimization}} \\
Optimizer & Adam \\
Adam betas & $(0.9,0.999)$ \\
Policy LR & $10^{-3}$ \\
Policy sched. & Cosine \\
Grad. clipping & $\|\nabla\|_{\infty}\leq 1$ \\
\midrule
\multicolumn{2}{l}{\textbf{Downstream training}} \\
Downstream steps & $50\mathrm{K}$ \\
Downstream LR & $10^{-3}$ \\
Downstream sched. & ExpLR \\
Downstream decay $\gamma$ & $0.95$ \\
Decay period & $1000$ \\
Downstream seed & 42 \\
\bottomrule
\end{tabular}
\end{minipage}
\hfill
\begin{minipage}[t]{0.32\textwidth}
\centering
\textbf{Action-BED}\\[0.35em]
\begin{tabular}{lr}
\toprule
\textbf{Field} & \textbf{Value} \\
\midrule
\multicolumn{2}{l}{\textbf{Joint training}} \\
Batch size & $512$ \\
Joint steps & $50\mathrm{K}$ \\
Warm-up steps & $10\mathrm{K}$ \\
Warm-up training & Downstream only \\
Warm-up designs & Random \\
Seed & $42$ \\
\midrule
\multicolumn{2}{l}{\textbf{Optimization}} \\
Optimizer & Adam \\
Adam betas & $(0.8,0.998)$ \\
Policy LR & $10^{-4}$ \\
Scheduler & ExpLR \\
Decay $\gamma$ & $0.96$ \\
Decay period & $400$ \\
Weight decay & $0$ \\
\bottomrule
\end{tabular}
\end{minipage}

\end{threeparttable}
\vspace{-1em}
\end{table}

\subsubsection{Stochastic Double Pendulum}
\label{appendix:doublependulum}

We next consider a controlled stochastic double-pendulum system. Compared with the single-pendulum experiment, this benchmark involves higher-dimensional, nonlinear, and coupled dynamics, making the generated trajectories substantially more sensitive to both the applied controls and the latent physical parameters. The state is
\[
    x_t =
    \begin{bmatrix}
        q_t \\
        \dot q_t
    \end{bmatrix}
    =
    (q_{1,t},q_{2,t},\dot q_{1,t},\dot q_{2,t})^\top
    \in \mathbb R^4,
\]
where $q_1,q_2$ are the joint angles and $\dot q_1,\dot q_2$ their angular velocities. The unknown parameter is
\[
    \theta = (m_1,m_2,l_1,l_2)^\top,
\]
corresponding to the masses and lengths of the two links. As in the single-pendulum setting, observations are the $T=50$ discretized states of the system, collected along a trajectory. The design variable is a pair of bounded torques
\[
    \xi_t = (\xi_{1,t},\xi_{2,t})^\top,
    \qquad
    \xi_{1,t}\in[-4,4],
    \qquad
    \xi_{2,t}\in[-2,2],
\]
applied to the two joints. The downstream task is to infer~$\theta$ from the resulting controlled trajectory. 

\paragraph{Stochastic Dynamics.}
The double-pendulum dynamics follow the standard manipulator form
\[
    M(q)\ddot q + C(q,\dot q)\dot q = \xi + \tau_g(q),
\]
where $M(q)$ is the inertia matrix, $C(q,\dot q)$ contains the Coriolis and centrifugal terms, and $\tau_g(q)$ is the gravitational torque. We use $g=9.81$ and add process noise to the acceleration dynamics. The resulting stochastic differential equation is
\[
    dq_t = \dot q_t\,dt,
    \qquad
    d\dot q_t
    =
    M(q_t)^{-1}
    \big(
        \tau_g(q_t) + \xi_t - C(q_t,\dot q_t)\dot q_t
    \big)\,dt
    + L\,d\beta_t,
\]
where $\beta_t=(\beta_{1,t},\beta_{2,t})^\top$ is a two-dimensional Brownian motion with independent components and
\[
    L =
    \begin{bmatrix}
        0.1 & 0 \\
        0 & 0.1
    \end{bmatrix}.
\]
The SDE is discretized with the same Euler--Maruyama scheme as in the single-pendulum experiment, with $\Delta t = 0.05$. The observation process is therefore Markovian and non-exchangeable: each new state depends on the previous state, the applied joint torques, and the latent physical parameters.

We place a log-normal prior on the physical parameters,
\[
    p(\theta)
    =
    \operatorname{LogNormal}\!\left(
    \begin{bmatrix}
        0 \\
        0 \\
        0 \\
        0
    \end{bmatrix},
    0.01\,I_4
    \right),
\]
so that masses and lengths remain positive. This task therefore tests whether adaptive design policies can select torque sequences that generate trajectories informative about coupled inertial, gravitational, and velocity-dependent effects.

\paragraph{Hyperparameters.}
The main training hyperparameters are reported in Table~\ref{tab:dynamical_systems_hyperparameters_double}. For this task, \textsc{Action-BED} is trained for $100$K gradient steps, preceded by a $20$K-step warm-up phase in which only the downstream action network is optimised. \textsc{ALINE} is trained under the same total budget of $100$K steps, including a $20$K-step burn-in phase for the posterior inference network. Finally, \textsc{DAD} and its associated downstream networks are trained for $50$K steps, as no further improvement was observed beyond this point. In all cases, downstream action networks are trained until their predictive performance saturates.

\begin{table}[t]
\centering
\scriptsize
\setlength{\tabcolsep}{3.0pt}
\renewcommand{\arraystretch}{1.12}
\caption{\small \textbf{Hyperparameters for stochastic double pendulum.}
ExpLR denotes an exponential learning-rate scheduler. The \textsc{Random} downstream network uses the same hyperparameters as \textsc{DAD}.}
\label{tab:dynamical_systems_hyperparameters_double}

\begin{threeparttable}
\begin{minipage}[t]{0.32\textwidth}
\centering
\textbf{DAD}\\[0.35em]
\begin{tabular}{lr}
\toprule
\textbf{Field} & \textbf{Value} \\
\midrule
\multicolumn{2}{l}{\textbf{Policy training}} \\
Batch size & $512$ \\
Contrastive samples $L$ & $2000$ \\
Policy steps & $50\mathrm{K}$ \\
Policy seed & $42$ \\
\midrule
\multicolumn{2}{l}{\textbf{Policy optimization}} \\
Optimizer & Adam \\
Adam betas & $(0.8,0.998)$ \\
Policy LR & $10^{-4}$ \\
Scheduler & ExpLR \\
Decay $\gamma$ & $0.96$ \\
Decay period & $400$ \\
Weight decay & $0$ \\
\midrule
\multicolumn{2}{l}{\textbf{Downstream training}} \\
Downstream steps & $50\mathrm{K}$ \\
Downstream seed & $42$ \\
Downstream LR & $10^{-4}$ \\
Downstream sched. & ExpLR \\
\bottomrule
\end{tabular}
\end{minipage}
\hfill
\begin{minipage}[t]{0.32\textwidth}
\centering
\textbf{ALINE}\\[0.35em]
\begin{tabular}{lr}
\toprule
\textbf{Field} & \textbf{Value} \\
\midrule
\multicolumn{2}{l}{\textbf{Policy training}} \\
Batch size & $200$ \\
Policy steps & $80\mathrm{K}$ \\
Burn-in steps & $20\mathrm{K}$ \\
Burn-in training & Posterior only \\
Query pool, burn-in & $50$ \\
Query pool, train & $150$ \\
Query proposal & Uniform \\
Reward discount & $1$ \\
Policy weight $\alpha$ & $1$ \\
Policy seed & $123$ \\
\midrule
\multicolumn{2}{l}{\textbf{Policy optimization}} \\
Optimizer & Adam \\
Adam betas & $(0.9,0.999)$ \\
Policy LR & $10^{-3}$ \\
Policy sched. & Cosine \\
Grad. clipping & $\|\nabla\|_{\infty}\leq 1$ \\
\midrule
\multicolumn{2}{l}{\textbf{Downstream training}} \\
Downstream steps & $100\mathrm{K}$ \\
Downstream LR & $10^{-3}$ \\
Downstream sched. & ExpLR \\
Downstream decay $\gamma$ & $0.95$ \\
Decay period & $1000$ \\
Downstream seed & 42 \\
\bottomrule
\end{tabular}
\end{minipage}
\hfill
\begin{minipage}[t]{0.32\textwidth}
\centering
\textbf{Action-BED}\\[0.35em]
\begin{tabular}{lr}
\toprule
\textbf{Field} & \textbf{Value} \\
\midrule
\multicolumn{2}{l}{\textbf{Joint training}} \\
Batch size & $512$ \\
Joint steps & $100\mathrm{K}$ \\
Warm-up steps & $20\mathrm{K}$ \\
Warm-up training & Downstream only \\
Warm-up designs & Random \\
Seed & $42$ \\
\midrule
\multicolumn{2}{l}{\textbf{Optimization}} \\
Optimizer & Adam \\
Adam betas & $(0.8,0.998)$ \\
Policy LR & $5\times 10^{-4}$ \\
Scheduler & ExpLR \\
Decay $\gamma$ & $0.96$ \\
Decay period & $400$ \\
Weight decay & $0$ \\
\bottomrule
\end{tabular}
\end{minipage}

\end{threeparttable}
\vspace{-1em}
\end{table}

\subsubsection{Additional Results}
\label{appendix:additional_results_pendulum}

We report additional compute-time results for both dynamical-system tasks. Tables~\ref{tab:compute_time_single_pendulum} and~\ref{tab:compute_time_double_pendulum} decompose the total wall-clock GPU time into design-policy training and downstream-network training whenever these stages are trained separately.

\begin{table}[t]
\centering
\scriptsize
\setlength{\tabcolsep}{5.0pt}
\renewcommand{\arraystretch}{1.15}
\caption{\small \textbf{Single Pendulum.} Training time decomposition.}
\label{tab:compute_time_single_pendulum}
\begin{tabular}{lccc}
\toprule
\textbf{Method}
& \textbf{Design}
& \textbf{Downstream}
& \textbf{Total} \\
\midrule
\textsc{Action-BED} (MSE)
& -- & -- & 240 \\
\textsc{Action-BED} (QLL)
& -- & -- & 287 \\
\textsc{Action-BED} (W-MSE)
& -- & -- & 290 \\
\addlinespace[1pt]
\midrule
\textsc{DAD} (MSE)
& 463 & 207 & 670 \\
\textsc{DAD} (QLL)
& 463 & 246 & 709 \\
\textsc{DAD} (W-MSE)
& 463 & 208 & 671 \\
\addlinespace[1pt]
\midrule
\textsc{ALINE} (MSE)
& 1520 & 600 & 2120 \\
\textsc{ALINE} (QLL)
& 1520 & 621 & 2141 \\
\textsc{ALINE} (W-MSE)
& 1520 & 602 & 2122 \\
\addlinespace[1pt]
\midrule
\textsc{Random}
& 0 & 19 & 19 \\
\bottomrule
\end{tabular}
\vspace{-1em}
\end{table}

\begin{table}[t]
\centering
\scriptsize
\setlength{\tabcolsep}{5.0pt}
\renewcommand{\arraystretch}{1.15}
\caption{\small \textbf{Double Pendulum: Training time decomposition.} 
Wall-clock GPU times are reported in minutes.
For \textsc{Action-BED}, the design and downstream action networks are trained jointly, so only total training time is reported.}
\label{tab:compute_time_double_pendulum}
\begin{tabular}{lccc}
\toprule
\textbf{Method}
& \textbf{Design}
& \textbf{Downstream}
& \textbf{Total} \\
\midrule
\textsc{Action-BED} (MSE)
& -- & -- & 501 \\
\textsc{Action-BED} (QLL)
& -- & -- & 504 \\
\textsc{Action-BED} (W-MSE)
& -- & -- & 504 \\
\addlinespace[1pt]
\midrule
\textsc{DAD} (MSE)
& 862 & 72 & 934 \\
\textsc{DAD} (QLL)
& 862 & 78 & 940 \\
\textsc{DAD} (W-MSE)
& 862 & 71 & 933 \\
\addlinespace[1pt]
\midrule
\textsc{ALINE} (MSE)
& 1938 & 1075 & 3013 \\
\textsc{ALINE} (QLL)
& 1938 & 1124 & 3062 \\
\textsc{ALINE} (W-MSE)
& 1938 & 1069 & 3007 \\
\addlinespace[1pt]
\midrule
\textsc{Random}
& 0 & 58 & 58 \\
\bottomrule
\end{tabular}
\vspace{-1em}
\end{table}




\subsection{Masked MNIST Classification Experiment}
\label{appendix:mnist_classif_exp}

We consider a sequential masked-classification task inspired by \citet{iollo_bayesian_2025} and based on the MNIST dataset \cite{lecun1998gradient}. The latent variable is an image $\theta \in \mathbb{R}^{28 \times 28}$, and the target of interest is its class label $z \in \{0,\dots,9\}$. The goal is to identify the digit class from a small number of partial and noisy observations of the image.

At each step, the design $\xi \in [1,28]^2$ specifies a continuous spatial location in pixel coordinates, interpreted as the top-left corner of a local patch. The observation is a noisy $5\times 5$ patch extracted around this location:
\[
    y = A_{\xi}\theta + \eta,
    \qquad
    \eta \sim \mathcal N(0,\sigma^2 I_d),
\]
where $d=5^2$, $A_{\xi}\theta$ denotes the extracted patch, and zero padding is used whenever the patch extends outside the image domain. This illustrates that the observations are reparametrizable, which allows gradients to be propagated through the observation process. Conditionally on the image, the observation model is therefore explicit:
\[
    p(y \mid \theta,\xi)
    =
    \mathcal N(A_{\xi}\theta,\sigma^2 I_d).
\]

A direct implementation of $A_\xi$ as a discrete masking operator would be non-differentiable with respect to $\xi$: small changes in the queried location would not affect the selected pixels until crossing an integer pixel boundary. To obtain a differentiable simulator, we instead use a continuous counterpart of the patch-extraction operator. In practice, the image is queried by grid sampling at the $5\times 5$ patch coordinates induced by continuous~$\xi$. When a sampling coordinate is not integer-valued, its intensity is obtained by bilinear interpolation from the four neighboring pixels. Thus, moving the design location smoothly changes the observed patch, making the simulator differentiable with respect to~$\xi$ while retaining the interpretation of localized pixel measurements.

Over a horizon $T=5$, the policy adaptively selects patch locations $\xi_{1:T}$ based on the previously observed patches, with the objective of producing a final history informative for classification. This means that the task has a hierarchical structure. Observations are generated from the latent image~$\theta$, but the target of interest is the discrete label~$z$. Consequently, the relevant class-conditional observation model is obtained by marginalizing over images within each class,
\[
    p(y \mid z,\xi)
    =
    \int_\Theta p(y \mid \theta,\xi)\,p(\theta \mid z)\,d\theta.
\]
Although $p(y \mid \theta,\xi)$ is available in closed form, this class-conditional likelihood is not. In practice, the joint prior over images and labels is only available through the empirical dataset,
\[
    p(\theta,z)
    =
    \frac{1}{N}\sum_{i=1}^N
    \delta_{(\theta_i,z_i)}(\theta,z).
\]
The experiment therefore provides a benchmark for sequential design in a high-dimensional, hierarchical, and partially implicit setting. The design problem is challenging because only a small number of localized measurements are allowed, while the latent image is high-dimensional. Informative policies must therefore learn where to observe the image so as to reveal discriminative regions for the downstream classifier. This directly tests whether an acquisition strategy can align its designs with the final prediction task, rather than merely collecting information about the full latent image.

\subsubsection{Baselines, architectures, and downstream objective.}

\paragraph{Protocol.} We follow the same experimental protocol as in the previous experiments, with the main difference that the label-level observation model is implicit. As a result, EIG bounds are not reported for this task. Acquisition policies are trained using only images from the training set, while final performance is evaluated and reported on both the training and test sets. All downstream classifiers share the architecture reported in Table~\ref{tab:downstream_net_architecture_mnist} and are trained with the cross-entropy loss
\[
    \mathcal L_{\mathrm{CE}}(\phi, \psi)
    =
    - \mathbb E_{(z,\theta),h_T}
    \left[
        \log \pi_a^\psi(z \mid h_T)
    \right],
\]
where $z$ denotes the class label and $h_T$ the final acquisition history. For each method, training is monitored on a validation split and stopped by early stopping once validation performance no longer improves. The resulting training budgets and selected hyperparameters are summarised in Tables~\ref{tab:hyperparameters_mnist_1} and~\ref{tab:hyperparameters_mnist_2}. To avoid redundancy, these tables report only the hyperparameters that differ from those in Table~\ref{tab:lfmlphparams_baselines}; all omitted entries are kept unchanged. We report the final classification loss and accuracy on both the training and test sets.

\begin{table}[H]
\centering
\scriptsize
\setlength{\tabcolsep}{5.0pt}
\renewcommand{\arraystretch}{1.15}

\caption{\small
\textbf{MNIST masked classification downstream architecture}.
The downstream classifier uses the same encoder--aggregation structure as the acquisition network, with a larger pair encoder to handle the higher-dimensional masked-image observations. The output dimension $K=10$ is the number of classes.
}
\label{tab:downstream_net_architecture_mnist}

\begin{tabular}{llll}
\toprule
\textbf{Module} & \textbf{Layer} & \textbf{Description} & \textbf{Dimension / Activation} \\
\midrule

\multirow{2}{*}{Pair encoder}
& Input
& $(\xi_t,y_t)$
& $\dim(\xi)+\dim(y)$ / -- \\
& H1
& Fully connected
& $512$ / ReLU \\
& H2
& Fully connected
& $256$ / ReLU \\
& H3
& Fully connected
& $128$ / ReLU \\
& Output
& Fully connected
& $32$ / -- \\

\addlinespace[1pt]
\midrule

\multirow{2}{*}{History aggregation}
& Input
& $\{R(\xi_s,y_s)\}_{s=1}^{T}$
& $T \times 32$ / -- \\
& Pooling
& Sum or mean pooling
& $32$ / -- \\

\addlinespace[1pt]
\midrule

\multirow{2}{*}{Prediction head}
& Input
& Aggregated history
& $32$ / -- \\
& H1
& Fully connected
& $128$ / ReLU \\
& H2
& Fully connected
& $64$ / ReLU \\
& Output
& Class logits
& $K$ / -- \\

\bottomrule
\end{tabular}
\vspace{-1em}
\end{table}

\begin{table}[t]
\centering
\scriptsize
\setlength{\tabcolsep}{5.0pt}
\renewcommand{\arraystretch}{1.12}
\caption{\small \textbf{MNIST masked classification hyperparameters: \textsc{DAD} and \textsc{iDAD}.}
For \textsc{iDAD}, no additional downstream classifier is trained: predictions are obtained directly from the critic logits. The contrastive denominator is computed exactly over the discrete label space. The \textsc{Random} downstream network uses the same hyperparameters as \textsc{DAD}.
}
\label{tab:hyperparameters_mnist_1}

\begin{tabular}{lcc}
\toprule
\textbf{Field} & \textbf{DAD} & \textbf{iDAD} \\
\midrule
Batch size & $512$ & $512$ \\
Learning rate & $5 \times 10^{-4}$ & $10^{-3}$ \\
Policy steps & $50\mathrm{K}$ & $10\mathrm{K}$ \\
Contrastive samples & $L=500$ & Exact over labels \\
Downstream steps & $50\mathrm{K}$ & -- \\
Downstream LR & $10^{-4}$ & -- \\
\bottomrule
\end{tabular}
\vspace{-1em}
\end{table}

\begin{table}[t]
\centering
\scriptsize
\setlength{\tabcolsep}{5.0pt}
\renewcommand{\arraystretch}{1.12}
\caption{\small \textbf{MNIST masked classification hyperparameters: \textsc{Action-BED} and \textsc{ALINE}.}
For \textsc{ALINE}, predictions are obtained directly from the classification head added to the posterior inference network.
}
\label{tab:hyperparameters_mnist_2}

\begin{tabular}{lcc}
\toprule
\textbf{Field} & \textbf{ALINE} & \textbf{Action-BED} \\
\midrule
Batch size & $200$ & $512$ \\
Number of queries & $128$ & -- \\
Number of queries, burn-in & $128$ & -- \\
Learning rate & $10^{-3}$ & $5\times 10^{-4}$ \\
Total steps & $90\mathrm{K}$ & $10\mathrm{K}$ \\
Burn-in steps & $10\mathrm{K}$ & -- \\
Warm-up steps & -- & $0$ \\
Joint steps & -- & $10\mathrm{K}$ \\
\bottomrule
\end{tabular}
\vspace{-1em}
\end{table}

\paragraph{Architectures.} The acquisition policies used by \textsc{DAD}, \textsc{iDAD}, and \textsc{Action-BED} rely on the same architectural backbone as in Appendix~\ref{appendix:architectures}. Apart from the task-specific dimensions of the design--observation pairs, the encoding dimensions are kept unchanged with respect to Table~\ref{tab:architecture-dimensions}. The only additional modification is a sigmoid activation at the output of the emitter network, so that the proposed designs lie in the normalized domain $[0,1]^2$ before being reprojected onto the pixel grid.

\textsc{Action-BED} jointly trains the acquisition policy and downstream classifier for $10$K gradient steps using the downstream cross-entropy objective, with no warm-up phase. \textsc{DAD} is trained on sPCE objective, since it requires an explicit likelihood. This objective is defined with respect to the latent image $\theta$ rather than the class label $z$, therefore,  its policy seeks designs that discriminate between images, rather than directly between classes. For both \textsc{DAD} and \textsc{Random}, a downstream classifier with the same architecture as in \textsc{Action-BED} is then trained separately for $50$K gradient steps on trajectories generated by the corresponding policy.

We also compare against \textsc{iDAD}, which adapts the contrastive InfoNCE objective to the implicit class-conditional setting. It jointly learns a design policy $\pi_d^{\mathrm{iDAD}}$ and a critic $U_\psi(z,h_T)$ measuring the compatibility between a candidate label and the observed history. Since the latent parameter space is discrete, the nested expectation appearing in the denominator of the InfoNCE objective can be accurately approximated by summing over all possible labels rather than relying on Monte Carlo negative samples. In practice, this amounts to replacing the usual contrastive denominator by a categorical normalization over the label space,
\[
    \mathcal L_{\mathrm{iDAD}}(\phi,\psi)
    =
    -\mathbb E_{z,h_T}
    \left[
    \log
    \frac{\exp\{U_\psi(z,h_T)\}}
    {\sum_{z'=0}^{9} \exp\{U_\psi(z',h_T)\}}
    \right].
\]

For \textsc{ALINE}, we keep the native encoder--decoder architecture and modify only the target prediction head. Whereas the original implementation parameterises the posterior by a Gaussian mixture, the masked classification task requires a discrete posterior over labels. We therefore replace the Gaussian-mixture head by a classification head producing categorical logits. The target decoder outputs embeddings $z_{\mathrm{target}}$, with one target token per class. Each token embedding is passed through a small MLP to produce one scalar logit,
\[
    \ell_k = f_{\mathrm{cls}}(z_{\mathrm{target},k}),
    \qquad k=1,\dots,10,
\]
and hence a categorical posterior through a softmax. The posterior inference network is trained with the standard cross-entropy loss. This preserves ALINE's native token-based architecture while adapting its output layer to the discrete label space.

Similarly, \textsc{iDAD} already includes a critic network which, in this setting, directly scores class labels as a function of the acquisition history. Its outputs can therefore be interpreted as class logits. Consequently, for both \textsc{iDAD} and \textsc{ALINE}, we do not train an additional downstream classifier: predictions are obtained directly from the critic logits for \textsc{iDAD} and from the classification head logits for \textsc{ALINE}. 

\paragraph{Evaluation.} For MNIST masked classification, loss and accuracy are evaluated after training on the full training and test sets. Standard errors are computed across examples: for the cross-entropy loss, we report the empirical standard error of the per-example losses; for accuracy, we use the Bernoulli standard error associated with the empirical success rate. These uncertainty estimates therefore reflect finite-sample evaluation variability. Additional Figure~\ref{fig:four_policy_rollouts_mnist} shows representative design trajectories for the different methods.

\begin{figure}[t]
    \centering
    \small
    \begin{subfigure}[t]{0.24\textwidth}
        \centering
        \includegraphics[width=\linewidth]{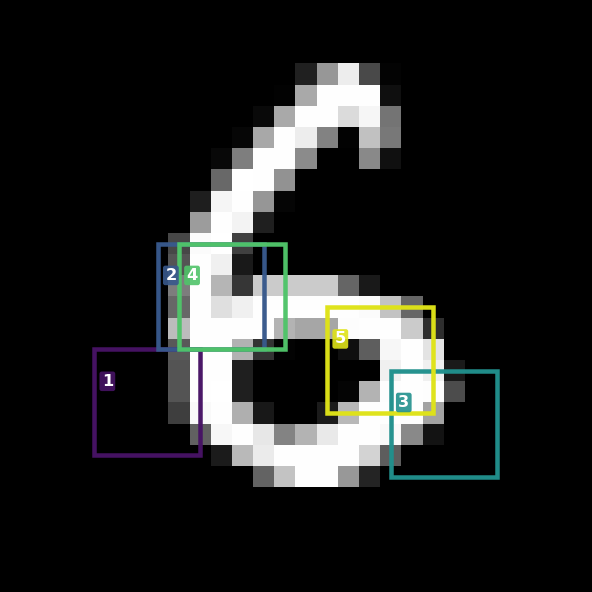}
        \caption{\textsc{DAD}}
    \end{subfigure}
    \hfill
    \begin{subfigure}[t]{0.24\textwidth}
        \centering
        \includegraphics[width=\linewidth]{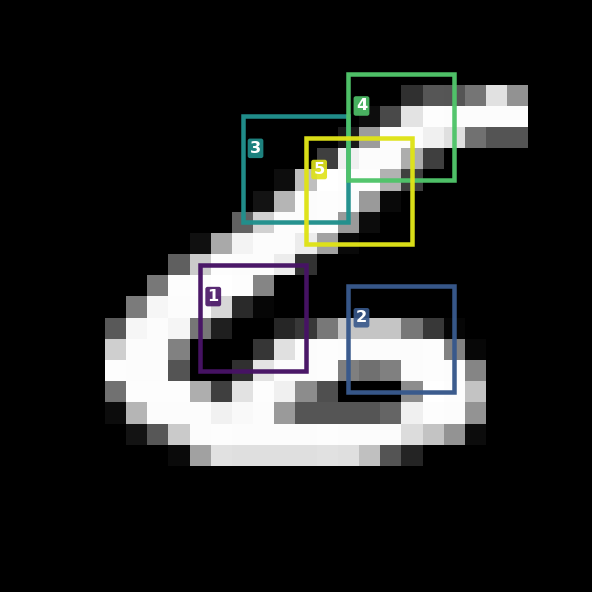}
        \caption{\textsc{iDAD}}
    \end{subfigure}
    \hfill
    \begin{subfigure}[t]{0.24\textwidth}
        \centering
        \includegraphics[width=\linewidth]{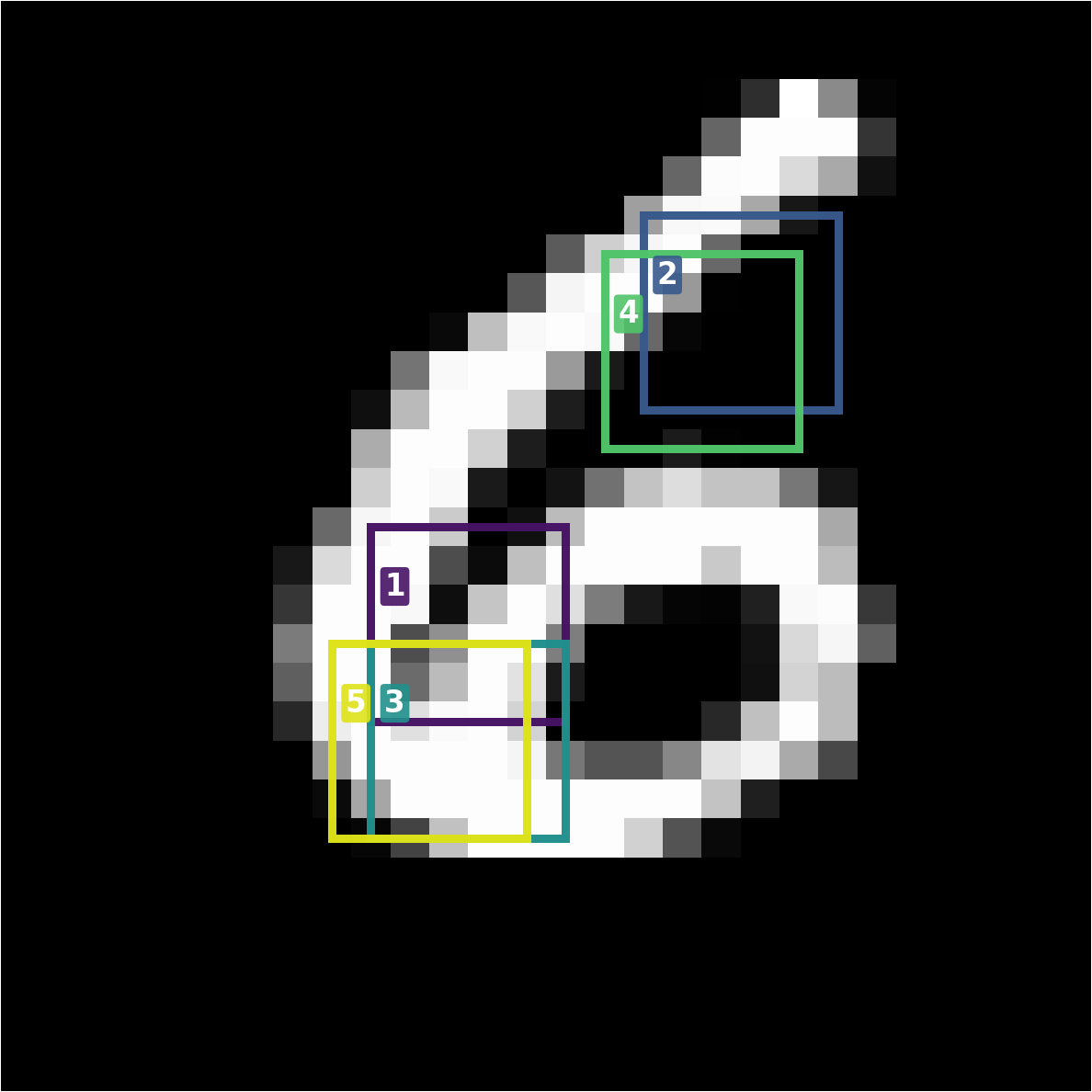}
        \caption{\textsc{ALINE}}
    \end{subfigure}
    \hfill
    \begin{subfigure}[t]{0.24\textwidth}
        \centering
        \includegraphics[width=\linewidth]{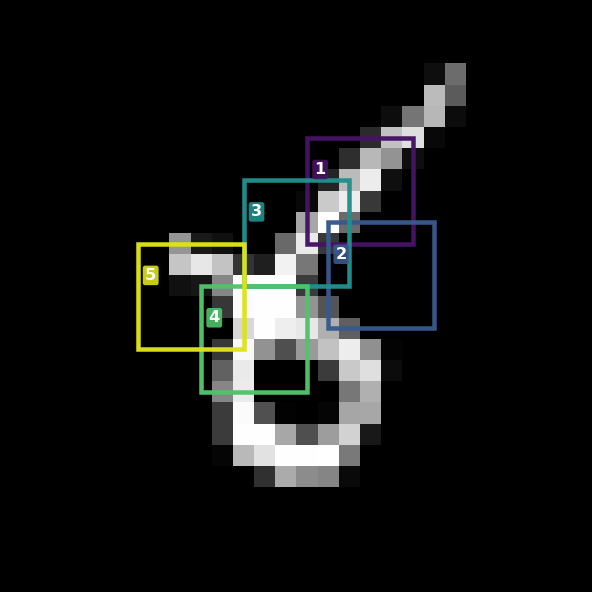}
        \caption{\textsc{Action-BED}}
    \end{subfigure}

    \caption{\small \textbf{Comparison of four policy rollouts.}
    Each panel shows one experimental design trajectory under the trained design policy.}
    \label{fig:four_policy_rollouts_mnist}
\end{figure}

\begin{table}[t]
\centering
\scriptsize
\setlength{\tabcolsep}{5.0pt}
\renewcommand{\arraystretch}{1.15}

\begin{threeparttable}
\caption{
\small \textbf{Pair Encoder and Emitter Architectures}.
}
\label{tab:pair_encoder_emitter_architecture}

\begin{tabular}{llll}
\toprule
\textbf{Module} & \textbf{Layer} & \textbf{Description} & \textbf{Dimension / Activation} \\
\midrule

\multirow{2}{*}{Pair encoder}
& Input
& $\xi, y$
& $3$ / -- \\
& H1
& Fully connected
& $256$ / ReLU \\
& Output
& Fully connected
& $16$ / -- \\

\addlinespace[1pt]
\midrule

\multirow{2}{*}{Emitter}
& Input
& $R(h_t)$
& $16$ / -- \\
& H1
& Fully connected
& $2$ / -- \\
& Output
& $\xi$
& $2$ / -- \\

\bottomrule
\end{tabular}

\end{threeparttable}
\end{table}

\newpage
\subsection{Robustness to Network Architecture}
\label{appendix:light_archi}

This section reports additional results obtained under the same experimental protocols as in the main experiments, but using lighter neural architectures for the acquisition policy, and smaller computational budgets, while retaining the downstream action architectures used in the main experiments. The goal is not to improve upon the main quantitative results, but rather to assess whether the qualitative conclusions are robust to reduced model capacity. In particular, we investigate whether the relative behaviour of \textsc{Action-BED} and the information-theoretic baselines \textsc{DAD} and \textsc{iDAD} remains consistent. Across tasks, the experimental setting, simulators, evaluation metrics, and train/test protocols are kept identical to those used in the corresponding main experiments. As expected, this generally leads to weaker absolute performance compared with the main results. 

The precise architecture depends on the structure of the task. For exchangeable observation models, we use a DeepSets-style architecture in which each design--observation pair $(\xi_s,y_s)$ is encoded independently and the resulting embeddings are aggregated by sum or mean pooling. For sequential and non-exchangeable simulators, such as the stochastic pendulum, the history is instead processed by a recurrent network, since the likelihood depends explicitly on the previous state. For each task, the downstream action policy $\pi_a^\psi$ is kept identical to that used in the corresponding main experiment. It is implemented as an MLP taking the final history $h_T$ as input and returning the estimate or action associated with the chosen downstream loss.

\begin{table}[t]
\centering
\scriptsize
\setlength{\tabcolsep}{5.0pt}
\renewcommand{\arraystretch}{1.15}

\caption{\small
\textbf{Source Location Finding hyperparameters}.
}
\label{tab:lfmlphparams}

\begin{tabular}{lcc}
\toprule
\textbf{Hyperparameter} & \textbf{DAD} & \textbf{Action-BED} \\
\midrule

Batch size
& $2000$
& $2000$ \\

Contrastive samples $L$
& $2000$
& -- \\

Policy steps
& $50\mathrm{K}$
& $150\mathrm{K}$ \\

Downstream steps
& $150\mathrm{K}$
& -- \\

Warm-up steps
& --
& $50K$ \\

Optimizer
& Adam
& Adam \\

Adam betas
& $(0.8,0.998)$
& $(0.8,0.998)$ \\

Learning rate
& $5\times 10^{-4}$
& $7\times 10^{-4}$ \\

LR decay $\gamma$
& $0.95$
& $0.95$ \\

LR decay period
& $2000$
& $2000$ \\

Weight decay
& $0$
& $0$ \\

Random seed
& $42$
& $42$ \\

\bottomrule
\end{tabular}
\vspace{-1em}
\end{table}

\subsubsection{Source Location Finding}

\paragraph{Training.}
For the source-location experiment, \textsc{DAD} and \textsc{Action-BED} use the same data-acquisition architecture, reported in Table~\ref{tab:pair_encoder_emitter_architecture}. This architecture is the one used in \citet{foster_deep_2021} for the corresponding DAD experiment, ensuring that differences in performance are not driven by changes in policy-network capacity. Each design--observation pair $(\xi,y)$ is first processed by a pair encoder, which maps it to a fixed-dimensional embedding. The embeddings associated with the history $h_t$ are then aggregated by a permutation-invariant pooling operation, either sum or mean pooling, and passed to an emitter network that outputs the next design $\xi_{t+1}$. Both the pair encoder and the emitter are fully connected MLPs.

The \textsc{Random} baseline does not use a learned acquisition network: designs are sampled independently from the initial prior distribution $p(\theta)$. The corresponding training hyperparameters for the learned policies and downstream action networks are reported in Table~\ref{tab:lfmlphparams}.

\paragraph{Results.} The results are reported in Table~\ref{tab:downstream_comparison}. After training, all policies are evaluated using $L=100{,}000$ contrastive samples and a batch size of $2000$ rollouts for the sequential PCE and NMC estimates. Downstream losses are computed on independent batches of $2000$ rollouts. 
Uncertainties are reported as $\pm 1$ standard error across rollout histories. We also report the total training time for each method, including both acquisition-policy and downstream-action-network training when applicable.

\begin{table}[t]
\centering
\scriptsize
\setlength{\tabcolsep}{5.0pt}
\renewcommand{\arraystretch}{1.15}

\caption{\small
\textbf{Source Location Finding, lightweight architectures}.
Uncertainties are reported as $ \pm 1 \ \mathrm{s.e.}$ over $2048$ evaluation rollouts.
Compute time includes training both the design and downstream action networks.
}
\label{tab:downstream_comparison}

\begin{tabular}{lccccc}
\toprule
\multirow{2}{*}{\textbf{Method}}
& \multicolumn{4}{c}{\textbf{Performance ($\pm$ 1 s.e.)}}
& \multicolumn{1}{c}{\textbf{Compute}} \\
\cmidrule(lr){2-5}
\cmidrule(lr){6-6}
& MSE $(10^{-2})$ $\downarrow$
& QLL $\downarrow$
& sPCE $\uparrow$
& sNMC $\uparrow$
& Time (min) $\downarrow$ \\
\midrule

Action-BED (MSE)
& $\mathbf{2.6 \pm 0.2}$
& $-5.59 \pm 0.03$
& $9.14 \pm 0.04$
& $9.95 \pm 0.07$
& 97 \\

Action-BED (QLL)
& $4.2 \pm 0.4$
& $\mathbf{-6.39 \pm 0.04}$
& $10.36 \pm 0.04$
& $12.72 \pm 0.03$
& 102 \\

\addlinespace[1pt]
\midrule

DAD (MSE)
& $13.6 \pm 0.7$
& $-3.72 \pm 0.04$
& $\mathbf{10.38 \pm 0.03}$
& $\mathbf{12.91 \pm 0.03}$
& 283 \\

DAD (QLL)
& $15.5 \pm 1.2$
& $-4.02 \pm 0.04$
& $\mathbf{10.38 \pm 0.03}$
& $\mathbf{12.91 \pm 0.03}$
& 285 \\

\addlinespace[1pt]
\midrule

Random
& $22.4 \pm 0.8$
& $-3.16 \pm 0.04$
& $8.27 \pm 0.04$
& $8.49 \pm 0.05$
& \textbf{31} \\

\bottomrule
\end{tabular}
\vspace{-1em}
\end{table}

\subsubsection{Training Stability over Training Seeds}
\label{app:training_stability_seeds}

To assess the stability of \textsc{Action-BED}, we repeated training across multiple random seeds and evaluated the resulting policies under the same protocol. We found that performance remained consistent across runs, with only moderate variation in downstream loss and EIG estimates. Notably, the standard errors across independently trained policies are of a similar order of magnitude to those obtained from rollout-level uncertainty estimates, suggesting that training-seed variability is not substantially larger than evaluation noise. This suggests that the joint optimisation of the design and action policies is not overly sensitive to random initialisation or stochastic training trajectories. In particular, the observed gains over EIG-based baselines in terms of downstream losses are not driven by a single favourable run, but persist across independently trained policies.

\begin{table}[H]
\centering
\scriptsize
\setlength{\tabcolsep}{6.0pt}
\renewcommand{\arraystretch}{1.15}
\caption{\small \textbf{Training stability over random seeds on source location finding.}
Uncertainties are reported as one standard error over independently trained policies with 15 different random seeds. For each trained policy, metrics are estimated using the same evaluation protocol.}
\label{tab:seed_stability_source_location}
\begin{tabular}{lcccc}
\toprule
\textbf{Method} 
& \textbf{MSE} $(10^{-2})$ $\downarrow$
& \textbf{QLL} $\downarrow$
& \textbf{sPCE} $\uparrow$
& \textbf{sNMC} $\uparrow$ \\
\midrule
\textsc{Action-BED} (MSE)
& $2.7 \pm 0.1$
& $-5.54 \pm 0.02 $
& $9.04 \pm 0.04 $
& $9.79 \pm 0.06$ \\

\textsc{Action-BED} (QLL)
& $3.9 \pm 0.3$
& $-6.27 \pm 0.02$
& $10.21 \pm 0.07$
& $12.67 \pm 0.09$ \\
\midrule
\addlinespace[1pt]
DAD (MSE)
& $14.1 \pm 0.3$
& $-3.78 \pm 0.05$
& $10.36 \pm 0.07$
& $12.79 \pm 0.05$ \\

DAD (QLL)
& $15.7 \pm 0.9$
& $-3.96 \pm 0.06$
& $10.36 \pm 0.07$
& $12.79 \pm 0.05$ \\
\midrule
\addlinespace[1pt]
Random
& $21.2 \pm 0.4$
& $-3.19 \pm 0.08$
& $8.21 \pm 0.02$
& $8.41 \pm 0.02$ \\
\bottomrule
\end{tabular}
\vspace{-1em}
\end{table}

\newpage
\subsubsection{Dynamical Systems}

\paragraph{Architecture.}
For both pendulum tasks, the acquisition network must account for the temporal dependence of the simulator. We therefore use the same sequential architecture for the stochastic single and double pendulum, inspired by \citet{ivanova_implicit_2021}. Each design--state pair $(\xi_t,x_t)$ is first embedded by an MLP encoder, and the resulting sequence of embeddings is processed by a two-layer LSTM. The final recurrent state is then passed through an emitter MLP, whose output dimension matches the design dimension of the task. The full acquisition architecture is summarised in Table~\ref{tab:pendulum_acquisition_architecture}.

The downstream action network is kept identical to that used in the corresponding main experiments. It is implemented as an MLP taking the full trajectory $h_T$ as input and returning the parameter estimate associated with the chosen downstream loss. Unless otherwise stated, all optimisation hyperparameters are also inherited from the main experiments. The only change is the training budget: for the stochastic single pendulum, all policy, downstream, and joint networks are trained for $30$K gradient steps; for the double pendulum, all corresponding networks are trained for $50$K gradient steps. No downstream-only warm-up phase is used in either task.

The results are reported in Tables~\ref{tab:pendulum_downstream_losses_small_network} and~\ref{tab:double_pendulum_downstream_losses_small_network}. They are obtained using the same evaluation protocol as in the main experiments, with downstream losses computed on independent rollout batches and information-theoretic metrics estimated using sequential PCE and NMC bounds.

\begin{table}[H]
\centering
\scriptsize
\setlength{\tabcolsep}{5.0pt}
\renewcommand{\arraystretch}{1.15}

\caption{\small
\textbf{Acquisition architecture for the pendulum tasks}.
$\dim(\xi)=1$ for the stochastic single pendulum and $\dim(\xi)=2$ for the double pendulum. The $\tanh$ output is subsequently rescaled to the task-specific control domain.
}
\label{tab:pendulum_acquisition_architecture}

\begin{tabular}{llll}
\toprule
\textbf{Module} & \textbf{Layer} & \textbf{Description} & \textbf{Dimension / Activation} \\
\midrule

\multirow{2}{*}{Pair encoder}
& Input
& $(\xi_t,x_t)$
& $\dim(\xi)+\dim(x)$ / -- \\
& H1
& Fully connected
& $256$ / ReLU \\
& H2
& Fully connected
& $256$ / ReLU \\
& Output
& Fully connected
& $64$ / -- \\

\addlinespace[1pt]
\midrule

\multirow{2}{*}{History aggregator}
& Input
& $\{R(\xi_s,x_s)\}_{s=0}^{t}$
& $(t+1)\times 64$ / -- \\
& H1
& LSTM
& $64$ / -- \\
& H2
& LSTM
& $64$ / -- \\

\addlinespace[1pt]
\midrule

\multirow{2}{*}{Emitter}
& Input
& Final recurrent state
& $64$ / -- \\
& H1
& Fully connected
& $256$ / ReLU \\
& H2
& Fully connected
& $256$ / ReLU \\
& Output
& Design $\xi_{t+1}$
& $\dim(\xi)$ / $\tanh$ \\

\bottomrule
\end{tabular}
\vspace{-1em}
\end{table}

\begin{table}[H]
\centering
\scriptsize
\setlength{\tabcolsep}{5.0pt}
\renewcommand{\arraystretch}{1.15}

\caption{\small
\textbf{Stochastic Single Pendulum}.
Uncertainties are reported as $\pm 1  \ \mathrm{s.e}$ over $2048$ rollout histories.
}
\label{tab:pendulum_downstream_losses_small_network}

\begin{tabular}{lccccc}
\toprule
\multirow{2}{*}{\textbf{Method}}
& \multicolumn{4}{c}{\textbf{Performance ($\pm$ 1 s.e.)}}
& \multicolumn{1}{c}{\textbf{Compute}} \\
\cmidrule(lr){2-5}
\cmidrule(lr){6-6}
& MSE $(10^{-2})$ $\downarrow$
& QLL $\downarrow$
& sPCE $\uparrow$
& sNMC $\uparrow$
& Time (min) $\downarrow$ \\
\midrule

Action-BED (MSE)
& $\mathbf{1.0 \pm 0.09}$
& $-5.30 \pm 0.03$
& $\mathbf{3.35 \pm 0.02}$
& $\mathbf{3.40 \pm 0.02}$
& 45 \\

Action-BED (QLL)
& $1.1 \pm 0.11$
& $\mathbf{-5.85 \pm 0.04}$
& $2.55 \pm 0.01$
& $2.66 \pm 0.01$
& 47 \\

\addlinespace[1pt]
\midrule

DAD
& $1.7 \pm 0.16$
& $-5.34 \pm 0.05$
& $2.63 \pm 0.12$
& $2.65 \pm 0.11$
& 242 \\

\addlinespace[1pt]
\midrule

Random
& $2.90 \pm 0.13$
& $-4.15 \pm 0.08$
& $1.39 \pm 0.07$
& $1.42 \pm 0.08$
& $\mathbf{5}$ \\

\bottomrule
\end{tabular}
\vspace{-2em}
\end{table}

\begin{table}[H]
\centering
\scriptsize
\setlength{\tabcolsep}{5.0pt}
\renewcommand{\arraystretch}{1.15}

\caption{\small
\textbf{Stochastic Double Pendulum}.
Uncertainties are reported as $\pm 1  \ \mathrm{s.e}$ over $2048$ rollout histories.
}
\label{tab:double_pendulum_downstream_losses_small_network}

\begin{tabular}{lccccc}
\toprule
\multirow{2}{*}{\textbf{Method}}
& \multicolumn{4}{c}{\textbf{Performance ($\pm$ 1 s.e.)}}
& \multicolumn{1}{c}{\textbf{Compute}} \\
\cmidrule(lr){2-5}
\cmidrule(lr){6-6}
& MSE $(10^{-3})$ $\downarrow$
& QLL $\downarrow$
& sPCE $\uparrow$
& sNMC $\uparrow$
& Time (min) $\downarrow$ \\
\midrule

Action-BED (MSE)
& $\mathbf{0.7 \pm 0.05}$
& $-9.25 \pm 0.02$
& $10.40 \pm 0.02$
& $14.50 \pm 0.13$
& 276 \\

Action-BED (QLL)
& $1.4 \pm 0.08$
& $\mathbf{-9.95 \pm 0.04}$
& $9.90 \pm 0.03$
& $12.50 \pm 0.07$
& 274 \\

\addlinespace[1pt]
\midrule

DAD (MSE)
& $1.0 \pm 0.04$
& $-9.32 \pm 0.03$
& $\mathbf{10.75 \pm 0.03}$
& $\mathbf{14.10 \pm 0.14}$
& 582 \\

\addlinespace[1pt]
\midrule

Random
& $2.0 \pm 0.09$
& $-6.80 \pm 0.05$
& $8.26 \pm 0.06$
& $8.41 \pm 0.08$
& $\mathbf{58}$ \\

\bottomrule
\end{tabular}
\vspace{-1em}
\end{table}

\subsubsection{MNIST Masked Classification}

For the MNIST masked-classification task, we use the same class of acquisition architectures as in the source-location experiment. In particular, the design policies are implemented with a DeepSets-style backbone: each design--observation pair is encoded independently, the resulting embeddings are aggregated over the acquisition history, and an emitter network outputs the next design. The precise architecture is reported in Table~\ref{tab:mnist_policy_architecture}.

The downstream classifier follows the same encoder--aggregation as summarised in Table~\ref{tab:downstream_net_architecture_mnist}. For \textsc{iDAD}, we use the same backbone for the design policy but the critic encoder is used as a downstream classifier as in the main experiment. 

Training hyperparameters for all methods are kept identical to those used in the corresponding main experiment. The corresponding classification results, including cross-entropy loss and accuracy on both training and test sets, are presented in Table~\ref{tab:mnist_classification_task_small_network}.

\begin{table}[t]
\centering
\scriptsize
\setlength{\tabcolsep}{5.0pt}
\renewcommand{\arraystretch}{1.15}

\caption{\small
\textbf{MNIST masked classification acquisition-policy architecture}.
}
\label{tab:mnist_policy_architecture}

\begin{tabular}{llll}
\toprule
\textbf{Module} & \textbf{Layer} & \textbf{Description} & \textbf{Dimension / Activation} \\
\midrule

\multirow{2}{*}{Pair encoder}
& Input
& $(\xi_t,y_t)$
& $\dim(\xi)+\dim(y)$ / -- \\
& H1
& Fully connected
& $256$ / ReLU \\
& H2
& Fully connected
& $128$ / ReLU \\
& H3
& Fully connected
& $64$ / ReLU \\
& Output
& Fully connected
& $16$ / -- \\

\addlinespace[1pt]
\midrule

\multirow{2}{*}{History aggregation}
& Input
& $\{R(\xi_s,y_s)\}_{s=1}^{t}$
& $t \times 16$ / -- \\
& Pooling
& Sum pooling
& $16$ / -- \\

\addlinespace[1pt]
\midrule

\multirow{2}{*}{Emitter}
& Input
& Aggregated history
& $16$ / -- \\
& Output
& Design $\xi_{t+1}$
& $2$ / Sigmoid \\

\bottomrule
\end{tabular}

\vspace{0.4em}

\begin{minipage}{\linewidth}
\footnotesize
\textit{Notes.}
The design-policy backbone is a DeepSets-style net. A sigmoid activation is used to constrain designs to $[0,1]^2$ before reprojection onto the pixel grid.
\end{minipage}
\vspace{-1em}
\end{table}

\begin{table}[H]
\centering
\scriptsize
\setlength{\tabcolsep}{5.0pt}
\renewcommand{\arraystretch}{1.15}

\caption{\small
\textbf{MNIST masked classification}.
CE denotes cross-entropy loss and Acc. denotes classification accuracy. Metrics are computed after training on the full training and test datasets. Uncertainties are reported as one standard error across evaluation batches. Compute time includes training the acquisition policy and, when applicable, the downstream classifier.
}
\label{tab:mnist_classification_task_small_network}

\begin{tabular}{lccccc}
\toprule
\multirow{2}{*}{\textbf{Method}}
& \multicolumn{4}{c}{\textbf{Performance ($\pm$ 1 s.e.)}}
& \multicolumn{1}{c}{\textbf{Compute}} \\
\cmidrule(lr){2-5}
\cmidrule(lr){6-6}
& Train CE $\downarrow$
& Train Acc. (\%) $\uparrow$
& Test CE $\downarrow$
& Test Acc. (\%) $\uparrow$
& Time (min) $\downarrow$ \\
\midrule

Action-BED
& $\mathbf{0.050 \pm 0.002}$
& $\mathbf{98.24 \pm 0.04}$
& $\mathbf{0.128 \pm 0.006}$
& $\mathbf{96.55 \pm 0.13}$
& $\mathbf{2.5}$ \\

\addlinespace[1pt]
\midrule

DAD
& $0.398 \pm 0.004$
& $83.10 \pm 0.03$
& $0.609 \pm 0.010$
& $77.75 \pm 0.12$
& 18 \\

iDAD
& $0.113 \pm 0.003$
& $96.28 \pm 0.07$
& $0.175 \pm 0.007$
& $94.76 \pm 0.19$
& 5 \\

\addlinespace[1pt]
\midrule

Random
& $0.786 \pm 0.015$
& $71.78 \pm 0.87$
& $0.797 \pm 0.015$
& $71.40 \pm 0.86$
& 5 \\

\bottomrule
\end{tabular}
\vspace{-2em}
\end{table}


\end{document}